\documentclass{article}
\usepackage{fullpage}
\usepackage[utf8]{inputenc}
\usepackage[T1]{fontenc}
\usepackage{authblk}
\usepackage{amsmath}
\usepackage{amssymb}
\usepackage{mathtools}
\usepackage{amsthm}
\usepackage{amsfonts}
\usepackage{nicefrac}
\usepackage{microtype}
\usepackage{xcolor}
\usepackage{placeins}
\usepackage{graphicx}
\usepackage{enumitem}
\usepackage{wrapfig}
\usepackage{thmtools}
\usepackage{thm-restate}
\usepackage{subcaption}
\usepackage{tcolorbox}
\usepackage{booktabs}
\usepackage{url}
\usepackage{refcount}
\usepackage[round]{natbib}
\definecolor{gdblue}{HTML}{0072B2}
\definecolor{muongreen}{HTML}{009E73}
\definecolor{signorange}{HTML}{D55E00}
\definecolor{darkscarlet}{rgb}{0.34, 0.01, 0.1}
\definecolor{yaleblue}{rgb}{0.06, 0.3, 0.57}
\definecolor{darkpowderblue}{rgb}{0.0, 0.2, 0.6}

\DeclareRobustCommand{\mybox}[2][gray!10]{%
  \begin{tcolorbox}[left=0pt,right=0pt,top=0pt,bottom=0pt,colback=#1,colframe=#1,
    width=\linewidth,enlarge left by=0mm,boxsep=5pt,arc=0pt,outer arc=0pt]
  #2
  \end{tcolorbox}}

\definecolor{algblue}{rgb}{0.1, 0.1, 0.44}
\newcommand{\algname}[1]{{\color{algblue!70!black}\sf #1}}
\newcommand{\approxalg}[1]{{\ttfamily\color{algblue!70!black}#1}}
\newcommand{\R}{\mathbb{R}}
\newcommand{\ip}[2]{\langle #1,#2\rangle}
\newcommand{\norm}[1]{\left\lVert #1\right\rVert}
\newcommand{\F}{\mathrm{F}}
\newcommand{\T}{\mathsf{T}}
\newcommand{\polar}{\mathrm{polar}}
\newcommand{\nuc}{*}
\newcommand{\tr}{\mathrm{tr}}

\newcommand{\rank}{\mathrm{rank}}
\newcommand{\sign}{\mathrm{sign}}
\newcommand{\diag}{\mathrm{diag}}
\newcommand{\range}{\mathrm{range}}
\newcommand{\spec}{\mathrm{spec}}

\newtheorem{lemma}{Lemma}
\newtheorem{proposition}{Proposition}
\newtheorem{corollary}{Corollary}
\newtheorem{remark}{Remark}

\usepackage{hyperref}
\usepackage[hyperpageref]{backref}
\hypersetup{pagebackref=true,colorlinks=true,citecolor=darkscarlet,linkcolor=darkpowderblue,
  filecolor=magenta,urlcolor=yaleblue,menucolor=gray}
\renewcommand*{\backrefalt}[4]{%
  \ifcase #1 \footnotesize{(Not cited.)}%
  \or        \footnotesize{(Cited on page~#2)}%
  \else      \footnotesize{(Cited on pages~#2)}%
  \fi}

\title{\bf A Quadratic Lens on Muon's Polar Update: Benefits, Pitfalls, and the Role of Newton--Schulz$^\dagger$}
\author[1,2,$\ast$,$\ddagger$]{\bf Egor Shulgin}
\author[3,4,$\ast$]{\bf Sam Laing}
\author[3]{\bf Antonio Orvieto}
\author[1]{\bf Peter Richt\'arik}
\affil[1]{King Abdullah University of Science and Technology (KAUST), Saudi Arabia}
\affil[2]{Mohamed bin Zayed University of Artificial Intelligence (MBZUAI), United Arab Emirates}
\affil[3]{ELLIS Institute T\"ubingen and MPI for Intelligent Systems, T\"ubingen AI Center, Germany}
\affil[4]{OpenEuroLLM}
\date{}

\begin{document}
\maketitle
{\renewcommand{\thefootnote}{$\ast$}\footnotetext{Equal contribution. Correspondence to Egor Shulgin, \texttt{shulgin.yegor@gmail.com}.}}
{\renewcommand{\thefootnote}{$\dagger$}\footnotetext{A shorter version of this work was presented at the ICML 2026 Workshop on High-dimensional Learning Dynamics.}}
{\renewcommand{\thefootnote}{$\ddagger$}\footnotetext{The initial versions of this work were completed while Egor Shulgin was at KAUST.}}

\begin{abstract}
Muon updates matrix-valued parameters along approximately orthogonalized directions and often trains large language models faster than \algname{AdamW}, yet it remains unclear when and why orthogonalization helps. We study \algname{PolarGD}, the orthogonalized update at the core of \algname{Muon}, on matrix quadratics, where optimizers can be analyzed in closed form, and compare update directions under line search to remove the effect of learning-rate tuning. Orthogonalization pays off when the gradient is far from low rank and its largest singular directions see high curvature. \algname{PolarGD} reaches the minimizer in one step whenever the rows of the error are orthogonal with equal norms, for any Hessian, whereas gradient descent needs all Hessian eigenvalues to be equal; and unlike sign descent, \algname{PolarGD} is unaffected by rotations of the problem. Over many steps, \algname{PolarGD} acts as gradient descent preconditioned by the current iterate, so its speed is governed by the conditioning of the iterate rather than of the Hessian. This explains fast regimes as well as a slow one, caused by tiny singular values of the iterate, in which \algname{PolarGD} is much slower than gradient descent. Finally, the few \approxalg{Newton--Schulz} iterations used in \algname{Muon} treat these tiny singular values as gradient descent does, so fewer iterations can converge faster than exact orthogonalization; in a transformer trained with \algname{Muon}, about a quarter of the singular values of the attention matrices fall in this regime.
\end{abstract}

\section{Introduction}

\algname{Muon} \citep{jordan2024muon} has emerged as a strong optimizer for training large neural networks.
It has matched or outperformed \algname{AdamW} in language-model pretraining, from speedrun
benchmarks \citep{jordan2024moddednanogpt} to large-scale training \citep{liu2025muon,shah2025practical}, and standard
optimizers more broadly in vision \citep{jordan2024muon,pethick2025training}, diffusion
\citep{schaipp2025optimization}, and tabular models \citep{gorishniy2026benchmarking}. For two-dimensional parameter blocks, \algname{Muon}
approximately orthogonalizes the momentum matrix with a few \approxalg{Newton--Schulz} iterations, that is,
it steps along an approximation of the polar factor $\polar(G)=UV^\top$ of $G=U\Sigma_G V^\top$.
The same direction is shared by a broader family of matrix optimizers. It is the steepest-descent
direction for the spectral norm, used in stochastic spectral descent
\citep{carlson2015stochasticRBM,carlson2015stochastic,carlson2015preconditioned} and revisited by
\citet{bernstein2024old}.
It is also the linear-minimization-oracle (LMO) direction of norm-constrained methods
\citep{pethick2025training}, and \algname{Shampoo} \citep{gupta2018shampoo} reduces to it when preconditioner
accumulation is switched off \citep{bernstein2024old}. Controlled benchmarks place such
matrix-based optimizers among the strongest available \citep{wen2026fantastic}, so understanding when
a step along $\polar(G)$ helps matters well beyond \algname{Muon}.

Existing analyses view these methods as non-Euclidean trust-region \citep{kovalev2025muon} or
LMO-based steepest descent \citep{pethick2025training} and provide convergence guarantees for \algname{Muon} \citep{li2025note}. These are
worst-case bounds over broad function classes, governed by smoothness constants in the chosen norm and by norm-equivalence
factors; where they suggest an advantage over (stochastic) gradient descent (\algname{GD}), it comes from comparing such constants
across norms \citep{kim2026convergence,davis2025spectral}. They do not separate the effect of the orthogonalized direction
from that of the step-size scale set by these constants, so when this direction improves on the gradient step remains open. Analyses of \algname{Muon} with finite \approxalg{Newton--Schulz} treat it as an approximation of the polar factor whose
error worsens the guarantees \citep{shulgin2026beyond,kim2026convergence,takezawa2026fedmuon}, related work attributes this
error to ill-conditioning and replaces the solver \citep{refael2025sumo,he2025lowrank}, and improved solvers aim to reduce it
\citep{cesista2025muonoptcoeffs,grishina2025accelerating,amsel2025polar}.

We study this question on matrix quadratics. Quadratic models are the classical setting for
understanding optimizers, and they have repeatedly explained optimizer behavior in deep learning,
from critical batch sizes and the value of preconditioning \citep{zhang2019algorithmic} to the
advantage of \algname{Adam} on transformers \citep{zhang2024transformers} and of sign descent under
Zipf-distributed data \citep{kunstner2025scalinglawsgradientdescent}. Their matrix-valued form
already contains the features that matter for orthogonalized updates: singular structure, anisotropic
curvature, dependence on initialization, and the choice of coordinate system. We focus on the core
update of \algname{Muon} and of the methods above,
$X_{k+1}=X_k-\gamma_k\polar(\nabla f(X_k))$, and compare it with \algname{GD} and \algname{SignGD}. We call this update \algname{PolarGD} by
analogy with \algname{SignGD}: it replaces the gradient by its polar factor, as \algname{SignGD} replaces it by its entrywise sign
(it has also been called spectral descent \citep{carlson2015stochastic}). Momentum and
stochasticity are left out on purpose, so that the effect of the update direction can be isolated.

To compare the three directions on equal footing, we let each take its best scalar step, computed by
exact line search. This removes learning-rate tuning from the comparison, so that methods are
compared through their directions rather than through step-size choices. On matrix quadratics, the
resulting one-step decreases are available in closed form and show a tradeoff between how much a
direction $D$ decreases the loss to first order, its gain $\ip{G}{D}$, and the curvature $\tr(D^\top HD)$ along it.

\paragraph{Contributions.} Our main contributions are as follows.
\begin{itemize}[leftmargin=1.2em,itemsep=0.3em]
\item \textbf{One-step geometry (Section~\ref{sec:one_step}).} We give closed-form line-search
decreases for \algname{GD}, \algname{SignGD}, and \algname{PolarGD}. Each method reaches the minimizer in one step under its own condition
(Theorem~\ref{thm:onestep_regimes}): \algname{GD} from every starting point only when the Hessian is a multiple of the
identity, \algname{SignGD} when the error has entries of equal magnitude and the Hessian preserves their signs (e.g., a
diagonal Hessian), and \algname{PolarGD},
for any positive-definite Hessian, when the rows of the error are orthogonal with equal norms. This explains
why orthogonal initialization is special for \algname{PolarGD} and diagonal structure for \algname{SignGD}. In the worst case over
iterates, however, the one-step comparison is pessimistic for \algname{PolarGD}: its guaranteed decrease is smaller than
\algname{GD}'s (Theorem~\ref{prop:precond:progress}), although \algname{PolarGD} is often faster over whole trajectories
(Figure~\ref{fig:representative_plot}).
\item \textbf{Rotation equivariance (Section~\ref{sec:curvature}).} \algname{PolarGD} is equivariant under
orthogonal changes of basis, while \algname{SignGD} is not (Theorem~\ref{thm:orthogonal}): rotating the problem rotates
the \algname{PolarGD} iterates but changes the \algname{SignGD} trajectory. Accordingly, \algname{PolarGD} keeps its performance when an ill-conditioned diagonal problem, such as a Zipf-type regression model of
language data, is rotated, whereas \algname{SignGD} degrades.
\item \textbf{Implicit preconditioning (Section~\ref{sec:precond}).} \algname{PolarGD} is a gradient step with a
preconditioner set by the current iterate, whose condition number is that of the iterate rather than of the
Hessian. When the iterate has full row rank, this gives an exact formula for the progress per step, for line search
and fixed steps. With line search, \algname{GD} satisfies \(f(X_{k+1})\le(1-1/\kappa(H))\,f(X_k)\); \algname{PolarGD} satisfies the same
bound with the condition number \(\kappa(H)\) of the Hessian replaced by the condition number \(\kappa(X_k)\) of the
iterate (taking \(X^\star=0\)). \algname{PolarGD} is therefore fastest when the rows of $X_k$ are orthogonal with equal norms and
slowest when $X_k$ has tiny singular values; this slow regime is tight, and in it \algname{GD} can be nearly one-step
exact while \algname{PolarGD} stays slow along the trajectory. We also show that a direction that is better for one step
need not give a faster trajectory.
\item \textbf{Rank deficiency and finite \approxalg{Newton--Schulz} (Section~\ref{subsec:precond:softened_updates}).}
We extend these identities to rank-deficient iterates and show that a few \approxalg{Newton--Schulz} steps, as used in
\algname{Muon}, give a softened polar update: they behave like damped polar, treating large singular values like exact polar and tiny ones
like a \algname{GD} step. Near rank deficiency this converges faster than exact polar, so finite \approxalg{Newton--Schulz} is not merely an
approximation of the polar factor.
\end{itemize}

\paragraph{Related and concurrent work.} \citet{gonon2026insights} also study \algname{Muon} on quadratics. On the isotropic
quadratic, exact polar steps with a constant step size can fail to reach the minimizer, and random perturbations of the
polar step restore reachability in a one-dimensional model; on general Hessians, they show numerically that the comparison
with \algname{GD} depends on the shape of the spectrum, that \approxalg{Newton--Schulz} can help or hurt, and that a direction that wins at every
step need not win along the trajectory. Our one-step and trajectory results are exact for any positive-definite Hessian.
We relate our one-step condition to \citet{davis2025spectral} in Section~\ref{sec:setup} and discuss other work in the sections where it
is relevant.

\section{Notation and Setup}
\label{sec:setup}

Let \(f:\R^{d\times m}\to\R\) be differentiable, \(G:=\nabla f(X)\), and
\(\ip{A}{B}:=\tr(A^\top B)\). We write $\norm{\cdot}_\F$, $\norm{\cdot}_2$, and $\norm{\cdot}_{\nuc}$
for the Frobenius, spectral, and nuclear norms, and $\norm{\cdot}_1$ for the entrywise $\ell_1$ norm.
We also write $\sigma_i(\cdot)$ for
singular values, $\kappa(A):=\sigma_{\max}(A)/\sigma_{\min}(A)$, $A^\dagger$ for the pseudoinverse with
$A^{\dagger/2}:=(A^\dagger)^{1/2}$ for $A\succeq0$, and $\Pi_{\mathcal V}$ for the orthogonal
projector onto a subspace $\mathcal V$.

We consider iterations
\begin{equation}
\label{eq:general_update}
  X_{k+1}=X_k-\gamma_k D_k,\qquad \gamma_k\ge0,
\end{equation}
where $D_k$ is computed from $G_k:=\nabla f(X_k)$ as one of the \algname{GD}, \algname{PolarGD}, and \algname{SignGD} directions
(dropping the index $k$) $D_{\rm GD}:=G$, $D_{\rm Pol}:=Q:=\polar(G)$, and $D_{\rm Sign}:=S:=\sign(G)$.
Here \(\polar(G):=UV^\top\) for a thin SVD \(G=U\Sigma_G V^\top\), and \(S\) is the entrywise sign
matrix with \(\sign(0)=0\). Both \(Q\) and \(S\) are LMO directions,
\(\mathrm{LMO}(G):=\arg\max_{\norm{D}\le1}\ip{G}{D}\): \(Q\) for the spectral norm and \(S\) for the
entrywise \(\ell_\infty\) norm, and \(\ip{G}{\mathrm{LMO}(G)}\) is the corresponding dual norm, namely
\(\norm{G}_{\nuc}\) and \(\norm{G}_1\), respectively.

\paragraph{Comparing directions by line search.}
We compare directions under exact line search, with step $\gamma^\star(D)\in\arg\min_{\gamma\ge0}f(X-\gamma D)$
and decrease $\Delta(D):=f(X)-f(X-\gamma^\star(D)D)$. With fixed steps, the polar and sign directions are used
in two forms: \emph{LMO-normalized}, $D_k=Q_k$ or $S_k$, which ignores the size of the gradient and is the form
used in most prior analyses \citep{kovalev2025muon,pethick2025training}, and \emph{dual-adapted},
$D_k=\norm{G_k}_{\nuc}Q_k$ or $\norm{G_k}_1S_k$, which multiplies the direction by the matching norm of the
gradient, as in stochastic spectral descent \citep{carlson2015stochastic} and \algname{PolarGrad} \citep{lau2025polargrad}. Since rescaling $D$ does not change the line-search
update $X-\gamma^\star(D)D$, line search treats both forms identically. The following bounds hold beyond
quadratics and become equalities for the quadratic model of Section~\ref{sec:one_step}.

\noindent\textbf{Assumption (matrix curvature bounds).}
There exist symmetric positive semidefinite $M\preceq L\in\R^{d\times d}$ such that for all
$X,D\in\R^{d\times m}$ the Bregman divergence
$B_f(X+D,X):=f(X+D)-f(X)-\ip{\nabla f(X)}{D}$ satisfies
\begin{equation}
\label{eq:matrix_sandwich}
  \tfrac12 \ip{M D}{D}
  \;\le\;
  B_f(X+D,X)
  \;\le\;
  \tfrac12 \ip{L D}{D},
\end{equation}
where $\ip{L D}{D}=\tr(D^\T L D)$ (and similarly for $M$). For $L=\ell I_d$ and $M=\mu I_d$, these are
the usual quadratic bounds implied by $\ell$-smoothness and $\mu$-strong convexity in the Frobenius
norm; a general $L$ captures anisotropic (e.g., row-wise) curvature. For
$f(X)=\tfrac12\tr(X^\T H X)$, both bounds hold with equality for $L=M=H$.

For a direction $D$ with gain $\ip{G}{D}>0$, write $\norm{D}_L^2:=\ip{L D}{D}$ for its curvature under the upper
bound. Minimizing the upper bound in~\eqref{eq:matrix_sandwich} along $X-\gamma D$
over $\gamma\ge0$ gives the step
\begin{equation}
\label{eq:gamma_model}
  \gamma_L(D) := \frac{\ip{G}{D}}{\norm{D}_L^2}.
\end{equation}

\begin{restatable}[One-step decrease along a direction]{proposition}{PropDecrease}
\label{prop:model_decrease}
Assume \eqref{eq:matrix_sandwich}. Fix $X$ and a direction $D$ with $\ip{G}{D}>0$ and $\norm{D}_L>0$.
Let $X^+ = X - \gamma_L(D) D$, where $\gamma_L(D)$ is given by \eqref{eq:gamma_model}. Then
\begin{equation}
\label{eq:model_decrease}
  f(X^+) \le f(X) - \frac{\ip{G}{D}^2}{2\norm{D}_L^2}.
\end{equation}
If $L=M$, equality holds and $\gamma_L(D)$ is the exact line-search step.
\end{restatable}

\begin{figure}[t]
\vspace{-2mm}
    \centering
    \includegraphics[width=\linewidth]{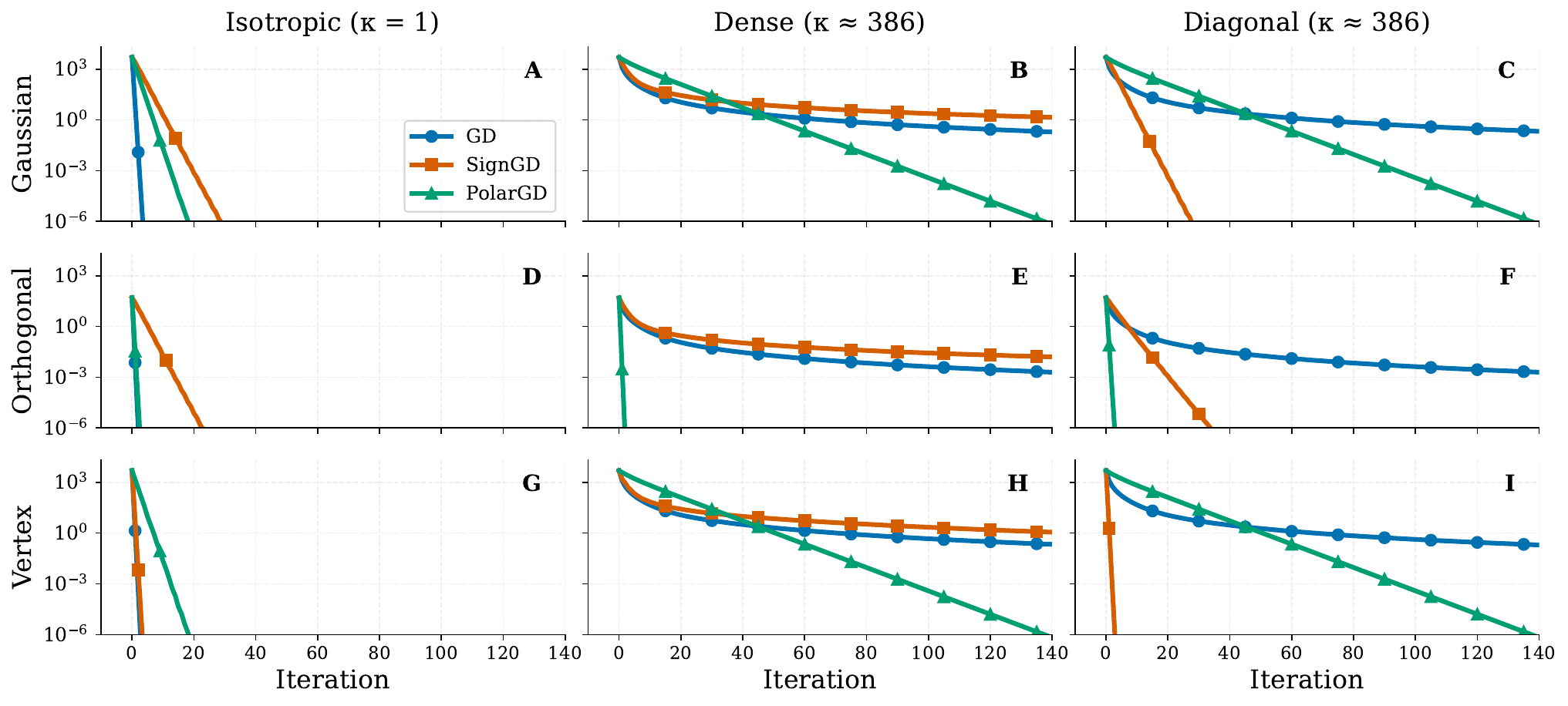}
    \vspace{-2mm}
    \caption{
    \textbf{Convergence across initialization--Hessian regimes.} 
    Performance of dual-adapted \textcolor{gdblue}{\textbf{\sf GD}}, \textcolor{signorange}{\textbf{\sf SignGD}} and \textcolor{muongreen}{\textbf{\sf PolarGD}} (the core \algname{Muon} update) on matrix quadratic objectives $f(X) = \frac{1}{2}\tr(X^\top H X)$ with tuned fixed step sizes, not line search (see App.~\ref{app:exp_details:lr_tune}). To isolate structural dependencies, we test three initializations: \emph{Gaussian} ($X_0 \sim \mathcal{N}(0,1)$), \emph{Orthogonal} and \emph{Vertex} ($X_0 \sim \{ -1,1\}$). We evaluate these against three Hessian profiles: \emph{Isotropic} ($H=I_d$), \emph{Dense} Marchenko--Pastur ($\kappa=386$) and \emph{Diagonal} (matching the spectrum in the dense case).
    Consistent with our analysis, each optimizer favors a distinct structural regime: \algname{GD} is robust to initialization but degrades under ill-conditioning, \algname{SignGD} benefits from coordinate-aligned curvature, and \algname{PolarGD} is rotation-equivariant and benefits from row-isotropic initialization. No single optimizer dominates across all regimes.
    }
    \label{fig:representative_plot}
    \vspace{-4mm}
\end{figure}
Proposition~\ref{prop:model_decrease} judges a direction by its squared gain over its curvature,
$\ip{G}{D}^2/\norm{D}_L^2$, independently of its scale. Exact line search decreases $f$ at least as much, and
exactly as much when $L=M$ (Appendix~\ref{app:onestep_proofs} brackets the exact decrease when $L\ne M$). Since
$\ip{G}{G}=\norm{G}_\F^2$ and $\ip{G}{Q}=\norm{G}_{\nuc}$, the polar direction guarantees a larger decrease than \algname{GD}
if and only if
\begin{equation}\label{eq:pol_beats_gd}
  \frac{\norm{G}_L}{\norm{Q}_L}>\frac{\norm{G}_\F^2}{\norm{G}_{\nuc}}.
\end{equation}
This is why we allow matrix-valued curvature bounds in~\eqref{eq:matrix_sandwich}. A scalar bound
$L=\ell I_d$ charges every direction the same curvature per unit Frobenius norm, $\norm{D}_L^2=\ell\norm{D}_\F^2$. The
left side of~\eqref{eq:pol_beats_gd} is then $\norm{G}_\F/\norm{Q}_\F=\norm{G}_\F/\sqrt r$, with $r=\rank(G)$, while
the Cauchy--Schwarz inequality $\norm{G}_{\nuc}\le\sqrt r\,\norm{G}_\F$ shows that the right side is at least
$\norm{G}_\F/\sqrt r$. So~\eqref{eq:pol_beats_gd} never holds: under a scalar bound, the polar direction never
guarantees more decrease than \algname{GD}, and guarantees as much only when all nonzero singular values of $G$ are equal.
A matrix $L$ changes this. Writing $G=\sum_i\sigma_iu_iv_i^\T$, the \algname{GD} curvature
$\norm{G}_L^2=\sum_i\sigma_i^2\,u_i^\T Lu_i$ weights the curvature along each left singular vector $u_i$ by
$\sigma_i^2$, whereas the polar curvature $\norm{Q}_L^2=\sum_iu_i^\T Lu_i$ weights all of them equally. When the
large singular directions of $G$ see high curvature, \algname{GD} pays for it most, and the polar direction can win. The right side depends only on the singular values of $G$:
for fixed $\norm{G}_\F$, it is largest when $G$ is close to rank one (at rank one, $Q$ is proportional to $G$ and
the two directions coincide; Remark~\ref{rmk:vector}) and smallest when all nonzero singular values of $G$ are
equal. On quadratics, $L=M=H$ and~\eqref{eq:pol_beats_gd} compares the exact line-search
decreases (Section~\ref{sec:one_step}). Squaring~\eqref{eq:pol_beats_gd} gives $\norm{G}_{\nuc}^2/\norm{G}_\F^2>\norm{G}_\F^2\norm{Q}_L^2/\norm{G}_L^2$.
Bounding the two curvatures by their worst cases, $\norm{G}_L^2\le\lambda_{\max}(L)\norm{G}_\F^2$ and
$\norm{Q}_L^2\le\tr(L)$, which are the Frobenius- and spectral-norm smoothness constants implied by~\eqref{eq:matrix_sandwich},
replaces the right side by $\tr(L)/\lambda_{\max}(L)$. For curvature $L\propto AA^\top$ generated by incoming activations
$A$ (up to transposition), this is their stable rank, and the result is the criterion of \citet{davis2025spectral}. Keeping
the matrix curvature makes~\eqref{eq:pol_beats_gd} exact for line search on quadratics, and the two criteria can disagree
in either direction. With $\ip{G}{S}=\norm{G}_1$ in place of $\norm{G}_{\nuc}$, the same comparison for sign descent is that of
\citet{balles2020geometrysigngradientdescent}, who find that it beats \algname{GD} when the Hessian is concentrated on its diagonal and
its largest eigenvalue is much larger than the average, again with scalar ($\ell_\infty$ and Euclidean) smoothness constants. \citet{shen2025convergence} and \citet{kovalev2025non} assume the same one-sided curvature bound for \algname{Muon}-type
methods. \citet{shen2025convergence} check along training trajectories the condition
$\norm{G}_{\nuc}^2/\norm{Q}_L^2\ge\norm{G}_\F^2/\lambda_{\max}(L)$, which bounds only \algname{GD}'s curvature by its worst case and is
therefore implied by~\eqref{eq:pol_beats_gd}, but not equivalent to it. \citet{kovalev2025non} analyze non-Euclidean \algname{SGD}
with momentum, using extrapolation or variance reduction for their faster stochastic rates, and obtain rates in $\tr(L)$
rather than a dimension factor times $\lambda_{\max}(L)$, which amounts to bounding both curvatures by their worst cases.

This one-step comparison is pessimistic for \algname{PolarGD} (its worst case is also worse than \algname{GD}'s;
Theorem~\ref{prop:precond:progress}), yet in Figure~\ref{fig:representative_plot} \algname{PolarGD} beats \algname{GD} by orders of
magnitude for every initialization on the dense and diagonal Hessians, and the gap widens with $\kappa(H)$
(Figure~\ref{fig:ls_vs_fixed}(c)); Section~\ref{sec:precond} studies this gap.

\section{One-Step Geometry}
\label{sec:one_step}

On matrix quadratics, matrix geometry, curvature, and initialization already interact nontrivially, while the exact line-search quantities of Section~\ref{sec:setup} remain available in closed form.
We consider
\begin{equation}
\label{eq:quad_model}
  f(X)=\tfrac12\tr\bigl((X-X^\star)^\T H (X-X^\star)\bigr),
\end{equation}
with $X\in\R^{d\times m}$, $H=H^\T\succeq 0$, and gradient $G=H(X-X^\star)$; we set $X^\star=0$ when convenient. Here \eqref{eq:matrix_sandwich} holds with equality for $L=M=H$, so by Proposition~\ref{prop:model_decrease} the exact line-search step and decrease along any $D$ with $\ip{G}{D}>0$ and $\tr(D^\T H D)>0$ are
\begin{equation}
\label{eq:quad_gamma_star}
  \gamma^\star(D)=\frac{\ip{G}{D}}{\tr(D^\T H D)},
  \qquad
  \Delta(D)= \frac{\ip{G}{D}^2}{2\,\tr(D^\T H D)}.
\end{equation}
\paragraph{Exact one-step decreases for \algname{GD}, \algname{PolarGD}, \algname{SignGD}.}
With $D_{\mathrm{GD}}=G$, $D_{\mathrm{Pol}}=Q$, and $D_{\mathrm{Sign}}=S$,
\begin{equation}
\label{eq:quad_deltas_three}
  \Delta_{\mathrm{GD}} = \frac{\norm{G}_\F^4}{2\,\tr(G^\T H G)},
  \qquad
  \Delta_{\mathrm{Pol}} = \frac{\norm{G}_{\nuc}^2}{2\,\tr(Q^\T H Q)},
  \qquad
  \Delta_{\mathrm{Sign}} = \frac{\norm{G}_1^2}{2\,\tr(S^\T H S)}.
\end{equation}
As in Section~\ref{sec:setup}, \algname{GD} weights the curvature along the left singular vectors of $G$ by
$\sigma_i^2$, whereas \algname{PolarGD} uses the unweighted curvature; \algname{SignGD} instead pays the coordinate-wise curvature
$\tr(S^\T HS)$, which can be favorable when $H$ is coordinate-aligned (Theorem~\ref{thm:onestep_regimes}).

\begin{remark}[Vector and rank-one gradients] If $m=1$ (vector case), then $\norm{G}_{\nuc}=\norm{G}_\F$ and $Q=G/\norm{G}_\F$,
so \algname{GD} and \algname{PolarGD} are collinear and therefore identical under exact line search.
More generally, if $\rank(G)=1$, then $Q$ is proportional to $G$ and \algname{GD}/\algname{PolarGD} coincide under exact line search.
\label{rmk:vector}
\end{remark}

\mybox[gray!8]{On our matrix quadratic model, the one-step progress and optimal step-size of \algname{GD}, \algname{PolarGD} and \algname{SignGD} can be written in closed form. How the algorithms rank depends on (1) \textit{Initialization}:  the evaluation point $X$, (2) \textit{Curvature}: the eigenvalues of $H$, (3) \textit{Geometry}: the eigenvectors of $H$.}

We first isolate the role of the iterate with an isotropic Hessian (Section~\ref{subsec:isotropic}).

\subsection{Isotropic Quadratic Setting}
\label{subsec:isotropic}

Consider $H=\lambda I_d$ with $\lambda>0$ and $X^\star=0$, so that $G=\lambda X$. Let $X=U\Sigma_X V^\T$ be a thin SVD of rank $r$, and let $s$ be the number of nonzero entries of $X$. By \eqref{eq:quad_gamma_star}, \algname{GD} uses $\gamma^\star(G)=1/\lambda$ and reaches $X^\star$ in one step from every $X$. \algname{PolarGD} uses $\gamma^\star(Q)=\norm{X}_{\nuc}/r$, so that
\[
  X^+_{\rm Pol}=U\Bigl(\Sigma_X-\frac{\norm{X}_{\nuc}}{r}I_r\Bigr)V^\T :
\]
it subtracts the \emph{same scalar from every nonzero singular value}. \algname{SignGD} analogously subtracts the same scalar from every nonzero entry, $X^+_{\rm Sign}=X-\frac{\norm{X}_1}{s}\sign(X)$. The resulting loss ratios are (App.~\ref{app:isotropic})
\begin{equation}
\label{eq:isotropic_ratios}
  \frac{f(X^+_{\rm Pol})}{f(X)} = 1-\frac{\norm{X}_{\nuc}^2}{r\norm{X}_\F^2},
  \qquad
  \frac{f(X^+_{\rm Sign})}{f(X)} = 1-\frac{\norm{X}_1^2}{s\norm{X}_\F^2}.
\end{equation}
Writing $\bar\sigma$ and $\overline{\sigma^2}$ for the means of the nonzero singular values of $X$ and of their squares, the \algname{PolarGD} ratio equals $1-\bar\sigma^2/\overline{\sigma^2}$, the relative spread of the singular values; the \algname{SignGD} ratio is the same quantity for the magnitudes of the nonzero entries. Hence \algname{PolarGD} is one-step exact if and only if all nonzero singular values of $X$ are equal ($X$ is a scaled partial isometry), and \algname{SignGD} if and only if all nonzero entries have equal magnitude. Thus, even for the best-conditioned Hessian, where \algname{GD} needs a single step from anywhere, the efficiency of \algname{PolarGD} and \algname{SignGD} depends on the geometry of the initialization (Figure~\ref{fig:representative_plot}, first column: \algname{GD} converges in one step from every initialization, \algname{PolarGD} only from the orthogonal one, and \algname{SignGD} only from the vertex one). Section~\ref{subsec:immediate} characterizes one-step exactness for general $H$.
\subsection{One-Step Convergence Analysis}
\label{subsec:immediate}

For general $H$, whether a method reaches $X^\star$ in one step depends jointly on $H$ and the error $E=X-X^\star$.

\begin{restatable}[One-step convergence regimes under exact line search on quadratics]{theorem}{OneStepRegimes}
\label{thm:onestep_regimes}
Consider \eqref{eq:quad_model} with $H\succ 0$, let $E:=X-X^\star$, and call a direction $D\in\{G,Q,S\}$ \emph{one-step exact} at $X\neq X^\star$ if $X-\gamma^\star(D)D=X^\star$.
\begin{enumerate}
\item[\textup{(i)}]
\algname{GD} is one-step exact from \emph{every} $X$ if and only if $H=\lambda I_d$ for some $\lambda>0$.
(For a fixed $X\neq X^\star$, \algname{GD} is one-step exact if and only if $HE=\lambda E$ for some $\lambda>0$.)

\item[\textup{(ii)}] 
\algname{SignGD} is one-step exact at $X$ if and only if all nonzero entries of $E$ have a common magnitude and $\sign(HE)=\sign(E)$ entrywise, including zero entries.
In particular, if $H$ is strictly diagonally dominant with positive diagonal, i.e.
$H_{ii}>\sum_{j\neq i}|H_{ij}|$ for all $i$, then $\sign(HS)=S$ for every $S\in\{\pm1\}^{d\times m}$;
hence any \emph{vertex initialization} $E=cS$ with $c>0$ yields one-step convergence.

\item[\textup{(iii)}] 
Assume $\rank(G)=d \le m$. \algname{PolarGD} is one-step exact at $X$ if and only if $EE^\T = c\,I_d$ for some $c>0$, that is, $E=\sqrt{c}\,U V^\T$ with $U\in\R^{d\times d}$ orthogonal and $V\in\R^{m\times d}$ having orthonormal columns (a scaled orthogonal matrix if $m=d$).
\end{enumerate}
\end{restatable}

\mybox[gray!8]{One-step analysis on quadratics reveals a ``preference'' for each method: \algname{GD} converges in one step from every initialization if and only if the Hessian has equal eigenvalues. \algname{SignGD} relaxes this condition on the Hessian (diagonal dominance suffices), at the price of a vertex initialization. \algname{PolarGD} \textit{requires no condition on the Hessian}, provided the rows of the initial error $X_0-X^\star$ are orthogonal with equal norms (e.g., a scaled orthogonal matrix). This is why orthogonal initialization is special for \algname{PolarGD} in Figure~\ref{fig:representative_plot}, and vertex initialization for \algname{SignGD} on diagonal Hessians.}

\subsection{Rotation Equivariance and Curvature Adaptation}
\label{sec:curvature}

While most trends in Figure~\ref{fig:representative_plot} correlate with our one-step characterization, another pattern stands out: for every initialization, \algname{PolarGD} performs the same on the dense and on the diagonal Hessian, in stark contrast to \algname{SignGD}. The following result, proven in App.~\ref{subsec:orth_equiv}, explains this \textit{beyond the line-search setup}: the dense Hessian is a rotation of the diagonal one, so for initializations with rotation-invariant distributions (Gaussian, orthogonal) the two problems are equivalent in distribution for \algname{GD} and \algname{PolarGD}; the vertex case matches empirically.

\begin{restatable}[Orthogonal equivariance]{theorem}{Equivariance}
    Let $P\in\R^{d\times d}$ and $R\in\R^{m\times m}$ be orthogonal and $\tilde H:=P^\T HP$. Run \algname{GD} or \algname{PolarGD}, with exact line search or the same fixed (LMO-normalized or dual-adapted) step sizes, on \eqref{eq:quad_model} with Hessian $H$ from the error $E_0:=X_0-X^\star$, and with Hessian $\tilde H$ from $\tilde E_0:=P^\T E_0R$. Then the errors satisfy $E_k=P\tilde E_kR^\T$ and the losses coincide for all $k$. \algname{SignGD} does not have this property (App.~\ref{subsec:orth_equiv}).
    \label{thm:orthogonal}
\end{restatable}
The same holds for \algname{Muon}'s update with momentum and finite \approxalg{Newton--Schulz} (deterministic gradients), since
momentum commutes with the rotation and $\mathrm{NS}_q(P^\T MR)=P^\T\mathrm{NS}_q(M)R$.

\begin{figure}[!b]
\centering
\begin{minipage}[c]{0.66\linewidth}
\includegraphics[width=\linewidth]{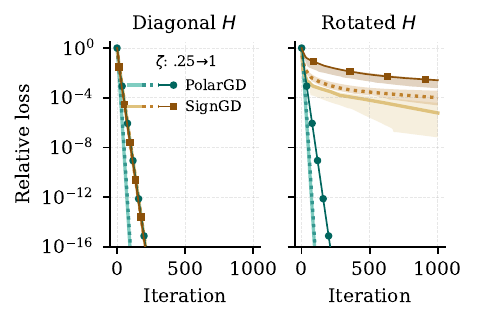}
\end{minipage}\hfill
\begin{minipage}[c]{0.31\linewidth}
\caption{Zipf quadratic with $\lambda_k(H) \propto k^{-\zeta}$, $\zeta \in \{0.25, 0.5, 1\}$, for diagonal and randomly rotated $H$ with the target $X^\star$ kept fixed; best-tuned dual-adapted fixed steps, median and range over 5 seeds. \algname{PolarGD} is unaffected by the rotation, while \algname{SignGD} degrades significantly. Details in App.~\ref{app:exp_details:zipf}.}
\label{fig:zipf_rotate}
\end{minipage}
\end{figure}

\paragraph{Zipf Laws experiment.} \citet{kunstner2025scalinglawsgradientdescent} model next-token prediction as linear regression over one-hot inputs, with squared loss $L(X)=\frac{1}{2n}\|ZX-Y\|_\F^2$. Since one-hot inputs are orthogonal, \(Z^\top Z\) is diagonal with entries given by token counts; under Zipf frequencies \(\pi_k\propto k^{-\zeta}\), this gives~\eqref{eq:quad_model} with $H=\diag(\pi_1,\dots,\pi_d)$ and a target $X^\star$ of next-token probabilities. On such ill-conditioned diagonal problems, \algname{SignGD} scales much more favorably than \algname{GD} \citep{kunstner2025scalinglawsgradientdescent}; heavy-tailed frequencies of this kind slow \algname{GD} on rare classes but not sign-based methods, which \citet{kunstner2024heavy} identify as a reason for \algname{Adam}'s advantage on language models. Figure~\ref{fig:zipf_rotate} shows that \algname{PolarGD} behaves similarly in the diagonal case (if $H$ and $X-X^\star$ are diagonal and $d=m$, then $Q=\sign(G)$; App.~\ref{app:diag_quadratic}) and that, unlike \algname{SignGD}, it keeps its performance when $H$ is randomly rotated; the advantage of sign descent requires curvature concentrated on the diagonal \citep{balles2020geometrysigngradientdescent}, which a rotation removes. Theorem~\ref{thm:orthogonal} guarantees this when $H$ and $X^\star$ are rotated together; Figure~\ref{fig:zipf_rotate} keeps $X^\star$ fixed and confirms it empirically, and App.~\ref{app:exp_details:zipf} shows both cases. This is consistent with, though not by itself an explanation of, the strong performance of \algname{Muon} in language modeling.

\paragraph{From one step to trajectories.} Theorem~\ref{thm:onestep_regimes} explains why row-isotropic iterates are special, and Theorem~\ref{thm:orthogonal} why \algname{PolarGD} behaves the same on diagonal and rotated Hessians. Neither explains the stronger phenomenon in Figures~\ref{fig:representative_plot} and~\ref{fig:zipf_rotate}: \algname{PolarGD} stays fast along the whole trajectory when $H$ is dense and ill-conditioned. Section~\ref{sec:precond} shows that the polar direction is a gradient step preconditioned by the iterate covariance $XX^\top$ rather than by $H$, which explains why row-isotropic iterates can be fast for any $\kappa(H)$ and why near-rank-deficient iterates are slow.

\mybox[gray!8]{On our quadratic loss, \algname{PolarGD} is empirically as robust to ill-conditioning as \algname{SignGD} (explained in Section~\ref{sec:precond}). Unlike \algname{SignGD}, it is orthogonally equivariant, so its performance carries over to every rotation of a problem, not only to coordinate-aligned ones.}

\section{Beyond One Step: Implicit Preconditioning}
\label{sec:precond}

The polar direction can be rewritten as a gradient step with a preconditioner set by the row covariance $XX^\top$
of the current iterate rather than by $H$ alone. This yields an optimistic regime, where row-isotropic iterates
cancel the curvature, and a pessimistic one, where tiny singular values make the preconditioner ill-conditioned.
We first note that one-step preferences can change along a trajectory (see also Appendix~\ref{app:2x2_alternation}).

\label{sec:alternations}
\paragraph{Local preferences need not be consistent along a trajectory.}
Even on convex quadratics, the sign of $\Delta_{\mathrm{Pol}}(X)-\Delta_{\mathrm{GD}}(X)$ can change with $X$, so pointwise one-step comparisons do not induce a global ordering between methods. Appendix~\ref{app:2x2_alternation} gives a closed-form $2\times 2$ diagonal example in which, along the exact-line-search \algname{GD} trajectory $\{X_k^{\mathrm{GD}}\}_{k\ge 0}$, the local winner between \algname{GD} and \algname{PolarGD} alternates deterministically at every iterate. \citet{gonon2026insights} give a numerical example of a related failure: a greedy rule that picks the better of the two directions chooses the polar direction at every step and is still much slower than \algname{GD} with line search. Both examples call for the trajectory-level analysis below.

\subsection{PolarGD as an Implicit Preconditioned Method}
\label{subsec:precond:polar_precond}

We return to the quadratic setting of Section~\ref{sec:one_step}, with \(H\succ0\). For the thin SVD
\(G=U\Sigma_GV^\top\) of Section~\ref{sec:setup}, of rank \(r\),
\begin{equation}
\label{eq:precond:polar_whiten_compact}
  Q=\polar(G)=(GG^\top)^{\dagger/2}G,
  \qquad
  QQ^\top=\Pi_{\operatorname{range}(G)},
\end{equation}
so polar normalization sets all singular values of \(G\) to one. Writing \(\Sigma:=XX^\top\) and using \(G=HX\),
\begin{equation}
\label{eq:precond:A_def_compact}
  Q=\polar(HX)=(H\Sigma H)^{\dagger/2}HX=:A(X)X,
  \qquad
  A(X):=(H\Sigma H)^{\dagger/2}H.
\end{equation}
So \algname{PolarGD} applies a left preconditioner determined by the current iterate.

\begin{figure}[!b]
\centering
\begin{minipage}[c]{0.58\linewidth}
\includegraphics[width=\linewidth]{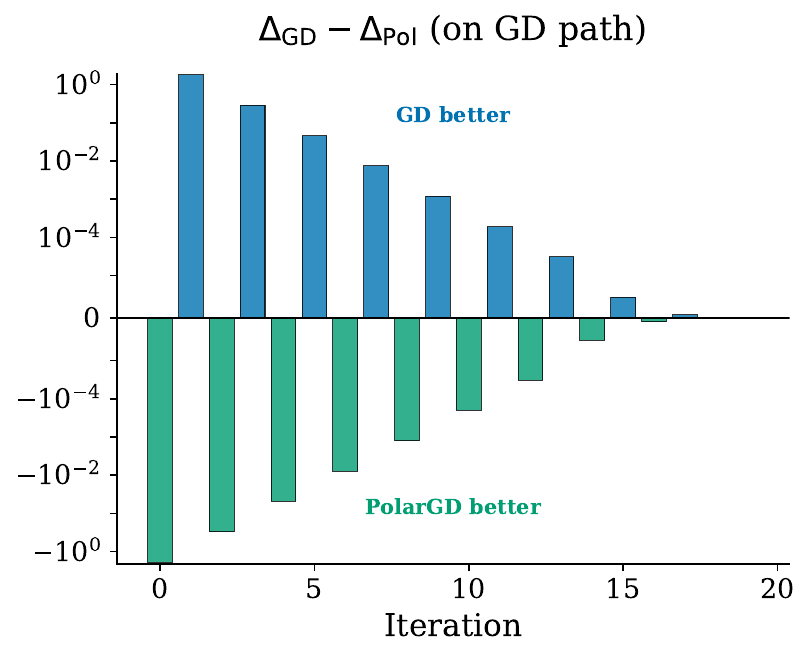}
\end{minipage}\hfill
\begin{minipage}[c]{0.39\linewidth}
\caption{\textbf{Alternating local winner.} A $2\times 2$ diagonal quadratic with $\kappa(H)=10$ (details in App.~\ref{app:2x2_alternation}). Along the \algname{GD} trajectory, the one-step advantage $\Delta_{\mathrm{GD}} - \Delta_{\mathrm{Pol}}$ alternates sign every iteration. Loss curves are shown in Figure~\ref{fig:alternating}.
}
\label{fig:alternating_winner}
\end{minipage}
\end{figure}

The cleanest identities occur in the full-row-rank regime \(d\le m\), \(\rank(X)=d\). We state the
main results in this regime and give the active-subspace extension below.

\begin{restatable}[Conditioning of the implicit preconditioner]{proposition}{PropConditioning}
\label{prop:precond:conditioning}
Assume \(H\succ0\) and \(\rank(X)=d\).
\begin{equation}
\label{eq:precond:conditioning_identities}
  A(X)\Sigma A(X)^\top=I_d,
  \qquad
  A(X)^\top A(X)=\Sigma^{-1}.
\end{equation}
Consequently, the singular values of \(A(X)\) are \(\{1/\sigma_i(X)\}_{i=1}^d\), and
\begin{equation}
\label{eq:precond:kappaA}
  \kappa(A(X))=\kappa(X)=\sqrt{\kappa(\Sigma)}.
\end{equation}
\end{restatable}

This is the first qualitative difference from \algname{GD}. For \algname{GD} on a quadratic, the linear rate is governed by \(\kappa(H)\). For \algname{PolarGD}, the conditioning of the implicit preconditioner is governed by the current iterate covariance. Thus the Hessian condition number alone is not predictive: a badly conditioned \(H\) can be harmless near row-isotropic iterates, while a well-conditioned \(H\) can be paired with an ill-conditioned preconditioner if \(X\) has tiny singular values.

In full row rank, \(QQ^\top=I_d\), hence $\tr(Q^\top H Q)=\tr(H)$.
Combining this with the line-search identity from Section~\ref{sec:one_step} yields the following
exact contraction.

\begin{restatable}[Exact progress coefficient for \algname{PolarGD}]{theorem}{PropProgress}
\label{prop:precond:progress}
Assume \(H\succ0\) and \(\rank(X)=d\). Let \(Q=\polar(HX)\), and let
\(X^+=X-\gamma^\star(Q)Q\) be the exact line-search update. Then
\begin{equation}
\label{eq:precond:gamma_star}
  \gamma^\star(Q)=\frac{\|HX\|_{\nuc}}{\tr(H)},
  \qquad
  f(X^+)=f(X)-\frac{\|HX\|_{\nuc}^2}{2\tr(H)}.
\end{equation}
Equivalently,
\begin{equation}
\label{eq:precond:progress_coeff}
  \frac{f(X^+)}{f(X)}=1-p(X),
  \qquad
  p(X):=\frac{\|HX\|_{\nuc}^2}{\tr(H)\tr(HXX^\top)}.
\end{equation}
Moreover,
\begin{equation}
\label{eq:precond:p_bounds}
  \max\left\{
    \frac{\lambda_{\min}(H)}{\tr(H)},\,
    \frac{4\kappa(X)}{(1+\kappa(X))^2}
  \right\}
  \le p(X)\le 1,
\end{equation}
where both lower bounds are tight, and the upper bound is attained if and only if \(XX^\top=cI_d\), in which case \algname{PolarGD} reaches the
minimizer in one step. Finally, for a fixed dual-adapted step \(X^+=X-\gamma\norm{HX}_{\nuc}Q\),
\begin{equation}
\label{eq:precond:fixed_step}
  \frac{f(X^+)}{f(X)}=1-\alpha(2-\alpha)\,p(X),
  \qquad
  \alpha:=\gamma\tr(H).
\end{equation}
\end{restatable}
The full-rank assumption holds approximately in all nine regimes of Figure~\ref{fig:representative_plot}: the numerical rank of the gradients stays close to $d$ along the trajectories (Figure~\ref{fig:rank_track}). The coefficient \(p(X)\) is scale-invariant, so it measures the shape of the iterate rather than its magnitude. Row-isotropic iterates give \(p(X)=1\), recovering Theorem~\ref{thm:onestep_regimes}(iii). The second lower bound in~\eqref{eq:precond:p_bounds} behaves like \(4/\kappa(X)\) for ill-conditioned iterates. It is the Kantorovich factor \(4\kappa(H)/(1+\kappa(H))^2\approx4/\kappa(H)\) of \algname{GD} with exact line search, with \(\kappa(H)\) replaced by \(\kappa(X)\): in the worst case, \algname{GD} is limited by the conditioning of the curvature and \algname{PolarGD} by that of the iterate (Remark~\ref{rem:cv} gives a coefficient-of-variation form of \(p(X)\)). On dense quadratics, \algname{GD}'s iteration count grows roughly linearly with \(\kappa(H)\), while tuned \algname{PolarGD}'s grows only slowly (Figure~\ref{fig:ls_vs_fixed}(c)). The first lower bound, however, can be as small as \(1/(1+(d-1)\kappa)\approx 1/(\kappa d)\), roughly \(d/4\) times smaller than \(4\kappa/(\kappa+1)^2\approx 4/\kappa\). It is tight: Appendix~\ref{app:pessimistic_polar_bound} constructs a near-rank-deficient full-rank family with \(p(X)\to \lambda_{\min}(H)/\tr(H)\), on which \algname{GD} is nearly one-step exact, while \algname{PolarGD} contracts at exactly \(\bigl(\frac{(d-1)\kappa-1}{(d-1)\kappa+1}\bigr)^2\) per step from the second step on, \algname{GD}'s worst-case rate with \(\kappa\) replaced by \((d-1)\kappa\) (Corollary~\ref{cor:two_cluster}). Figure~\ref{fig:ns_abl} shows such an instance with $H=I_d$: \algname{GD} converges in one step, whereas exact \algname{PolarGD} contracts by a fixed factor per step.
Thus exact line search alone does not make \algname{PolarGD} uniformly better than \algname{GD}; its advantage depends on how the iterate covariance evolves along the trajectory. By~\eqref{eq:precond:fixed_step}, the same coefficient governs the tuned fixed steps of Figure~\ref{fig:representative_plot}; line search is the rate \(\alpha=1\), which maximizes the one-step factor but, as Section~\ref{sec:alternations} illustrates, need not give the best trajectory. The adverse example also motivates the softened polar variants below: the bad regime is caused by exact polar normalization of tiny singular modes.

\begin{figure}[t]
    \centering
    \includegraphics[width=0.9\linewidth]{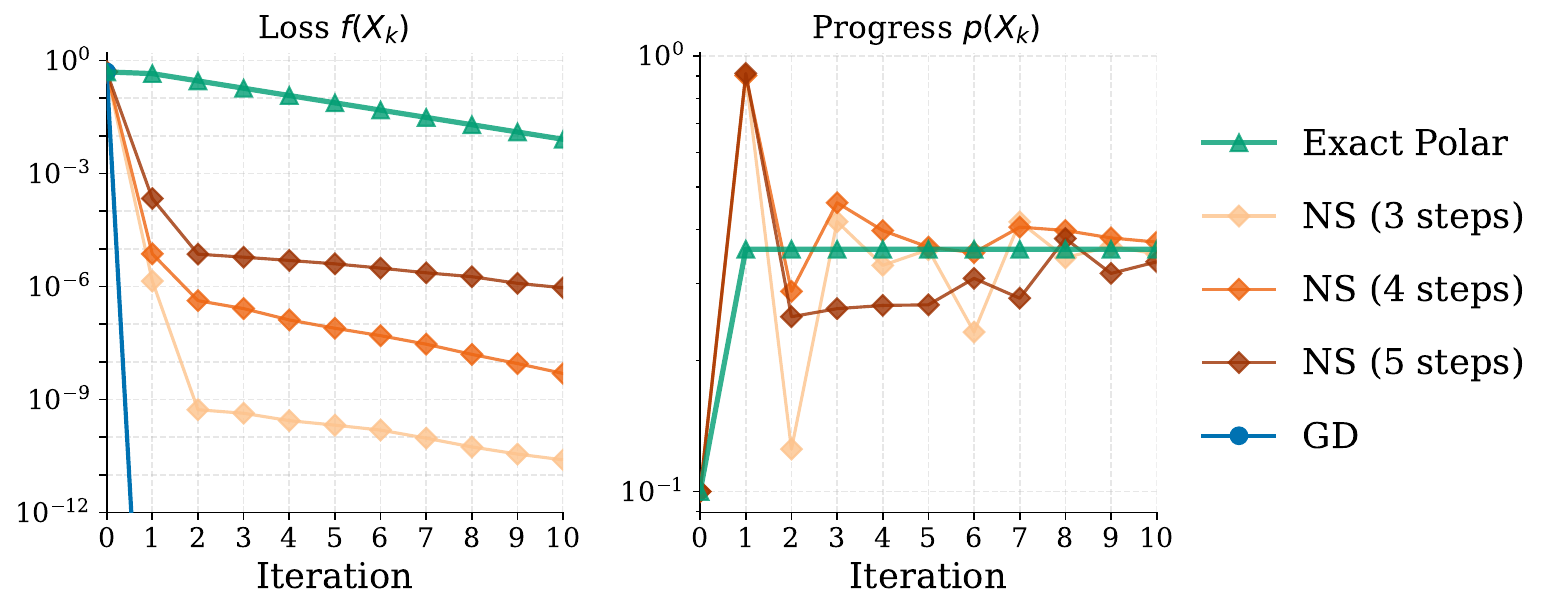}
    \caption{\textbf{Spectral recovery near rank deficiency.}
    \textit{(Left)} Loss $f(X_k)$, \textit{(Right)} exact progress coefficient $p(X_k)$. We initialize at $X_0 = O^{\T}\diag(1, \delta, \ldots, \delta)O$ for $\delta=10^{-4}$, $O$ random orthogonal, $H=I_d$ and $d=10$; with $H=I_d$, \algname{GD} is one-step exact, so this setting isolates the effect of the iterate's spectrum on \algname{PolarGD}. All methods use exact line search. Exact \algname{PolarGD} contracts at the fixed rate \((1-2/d)^2\) per step after the first (Corollary~\ref{cor:two_cluster}), while \approxalg{Newton--Schulz} damps the tiny singular values and recovers quickly, the more so with fewer steps. \algname{GD} is shown for reference. Further details in Appendix~\ref{app:exp_details}.
    }
    \label{fig:ns_abl}
\end{figure}

\paragraph{Relation to prior preconditioning views.}
Writing $\polar(G)=(GG^\top)^{-1/2}G$ as a gradient-built preconditioner is common \citep{lau2025polargrad,ma2026preconditioning}; on quadratics, substituting $G=HX$ shows that this preconditioner is set by the iterate, with conditioning $\kappa(X)$ whatever $H$ is (Proposition~\ref{prop:precond:conditioning}). Under smoothness and the Polyak--{\L}ojasiewicz condition, \citet{lau2025polargrad} prove linear rates for the dual-adapted step that depend on $\kappa(H)$, such as $1-1/(r\kappa(H))$ for gradients of rank $r$; on quadratics, the first lower bound in~\eqref{eq:precond:p_bounds} has this form and is tight, so this dependence cannot be removed in the worst case, while the second bound depends on $\kappa(X)$ alone. When $H$ and $XX^\top$ commute, $Q=\polar(X)$ and $H$ drops out of the direction (Corollary~\ref{cor:precond:commuting}). \citet{ma2026preconditioning} exploit decoupling of this kind, along trajectories aligned with the eigenbasis of the problem and with geometrically decaying steps, to prove rates independent of the condition number for matrix factorization and for an in-context-learning quadratic, and leave matrix linear regression, the setting studied here, open.

\paragraph{Greedy steps are not the best fixed steps.}
For \algname{GD}, exact line search typically does at least as well as a tuned fixed step; for \algname{PolarGD} and
\algname{SignGD} it does not (Figure~\ref{fig:ls_vs_fixed}; Appendix~\ref{app:dim_ret} gives further settings). In the
dual-adapted parameterization, line search is the fixed rate \(\alpha=1\) in~\eqref{eq:precond:fixed_step}, which
maximizes the one-step factor, yet the tuned rate \(\alpha\approx1.7\) converges much faster: it keeps \(p(X_k)\) near
\(0.6\), whereas under line search \(p(X_k)\) decays to \(0.01\) within \(150\) iterations. Line search itself creates the
tiny singular values that make \(p(X)\) small. When \(H\) and \(XX^\top\) commute, a fixed step maps the singular values
of \(X\) to \(s_i^+=|s_i-\alpha\bar s_H|\), where \(\bar s_H\) is their \(\lambda\)-weighted mean
(Corollary~\ref{cor:precond:commuting}). Line search subtracts the mean, which lies in the bulk of the spectrum, so
singular values near it are sent close to zero and a growing tail of tiny ones forms; a longer step instead reflects
small singular values to large ones. As in Figure~\ref{fig:alternating_winner}, the locally optimal choice is not
optimal along the trajectory, and the gap grows with \(\kappa(H)\) (Figure~\ref{fig:ls_vs_fixed}(c)).

\mybox[gray!8]{On our quadratic loss, \algname{PolarGD} acts like a self-preconditioned \algname{GD} step: the polar
map sets all singular values of the gradient to one, so the effective preconditioner is controlled by the iterate
covariance \(XX^\top\), not by \(\kappa(H)\) alone. This helps near row-isotropic iterates, but can
be pessimistic near rank deficiency, where exact polar normalization gives tiny singular modes the
same weight as large ones.}

\paragraph{Active-subspace extension.}
When \(\rank(G)<d\), \(QQ^\top\) is no longer \(I_d\) but the projector
\(\Pi_G:=\Pi_{\operatorname{range}(G)}\). The exact line-search formula remains valid with active curvature
\(\tr(H\Pi_G)\):
\begin{equation}
\label{eq:precond:active_progress_main}
  \gamma^\star(Q)=\frac{\|G\|_{\nuc}}{\tr(H\Pi_G)},
  \qquad
  f(X^+)=f(X)-\frac{\|G\|_{\nuc}^2}{2\tr(H\Pi_G)}.
\end{equation}
Moreover, like \algname{GD} on this model, exact \algname{PolarGD} acts on \(X\) from the left and cannot increase rank: for any step size, \(\rank(X_{k+1})\le \rank(X_k)\) (Proposition~\ref{prop:rank_nonincreasing}).
Thus rank deficiency is an active-subspace regime rather than a failure of the theory. Appendix~\ref{app:rank_deficient_active_subspace} gives the full projector-based extension.
\subsection{Near-Rank-Deficiency and Softened Polar Updates}
\label{subsec:precond:softened_updates}

The representation \(Q=(GG^\top)^{\dagger/2}G\) also exposes a failure mode of exact polar
normalization: tiny nonzero singular values are inverted. Two simple variants reduce this sensitivity:
rank-truncated polar,
$Q^{(s)}:=U_{:,1:s}V_{:,1:s}^\top$,
which drops small modes, and damped polar,
\[
  Q_\varepsilon:=(GG^\top+\varepsilon I_d)^{-1/2}G,
\]
which maps each singular value to \(\sigma_i/\sqrt{\sigma_i^2+\varepsilon}\). Up to the factor \(\sqrt\varepsilon\),
damped polar is the soft spectral clipping \((I_d+GG^\top/c^2)^{-1/2}G\) with \(c=\sqrt\varepsilon\), a smooth version of
\(\sigma_i\mapsto\min(\sigma_i,c)\), which \citet{jiang2026enhancing} apply to the updates of \algname{AdamW}-type optimizers; there
\(c\) is large and only the largest singular values are capped, whereas \algname{Muon} operates in the opposite regime. We define
these variants and give basic properties in Appendix~\ref{app:trunc_damp}.

\paragraph{Finite \approxalg{Newton--Schulz} as damped polar.} 
Practical \algname{Muon} replaces the SVD by $q$ \approxalg{Newton--Schulz} steps on $G/\norm{G}_\F$, which keep the singular vectors of
$G$ and map each normalized singular value $s$ to $\phi_q(s)$ (Appendix~\ref{app:newton_schulz_filter}). Small
singular values are multiplied by a fixed factor $\ell_q$, which grows geometrically with $q$ ($\ell_5\approx485$ for
\algname{Muon}'s coefficients), while those above about $1/\ell_q$ are mapped close to $1$. This is the behavior of damped polar
with $\varepsilon_q=\ell_q^{-2}$, i.e., a rescaled soft spectral clipping at $c=1/\ell_q$: finite \approxalg{Newton--Schulz} acts like \algname{GD} on singular values below the damping scale
$1/\ell_q$ and like exact polar above it, and its damping decreases geometrically with the number of steps. Near rank
deficiency it therefore avoids the slow regime of exact polar, and fewer steps converge faster
(Figure~\ref{fig:ns_abl}). This contrasts with convergence bounds that treat the \approxalg{Newton--Schulz} error only as a degradation factor
\citep{shulgin2026beyond,kim2026convergence}: here a less accurate polar factor converges faster. \citet{gonon2026insights}
observe numerically that five \approxalg{Newton--Schulz} steps can beat exact polar on some quadratics; damping of tiny singular values
gives a deterministic mechanism for such gains.

\begin{figure}[t]
    \centering
    \includegraphics[width=0.95\linewidth]{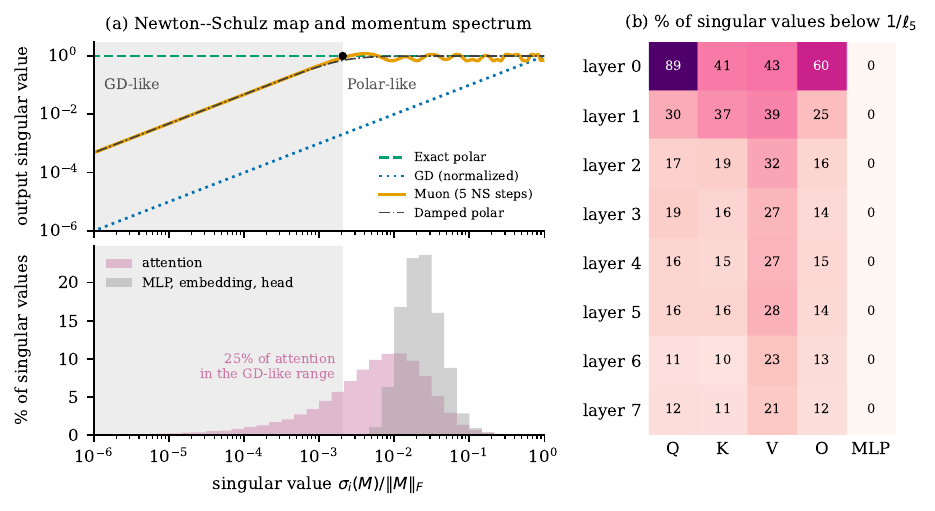}
    \caption{\textbf{\algname{Muon} departs from exact polar on a sizable part of the attention spectrum.} Momentum matrices $M$ that \algname{Muon} orthogonalizes in a 90M-parameter transformer (final checkpoint), with singular values normalized by $\norm{M}_\F$ as in training. \textit{(a)} Top: output vs.\ input singular value for exact polar, normalized \algname{GD}, \algname{Muon}'s five \approxalg{Newton--Schulz} (NS) steps, and damped polar with $\varepsilon=\ell_5^{-2}$; below the damping scale $1/\ell_5$ (shaded, \algname{GD}-like), \algname{Muon}'s output is proportional to its input, as for \algname{GD}, and above it \algname{Muon} behaves like exact polar. Bottom: distribution of the singular values. \textit{(b)} Percentage of singular values below $1/\ell_5$, by layer and matrix; the embedding and the LM head have none.}
    \label{fig:transformer}
\end{figure}

\paragraph{\approxalg{Newton--Schulz} in practice.} To check whether this regime occurs in practice, we examined a 90M-parameter
transformer trained with \algname{Muon} on 1.8B tokens (App.~\ref{app:transformer}). Figure~\ref{fig:transformer}(a) shows the
singular values of the momentum matrices $M$ that \algname{Muon} orthogonalizes, normalized by $\norm{M}_\F$ as in training.
About a quarter of the singular values of the attention matrices lie below the damping scale
$1/\ell_5\approx2\cdot10^{-3}$, where five \approxalg{Newton--Schulz} steps act like \algname{GD} rather than like exact polar; the MLP
matrices, the embedding, and the LM head have none. Over all 50 matrices, exact polar would place about 16\% of the
update's squared Frobenius norm on these directions and five \approxalg{Newton--Schulz} steps about 5\%. The effect is strongly
heterogeneous (Figure~\ref{fig:transformer}(b)): it is largest in the first layer and decreases with depth. As in
NanoGPT \citep[Figure~5]{shulgin2026beyond}, attention matrices are thus much harder to orthogonalize than MLP
matrices, and their smallest singular values explain why: \approxalg{Newton--Schulz} becomes polar-like on all directions only
once $\ell_q$ times the smallest normalized singular value exceeds $1$, which takes 9--18 steps for the attention
matrices and 4--5 for the MLP matrices.

\FloatBarrier
\section{Conclusion}
\label{sec:conclusion}
We isolated the geometry of the orthogonalized update shared by \algname{Muon} and related matrix optimizers by analyzing \algname{PolarGD} on matrix quadratics.
\algname{PolarGD} benefits from orthogonal equivariance and implicit preconditioning by the iterate, but it is not uniformly better than \algname{GD} or \algname{SignGD}: it is fast near row-isotropic iterates and slow near rank deficiency, where exact polar normalization overweights tiny singular values. Finite \approxalg{Newton--Schulz} softens the update exactly there and can then converge faster than exact polar, so it is not merely an approximation of the polar factor. The trajectory results rely on the quadratic structure, but Proposition~\ref{prop:model_decrease} only needs the curvature bounds~\eqref{eq:matrix_sandwich}, and equivariance and the softening by \approxalg{Newton--Schulz} hold for any objective, including the momentum matrix that \algname{Muon} orthogonalizes. Our transformer measurements show that the regime where exact polar and \approxalg{Newton--Schulz} differ occurs in practice. This suggests treating the number of \approxalg{Newton--Schulz} steps, or the damping level, as a design choice of the optimizer rather than only as an accuracy setting.

\label{subsec:precond:quadratic_precond}

\section*{Acknowledgements}
The research reported in this publication was supported by funding from King Abdullah University of Science and
Technology (KAUST): i) KAUST Baseline Research Scheme, ii) Center of Excellence for Generative AI, under award number
5940, iii) CRG Grant ORFS-CRG12-2024-6460. This research was partially
funded by the European Commission under the grant No.~101195233 (OpenEuroLLM). Antonio Orvieto acknowledges the
financial support of the Hector Foundation.

\bibliography{ref}

\clearpage

\appendix

\tableofcontents

\section*{Where the Main-Text Results Are Proved}
\begin{center}
\small
\begin{tabular}{@{}ll@{}}
\toprule
Main-text result & Appendix \\
\midrule
Proposition~\ref{prop:model_decrease} (one-step decrease along a direction) & \ref{app:onestep_proofs} \\
Isotropic case (Section~\ref{subsec:isotropic}) & \ref{app:isotropic} \\
Theorem~\ref{thm:onestep_regimes} (one-step exactness regimes) & \ref{app:proof_thm1} \\
Theorem~\ref{thm:orthogonal} (orthogonal equivariance) & \ref{subsec:orth_equiv} \\
$Q=\sign(G)$ for diagonal $H$ and $X-X^\star$ with $d=m$ (Section~\ref{sec:curvature}) & \ref{app:diag_quadratic} \\
Alternating local winner (Figure~\ref{fig:alternating_winner}) & \ref{app:2x2_alternation} \\
Proposition~\ref{prop:precond:conditioning} and Theorem~\ref{prop:precond:progress} (preconditioner, progress coefficient) & \ref{app:precond_proofs:identities} \\
Adverse family and its trajectory-level rate (Section~\ref{subsec:precond:polar_precond}) & \ref{app:pessimistic_polar_bound}, \ref{subsec:precond:root_dynamics} \\
Corollary~\ref{cor:precond:commuting} and the fixed-step update $s_i^+=|s_i-\alpha\bar s_H|$ (Section~\ref{subsec:precond:polar_precond}) & \ref{subsec:precond:root_dynamics} \\
Active-subspace formula and rank monotonicity (Section~\ref{subsec:precond:polar_precond}) & \ref{app:rank_deficient_active_subspace} \\
Rank-truncated and damped polar directions (Section~\ref{subsec:precond:softened_updates}) & \ref{app:trunc_damp} \\
Finite \approxalg{Newton--Schulz} as damped polar (Section~\ref{subsec:precond:softened_updates}) & \ref{app:newton_schulz_filter} \\
Transformer diagnostic (Section~\ref{subsec:precond:softened_updates}) & \ref{app:transformer} \\
\bottomrule
\end{tabular}
\end{center}

\section{Experimental Details}
\label{app:exp_details}
Below we provide implementation details of the experiments behind the main-text figures. Figure~\ref{fig:alternating_winner} is detailed in App.~\ref{app:2x2_alternation}, Figure~\ref{fig:transformer} in App.~\ref{app:transformer}, and Figure~\ref{fig:ls_vs_fixed} in App.~\ref{app:exp_details:ls_vs_fixed}.
\subsection{Learning Rate Selection}
\label{app:exp_details:lr_tune}
Unless stated otherwise, whenever we report results for a ``best-tuned'' learning rate, the following two-stage grid search is applied. In the first stage, we evaluate $50$ learning rates drawn from an evenly-spaced logarithmic grid over $[10^{-8}, 10^2]$. We select the candidate $\gamma_c$ that minimizes the loss at the final iterate among runs that did not diverge. In the second stage, we evaluate a further 50 points on a finer logarithmic interval $[\frac{\gamma_c}{10}, 10\gamma_c ]$. The final reported learning rate is the one whose run first crosses a threshold of $\varepsilon$ (in the least number of iterates). If no run reaches this threshold, we select the run which minimizes the loss at the final iterate. The tuned rates are dual-adapted (Section~\ref{sec:setup}), $X_{k+1} = X_k - \gamma\norm{G_k}_{\nuc}Q_k$ for \algname{PolarGD} and $X_{k+1} = X_k - \gamma\norm{G_k}_1S_k$ for \algname{SignGD}, except in Figure~\ref{fig:lineSearchPol}, which uses LMO-normalized steps $X_{k+1} = X_k - \gamma\,\mathrm{LMO}(G_k)$. 
\subsection{Initialization-Hessian Regimes (Figure~\ref{fig:representative_plot})}
\label{app:exp_details:fig1}
\paragraph{Hessian Structures.}
We test three specific curvature profiles to isolate the effects of condition number and coordinate alignment. To model Hessians arising in ill-conditioned, multi-output regression problems, we draw random matrices from the Marchenko--Pastur (MP) ensemble:
\[
    H = \frac{1}{N} Z^\T Z, \qquad Z \in \R^{N\times d} , \quad Z_{ij}\sim \mathcal{N}(0, 1) 
\]
Here the aspect ratio $d/N \in (0, 1]$ (the ratio of input features to samples) dictates the conditioning. Using this framework, we construct the following Hessian regimes
\begin{itemize}
    \item \textbf{Isotropic Hessian:} $H = I_d$. This regime has condition number $\kappa=1$.
    \item \textbf{Dense:} $H=\frac{1}{N} Z^\T Z$ with $d/N=0.8$.
    \item \textbf{Diagonal:} $H = \diag(\lambda_1, \ldots, \lambda_d)$, where $\lambda_i$ are the exact eigenvalues of the dense MP matrix.
\end{itemize}
\paragraph{Initialization Strategies.}
\begin{itemize}
    \item \textbf{Gaussian:} $X_0 \sim \mathcal{N}(0, 1)$ element-wise. This is the standard baseline, as in neural-network initialization.
    \item \textbf{Orthogonal:} $X_0$ is a random orthogonal matrix. 
    \item \textbf{Vertex:} $X_0 \in \{-1, 1 \}$ element-wise. Each entry is drawn randomly.
\end{itemize}

Within each Hessian-Initialization setting, the best learning rate for each optimizer is selected using the two-stage learning-rate grid of App.~\ref{app:exp_details:lr_tune} with threshold $\varepsilon = 10^{-8}$. We choose $d=100$ and $150$ iterations for each setting. 

To empirically validate the full-rank assumption of Theorem~\ref{prop:precond:progress}, we track $\rank(HX_k)$ at each iterate, computed in float32 precision via PyTorch \citep{paszke2019pytorchimperativestylehighperformance}.
\subsection{Zipfian Language-Modeling Proxy (Figure~\ref{fig:zipf_rotate})}
\label{app:exp_details:zipf}
To evaluate how these optimizers perform under the heavy-tailed geometries typical of language modeling, we construct a multi-output linear regression proxy inspired by \citet{kunstner2025scalinglawsgradientdescent}. 
\paragraph{Linear Bigram Model.}
Consider fitting a linear next-token prediction model $X\in \R^{d\times d}$ on $n$ token pairs $z_i, y_i \in \{0,1\}^d$ representing one-hot encodings. The squared loss is given by 
\[
L(X) = \frac{1}{2n} \| ZX - Y \|_\F^2,
\]
where $X_{ij}$ is an estimate of $p(\text{token}_j \mid \text{token}_i)$ for each $i,j \in [d]$. To simulate the heavy-tailed nature of language modeling, token frequencies follow a Zipfian distribution $\pi_k \propto k^{-\zeta}$ where $\zeta>0$ controls the heaviness of the tail.  Following their framework, we consider the equivalent quadratic problem 
\[
f(X) = \frac{1}{2} \tr((X-X^\star)^{\T} H (X - X^\star)),
\]
initialized at $X_0=0$, with Hessian $H = \diag(\pi_1, \ldots, \pi_d)$ and target $X^\star_{ij}=p(\text{token}_j \mid \text{token}_i)$. We draw each row of $X^\star$ independently from a Dirichlet distribution with parameter $100\,\pi$, so that the next-token distributions fluctuate around the Zipf frequencies; $X^\star$ then has full rank almost surely. A full-rank target matters: for a rank-one target such as $\mathbf 1\pi^\T$ (every row equal to $\pi$), every gradient $H(X_k-X^\star)$ from $X_0=0$ is rank one, so dual-adapted \algname{PolarGD} coincides with \algname{GD}.
\paragraph{Rotated Hessians.}
To test the rotation dependence of the optimizers, we additionally evaluate a rotated setting: $H = R^\T \diag(\pi) R$, where $R \in \mathbb{R}^{d \times d}$ is a random orthogonal matrix. This rotation preserves the spectrum while removing axis alignment. This approach follows recent work by \citet{zhang2025understandingadamrequiresbetter}, which demonstrates that \algname{Adam} is highly sensitive to random rotations.
We rotate the Hessian in two ways (Figure~\ref{fig:zipf_3panel}). (i)~As in Figure~\ref{fig:zipf_rotate}, the target $X^\star$ is kept fixed in the standard basis. For \algname{GD} and \algname{PolarGD}, this is equivalent to keeping $H$ diagonal and replacing the target by $RX^\star$, so it changes the problem and is not covered by Theorem~\ref{thm:orthogonal}. (ii)~$H$ and $X^\star$ rotated together, $X^\star\mapsto R^\T X^\star$: the same regression written in rotated feature coordinates, as with input embeddings $ZR$ in place of one-hot vectors. By Theorem~\ref{thm:orthogonal}, the \algname{GD} and \algname{PolarGD} trajectories are then exactly those of the diagonal case; in our runs, \algname{PolarGD} needs the same number of iterations to reach $10^{-15}$ as without rotation for every seed. \algname{SignGD} degrades in both cases: after $1000$ iterations, its relative loss is between $10^{-8}$ and $10^{-2}$, whereas without rotation it reaches $10^{-15}$ in about $200$ iterations.

\begin{figure}[h]
    \centering
    \includegraphics[width=0.9\linewidth]{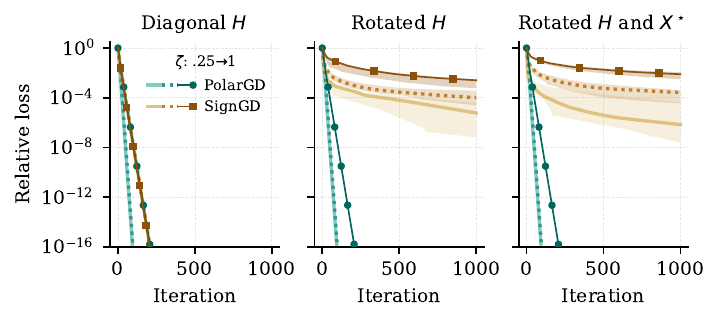}
    \caption{The Zipf quadratic of Figure~\ref{fig:zipf_rotate} with diagonal $H$, with $H$ rotated and $X^\star$ fixed (as in Figure~\ref{fig:zipf_rotate}), and with $H$ and $X^\star$ rotated together ($X^\star\mapsto R^\T X^\star$). Same tuning; median and range over 5 seeds. In the third panel, the \algname{PolarGD} curves coincide with those of the first by Theorem~\ref{thm:orthogonal}.}
    \label{fig:zipf_3panel}
\end{figure} 
We set $d=100$ and for each $\zeta \in \{0.25, 0.5, 1.0 \}$ and each of the three regimes, we tune the best dual-adapted learning rate with the procedure of App.~\ref{app:exp_details:lr_tune} with threshold $\varepsilon=10^{-15}$ on the relative loss $f(X_k)/f(X_0)$, separately for each of 5 seeds (each seed draws $X^\star$ and $R$), and run $1000$ iterations. The polar factor is computed by an SVD; along the reported trajectories, $\sigma_{\min}(G_k)/\sigma_{\max}(G_k)\ge10^{-8}$ until the relative loss reaches $10^{-15}$. 

\subsection{Spectral Recovery Near Rank Deficiency (Figure~\ref{fig:ns_abl})}
To illustrate that finite \approxalg{Newton--Schulz} behaves like damped polar, we compare it with exact polar normalization near rank deficiency.

\paragraph{Near-singular initialization.} 
We isolate the spectral recovery dynamics by setting the Hessian to the identity matrix $H=I$ in dimension $d=10$. We initialize the weights at a near rank-deficient state:
\[
X_0 = O^\T \diag(1, \delta, \ldots, \delta)O, \qquad O \in \R^{10\times 10} \ \text{random orthogonal}, \quad \delta = 10^{-4}
\]
\paragraph{Methods.}
We compare the \emph{exact} \algname{PolarGD} update against the practical \algname{Muon} update with $q\in\{3,4,5\}$ \approxalg{Newton--Schulz} steps to observe how tiny singular modes are handled. \approxalg{Newton--Schulz} uses \algname{Muon}'s coefficients $(a,b,c)=(3.4445,-4.7750,2.0315)$ and the Frobenius normalization of the reference implementation, all methods use exact line search, and we run $10$ iterations. We further track the progress coefficient $p(X_k)= \frac{\|HX_k\|_{\nuc}^2}{\tr(H)\tr(HX_kX_k^\top)}$ introduced in Section~\ref{subsec:precond:polar_precond}.
\begin{figure}
    \centering
    \includegraphics[width=0.9\linewidth]{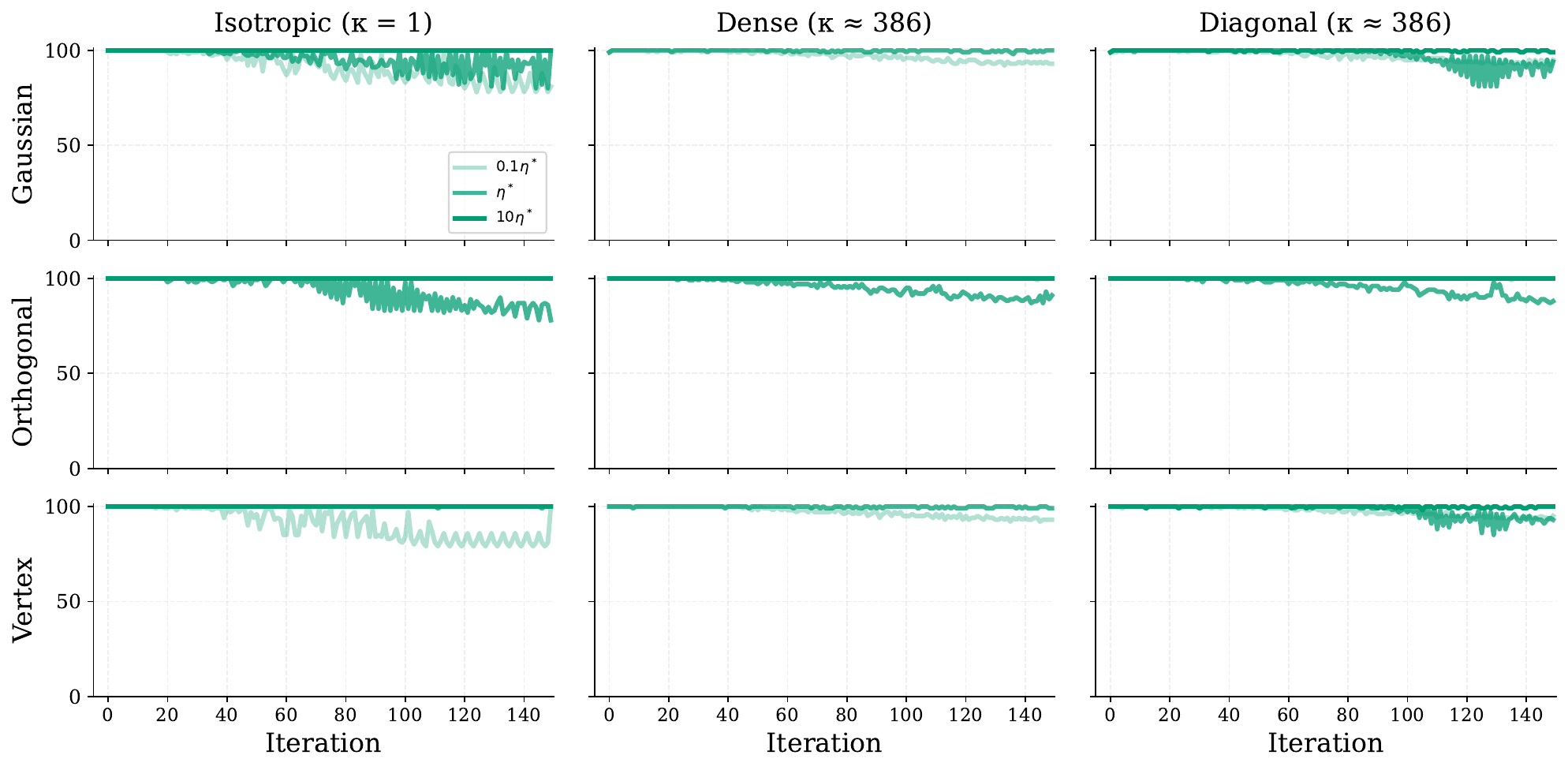}
    \caption{\textbf{Tracking the rank for \algname{PolarGD}.} For the experiments of App.~\ref{app:exp_details:fig1} (Figure~\ref{fig:representative_plot}), we track the rank of the gradient $G=HX_k$ for the best performing \algname{PolarGD} learning rate $\gamma_{\mathrm{best}}$ along with $0.1\gamma_{\mathrm{best}}$ and $10\gamma_{\mathrm{best}}$. The rank remains near full across the nine regimes, validating the full-rank assumption used in Theorem~\ref{prop:precond:progress}. }
    \label{fig:rank_track}
\end{figure}

\subsection{Line Search versus Tuned Dual-Adapted Steps (Figure~\ref{fig:ls_vs_fixed})}
\label{app:exp_details:ls_vs_fixed}
\algname{PolarGD} uses the dual-adapted step $X_{k+1}=X_k-(\alpha/\tr(H))\norm{G_k}_{\nuc}Q_k$, for which $\alpha=1$ is exact line search in full row rank (Theorem~\ref{prop:precond:progress}), and \algname{GD} uses either exact line search or the fixed step $X_{k+1}=X_k-(\eta/\lambda_{\max}(H))G_k$. The step sizes are tuned on the grids below instead of with the procedure of App.~\ref{app:exp_details:lr_tune}.

\paragraph{Panels (a) and (b).} We set $d=m=100$ and $H=Z^\T Z/N$ with $Z\in\R^{N\times d}$ standard Gaussian and $N=d/0.8$ (random seed $6$); $X_0$ has i.i.d.\ $\mathcal N(0,1)$ entries, and all methods run $150$ iterations. The \algname{PolarGD} rate is tuned over $\alpha\in\{0.2,0.3,\dots,1.0\}\cup\{1.05,1.10,\dots,1.95\}$ and the \algname{GD} step over $\eta$ on a $150$-point uniform grid in $[0.5,1.99]$, in both cases by the loss after $150$ iterations.

\paragraph{Panel (c).} We set $d=m=50$ and $H=O\diag(\lambda)O^\T$, where $O$ is the orthogonal factor of the QR decomposition of a Gaussian matrix and $\lambda$ is log-spaced from $1$ to $\kappa$, for $\kappa\in\{10,30,100,300,10^3,3\cdot10^3,10^4\}$; $X_0$ has i.i.d.\ $\mathcal N(0,1)$ entries, and we use $3$ seeds per $\kappa$. We report the number of iterations to reach $f(X_k)/f(X_0)\le10^{-6}$. The \algname{PolarGD} rate is tuned over $\alpha\in\{1.3,1.35,\dots,1.95\}$, with each run capped at $3\cdot10^4$ iterations, and the \algname{GD} fixed step $\eta/\lambda_{\max}(H)$ is scanned over $7$ evenly spaced values in $[0.85,1]\cdot2/(\lambda_{\min}(H)+\lambda_{\max}(H))$.

\section{Proofs for Sections~\getrefnumber{sec:setup} and~\getrefnumber{sec:one_step}}
\label{app:onestep_proofs}

\PropDecrease*
We prove a slightly stronger statement: if, in addition, $\norm{D}_M^2:=\ip{MD}{D}>0$, the exact step and decrease satisfy
\begin{equation}
\label{eq:exact_bracket}
  \frac{\ip{G}{D}}{\norm{D}_L^2} \le \gamma^\star(D) \le \frac{\ip{G}{D}}{\norm{D}_M^2},
  \qquad
  \frac{\ip{G}{D}^2}{2\norm{D}_L^2} \le \Delta(D) \le \frac{\ip{G}{D}^2}{2\norm{D}_M^2}.
\end{equation}
For $L=M$, the bounds coincide, which gives the equality case of Proposition~\ref{prop:model_decrease}.
\begin{proof}
Write $g:=\ip{G}{D}$, $c_L:=\norm{D}_L^2$, $c_M:=\norm{D}_M^2$, and $\phi(\gamma):=f(X-\gamma D)$ for $\gamma\ge0$. Applying~\eqref{eq:matrix_sandwich} at $X$ with direction $-\gamma D$ gives
\begin{equation}
\label{eq:app:phi_sandwich}
  f(X)-\gamma g+\tfrac{\gamma^2}{2}c_M \;\le\; \phi(\gamma) \;\le\; f(X)-\gamma g+\tfrac{\gamma^2}{2}c_L,
  \qquad \gamma\ge0.
\end{equation}

\emph{Step 1: the decrease~\eqref{eq:model_decrease}.} Evaluating the upper bound in~\eqref{eq:app:phi_sandwich} at $\gamma_L(D)=g/c_L$ gives $f(X^+)\le f(X)-g^2/c_L+g^2/(2c_L)=f(X)-g^2/(2c_L)$.

\emph{Step 2: derivative bounds.} We show that $-g+\gamma c_M\le\phi'(\gamma)\le-g+\gamma c_L$ for all $\gamma\ge0$; at $\gamma=0$ both hold with equality. For $\gamma>0$, let $Y:=X-\gamma D$, so that $\phi'(\gamma)=-\ip{\nabla f(Y)}{D}$. The lower bound in~\eqref{eq:matrix_sandwich}, applied at $X$ with direction $-\gamma D$ and at $Y$ with direction $+\gamma D$, gives
\[
  f(Y)\ge f(X)-\gamma g+\tfrac{\gamma^2}{2}c_M,
  \qquad
  f(X)\ge f(Y)+\gamma\ip{\nabla f(Y)}{D}+\tfrac{\gamma^2}{2}c_M,
\]
and adding the two yields $\phi'(\gamma)\ge-g+\gamma c_M$. The same argument with the upper bounds yields $\phi'(\gamma)\le-g+\gamma c_L$.

\emph{Step 3: the bracket~\eqref{eq:exact_bracket}.} Let $c_M>0$. By Step~2, $\phi'(\gamma)<0$ for $\gamma<g/c_L$ and $\phi'(\gamma)>0$ for $\gamma>g/c_M$, so every minimizer satisfies $g/c_L\le\gamma^\star(D)\le g/c_M$. Exact line search does at least as well as the step $\gamma_L(D)$, so $\Delta(D)\ge g^2/(2c_L)$ by Step~1; and minimizing the lower bound in~\eqref{eq:app:phi_sandwich} over $\gamma$ gives $\phi(\gamma)\ge f(X)-g^2/(2c_M)$ for all $\gamma\ge0$, hence $\Delta(D)\le g^2/(2c_M)$.
\end{proof}

\subsection{Isotropic Quadratics}
\label{app:isotropic}

Consider the isotropic setting $H=\lambda I_d$ with $\lambda>0$, so $f(X)=\frac{\lambda}{2}\norm{X}_\F^2$ and $G=\lambda X$.
\vspace{-2mm}
\begin{itemize}[leftmargin=1em,itemsep=0.2em]
    \item \textbf{\algname{GD}~(one-step exact).} From \eqref{eq:quad_gamma_star}, $\gamma^\star(G)=1/\lambda$ and $X^+=X-(1/\lambda)G=0$.
    \item \textbf{\algname{PolarGD} (additive shrinkage of singular values).} Let $r=\rank(X)=\rank(G)$ and write the thin SVD $X=U\Sigma_X V^\T$.
Then $Q=\polar(X)=UV^\T$ and $\norm{Q}_\F^2=r$. Exact line search yields
\[
\begin{aligned}
  \gamma^\star(Q) &= \frac{\ip{\lambda X}{Q}}{\lambda\ip{Q}{Q}}
  = \frac{\norm{X}_{\nuc}}{r},\\
  X^+ &= X - \frac{\norm{X}_{\nuc}}{r}\,\polar(X)\\
  &= U\Bigl(\Sigma_X - \frac{\norm{X}_{\nuc}}{r} I_r\Bigr)V^\T,
\end{aligned}
\]
so \algname{PolarGD} subtracts the \emph{same scalar from every nonzero singular value}.
The objective ratio is
\begin{equation}
\label{eq:isotropic_ratio_pol}
  \frac{f(X^+)}{f(X)} = 1 - \frac{\norm{X}_{\nuc}^2}{r\norm{X}_\F^2},
\end{equation}
which equals $0$ iff all nonzero singular values of $X$ are equal (i.e., $X$ is a scaled partial isometry).

\item \textbf{\algname{SignGD} (additive shrinkage of entries).} Let $S=\sign(X)$ and $s:=\norm{S}_\F^2$ (the number of nonzero entries, since $\sign(0)=0$).
Since $\ip{X}{\sign(X)}=\norm{X}_1$, line search yields
\[
  \gamma^\star(S)=\frac{\norm{X}_1}{s},
  \qquad
  X^+ = X - \frac{\norm{X}_1}{s}\,\sign(X),
\]
so \algname{SignGD} subtracts the \textit{same scalar from every nonzero entry}. The objective ratio is
\begin{equation}
\label{eq:isotropic_ratio_sign}
  \frac{f(X^+)}{f(X)} = 1 - \frac{\norm{X}_1^2}{s\norm{X}_\F^2},
\end{equation}
which equals $0$ iff all nonzero entries of $X$ have equal magnitude.
\end{itemize}
Table~\ref{tab:isotropic_summary} summarizes the gain, curvature, and one-step exactness condition of each direction.

\begin{table}[h]
\caption{Line search reveals distinct ``isotropy targets'' for different methods on isotropic quadratics.}
\centering
{\small
\begin{tabular}{@{}llll@{}}
\toprule
Direction $D$ & gain $\ip{G}{D}$ & curvature $\tr(D^\T H D)$ & one-step exactness for $H=\lambda I$ \\
\midrule
\algname{GD} ($D=G$)      & $\norm{G}_\F^2$ & $\lambda\norm{G}_\F^2$          & always (any $X$) \\
\algname{PolarGD} ($D=Q$) & $\norm{G}_{\nuc}$    & $\lambda\norm{Q}_\F^2=\lambda\rank(G)$ & iff nonzero singular values of $X$ are equal \\
\algname{SignGD} ($D=S$)  & $\norm{G}_1$    & $\lambda\norm{S}_\F^2$          & iff nonzero entries of $X$ have equal magnitude \\
\bottomrule
\end{tabular}
}
\label{tab:isotropic_summary}
\end{table}

\subsection{Exact Line-Search Quantities for Diagonal $H$}
\label{app:diag_quadratic}

We specialize the exact line-search identities to diagonal curvature
$H=\diag(\lambda_1,\dots,\lambda_d)$ with $\lambda_i\ge0$, the natural regime for coordinate methods such as \algname{SignGD}.
Work with the error $E:=X-X^\star\in\R^{d\times m}$ and write its rows as $E=[e_1^\T;\dots;e_d^\T]$ with $e_i\in\R^m$. Then
\[
f(E)=\frac12\sum_{i=1}^d \lambda_i\|e_i\|_2^2,
\qquad
G=HE,\quad g_i=\lambda_i e_i .
\]
All formulas below hold for any shape $d\times m$; shape or rank assumptions are stated where needed.

\paragraph{\algname{GD} (direction $D=G$).}
Using $\ip{G}{G}=\|G\|_\F^2$ and $\tr(G^\T H G)$, the diagonal structure yields
\[
\|G\|_\F^2
= \sum_{i=1}^d \|g_i\|_2^2
= \sum_{i=1}^d \lambda_i^2\|e_i\|_2^2,
\qquad
\tr(G^\T H G)
= \tr(E^\T H^3 E)
= \sum_{i=1}^d \lambda_i^3\|e_i\|_2^2.
\]
Therefore the exact line-search step size and one-step decrease are
\begin{equation}
\label{eq:diag_gd_step_delta}
\gamma^\star_{\mathrm{GD}}
= \frac{\|G\|_\F^2}{\tr(G^\T H G)}
= \frac{\sum_i \lambda_i^2\|e_i\|_2^2}{\sum_i \lambda_i^3\|e_i\|_2^2},
\qquad
\Delta_{\mathrm{GD}}
= \frac{\|G\|_\F^4}{2\,\tr(G^\T H G)}
= \frac{\Bigl(\sum_i \lambda_i^2\|e_i\|_2^2\Bigr)^2}{2\sum_i \lambda_i^3\|e_i\|_2^2}.
\end{equation}
Moreover, the \algname{GD} update remains rowwise-scaling (coupled only through the shared $\gamma^\star_{\mathrm{GD}}$):
\[
e_i^+ = (1-\gamma^\star_{\mathrm{GD}}\lambda_i)\,e_i,\qquad i=1,\dots,d.
\]

\paragraph{\algname{SignGD} (direction $D=S=\sign(G)$).}
Let $S=\sign(G)$ be the entrywise sign matrix with $\sign(0)=0$, and let $s_i^\T$ denote its $i$th row.
For diagonal $H$ with nonnegative diagonal, $S$ is coordinate-aligned with $E$ on the active rows:
if $\lambda_i>0$ then $\sign(g_i)=\sign(e_i)$ entrywise.
The gain term is
\[
\ip{G}{S}=\|G\|_1=\sum_{i=1}^d \lambda_i\|e_i\|_1,
\]
and the curvature term reduces to a weighted count of nonzeros:
\[
\tr(S^\T H S)=\sum_{i=1}^d \lambda_i\|s_i\|_2^2,
\qquad
\|s_i\|_2^2 = \#\{j:\ e_{ij}\neq 0\}\ \ (\lambda_i>0).
\]
Hence
\begin{equation}
\label{eq:diag_sign_step_delta}
\gamma^\star_{\mathrm{Sign}}
= \frac{\|G\|_1}{\tr(S^\T H S)}
= \frac{\sum_i \lambda_i\|e_i\|_1}{\sum_i \lambda_i\|s_i\|_2^2},
\qquad
\Delta_{\mathrm{Sign}}
= \frac{\|G\|_1^2}{2\,\tr(S^\T H S)}
= \frac{\Bigl(\sum_i \lambda_i\|e_i\|_1\Bigr)^2}{2\sum_i \lambda_i\|s_i\|_2^2}.
\end{equation}

\paragraph{Vertex regime (one-step exactness, diagonal $H$).}
If $E=cS$ for some $c>0$ and a sign pattern $S\in\{\pm 1,0\}^{d\times m}$, and $\lambda_i>0$ for every row $i$ in the support of $S$, then $G=HE=c(HS)$ and $\sign(HS)=S$ entrywise.
Consequently $\gamma^\star_{\mathrm{Sign}}=c$ and $E^+=0$ in one step, consistent with Theorem~\ref{thm:onestep_regimes}(ii).

\paragraph{Exact line-search quantities for \algname{PolarGD} (diagonal $H$).}

Let $G=U\Sigma_G V^\T$ be the thin SVD of $G$ with rank $r=\rank(G)$ and define $Q=\polar(G)=UV^\T$.
The gain remains $\ip{G}{Q}=\|G\|_{\nuc}$, while the diagonal structure makes the curvature term particularly interpretable:
\begin{equation}
\label{eq:diag_polar_curvature_weights}
\tr(Q^\T H Q)
= \tr(U^\T H U)
= \sum_{i=1}^d \lambda_i\, w_i,
\qquad
w_i := \|u_{i,:}\|_2^2,\quad \sum_{i=1}^d w_i = r,\quad 0\le w_i\le 1.
\end{equation}
Thus \algname{PolarGD} ``samples'' the diagonal spectrum $\{\lambda_i\}$ through the \emph{row-energy weights} of the left singular vectors of $G=HE$,
which depend on the iterate (and therefore on initialization).
The exact line-search step and one-step decrease are
\begin{equation}
\label{eq:diag_polar_step_delta}
\gamma^\star_{\mathrm{Pol}}
= \frac{\|G\|_{\nuc}}{\tr(Q^\T H Q)},
\qquad
\Delta_{\mathrm{Pol}}
= \frac{\|G\|_{\nuc}^2}{2\,\tr(Q^\T H Q)}
= \frac{\|G\|_{\nuc}^2}{2\sum_i \lambda_i w_i}.
\end{equation}

\paragraph{Full row rank simplification (shape condition).}
If $\rank(G)=d$ (full row rank), then necessarily $d\le m$ and $U\in\R^{d\times d}$ is orthogonal, so $QQ^\T=I_d$.
In this case \emph{for any} $H$ (in particular for diagonal $H$),
\[
\tr(Q^\T H Q) = \tr(H QQ^\T)=\tr(H)=\sum_{i=1}^d \lambda_i,
\]
so the \algname{PolarGD} curvature is constant (iterate-independent) and
\begin{equation}
\label{eq:diag_polar_fullrowrank}
\gamma^\star_{\mathrm{Pol}}=\frac{\|G\|_{\nuc}}{\tr(H)},\qquad
\Delta_{\mathrm{Pol}}=\frac{\|G\|_{\nuc}^2}{2\,\tr(H)}.
\end{equation}

\subsubsection{Local Comparison in the Diagonal Case (Explicit but Initialization-Dependent)}

In the diagonal setting, the exact local comparison $\Delta_{\mathrm{Pol}}>\Delta_{\mathrm{GD}}$ is equivalent to
\[
\frac{\|G\|_{\nuc}^2}{\tr(Q^\T H Q)} \;>\; \frac{\|G\|_\F^4}{\tr(G^\T H G)}
\quad\Longleftrightarrow\quad
\|G\|_{\nuc}^2\,\tr(G^\T H G)\;>\;\|G\|_\F^4\,\tr(Q^\T H Q).
\]
Using \eqref{eq:diag_gd_step_delta} and \eqref{eq:diag_polar_curvature_weights}, this becomes
\begin{equation}
\label{eq:diag_local_compare_explicit}
\|G\|_{\nuc}^2\Bigl(\sum_{i=1}^d \lambda_i^3\|e_i\|_2^2\Bigr)
\;>\;
\Bigl(\sum_{i=1}^d \lambda_i^2\|e_i\|_2^2\Bigr)^2
\Bigl(\sum_{i=1}^d \lambda_i w_i\Bigr),
\end{equation}
which makes explicit the two sources of iterate-dependence:
(i) the singular spectrum of $G=HE$ through $\|G\|_{\nuc}$, and
(ii) the alignment of the left singular subspace of $G$ with the coordinate axes through $(w_i)_i$.

\paragraph{Diagonal iterates (\algname{PolarGD} coincides with \algname{SignGD}).}
If $d=m$ and $E$ is diagonal, then $G=HE$ is also diagonal, and $Q=\polar(G)=\sign(G)$ (diagonal $\pm1/0$).
Thus \algname{PolarGD} and \algname{SignGD} coincide on this invariant subspace of diagonal matrices, and both act as an \emph{additive shrinkage} of the active diagonal entries with a shared step.
This regime underlies the explicit $2\times 2$ example of App.~\ref{app:2x2_alternation}.

\subsection{Proof of Theorem~\ref{thm:onestep_regimes}}
\label{app:proof_thm1}

\OneStepRegimes*
\begin{proof}
Let $G=HE$. Since $H\succ0$, $E=0$ is the unique minimizer, so $D$ is one-step exact at $X\ne X^\star$ if and only if $E=\gamma D$ for some $\gamma>0$; such a $\gamma$ is then the exact line-search step, since it attains the minimum value $0$ along the ray.

\emph{(i) \algname{GD}.} $E=\gamma HE$ with $\gamma>0$ if and only if $HE=\lambda E$ with $\lambda=1/\gamma>0$, which is the fixed-$X$ statement. If $H=\lambda I_d$, this holds for every $E$. Conversely, if it holds for every $E$, then every nonzero $v\in\R^d$ is an eigenvector of $H$ (take $E$ with a single nonzero column $v$). For linearly independent eigenvectors $u,v$ with eigenvalues $\lambda_u,\lambda_v$, the vector $u+v$ is also an eigenvector, say with eigenvalue $\mu$, so $(\lambda_u-\mu)u+(\lambda_v-\mu)v=0$ and $\lambda_u=\mu=\lambda_v$. Hence $H=\lambda I_d$, and $\lambda>0$ since $H\succ0$.

\emph{(ii) \algname{SignGD}.} $E=\gamma\sign(G)$ with $\gamma>0$ holds if and only if all nonzero entries of $E$ have the common magnitude $\gamma$ and $\sign(E)=\sign(G)=\sign(HE)$ entrywise, including zero entries: on the support of $E$ the signs must match, and off the support $HE$ must vanish. If $H$ is strictly diagonally dominant with positive diagonal and $s\in\{\pm1\}^d$, then for every $i$,
\[
  s_i(Hs)_i=H_{ii}+\sum_{j\neq i}H_{ij}s_is_j\ge H_{ii}-\sum_{j\neq i}|H_{ij}|>0,
\]
so $\sign(Hs)=s$. Applying this to each column of $S\in\{\pm1\}^{d\times m}$ gives $\sign(HS)=S$, so $E=cS$ satisfies the condition.

\emph{(iii) \algname{PolarGD}.} Since $\rank(G)=d$, we have $QQ^\T=I_d$. If $E=\gamma Q$, then $EE^\T=\gamma^2QQ^\T=\gamma^2I_d$. Conversely, suppose $EE^\T=cI_d$. Then $GG^\T=HEE^\T H=cH^2$, and the representation $Q=(GG^\top)^{-1/2}G$ in~\eqref{eq:precond:polar_whiten_compact} gives, since $H\succ0$,
\[
  Q=(GG^\T)^{-1/2}G=(cH^2)^{-1/2}HE=\frac{1}{\sqrt c}H^{-1}HE=\frac{1}{\sqrt c}E .
\]
Hence $E=\sqrt c\,Q$, and the step $\gamma=\sqrt c$ reaches $X^\star$. Finally, $EE^\T=cI_d$ means that all $d$ singular values of $E$ equal $\sqrt c$, i.e., $E=\sqrt c\,UV^\T$ for a thin SVD $E=U\Sigma_E V^\T$.
\end{proof}

\subsection{Orthogonal Equivariance and Coordinate Dependence}
\label{subsec:orth_equiv}

This subsection formalizes which of the three canonical direction rules
(\algname{GD}, \algname{PolarGD}, \algname{SignGD}) are equivariant under orthogonal changes
of coordinates. These equivariances justify working in a diagonal eigenbasis
for $H$ when analyzing \algname{GD} and \algname{PolarGD}, and clarify why the same reduction
is \emph{not} lossless for \algname{SignGD}.

\paragraph{Orthogonal actions on matrices.}
For $P\in\R^{d\times d}$ and $R\in\R^{m\times m}$ orthogonal, define the left--right
orthogonal action
\[
  \mathcal{T}_{P,R}(X) := P X R^\T .
\]
The Frobenius inner product and the Schatten norms are invariant under $\mathcal{T}_{P,R}$:
for any $A,B\in\R^{d\times m}$,
\[
  \ip{\mathcal{T}_{P,R}(A)}{\mathcal{T}_{P,R}(B)}=\ip{A}{B},
  \
  \norm{\mathcal{T}_{P,R}(A)}_\F=\norm{A}_\F,
  \
  \norm{\mathcal{T}_{P,R}(A)}_{\nuc}=\norm{A}_{\nuc},
  \
  \norm{\mathcal{T}_{P,R}(A)}_2=\norm{A}_2.
\]

\begin{lemma}[Bi-orthogonal equivariance of the polar factor]
\label{lem:polar_equiv}
Let $G\in\R^{d\times m}$ have thin SVD $G=U\Sigma_G V^\T$ and define its thin polar factor
$\polar(G):=UV^\T$.
Then for any orthogonal $P\in\R^{d\times d}$ and $R\in\R^{m\times m}$,
\[
  \polar(P G R^\T)= P\,\polar(G)\,R^\T.
\]
\end{lemma}

\begin{proof}
Write $G=U\Sigma_G V^\T$. Then $PGR^\T=(PU)\Sigma_G(RV)^\T$ is a thin SVD because $PU$ and $RV$
have orthonormal columns. Hence
$\polar(PGR^\T)=(PU)(RV)^\T=P(UV^\T)R^\T=P\,\polar(G)\,R^\T$.
\end{proof}

\paragraph{Quadratic objectives and change of basis.}
Consider the (shifted) quadratic model
\[
  f_H(E)=\frac12\,\tr(E^\T H E),
  \qquad H=H^\T\succeq 0,
  \qquad E\in\R^{d\times m}.
\]
Under the change of variables $E=\mathcal{T}_{P,R}(Y)=PYR^\T$, the objective becomes
\[
  f_H(PYR^\T)=\frac12\,\tr\!\bigl(Y^\T (P^\T H P)\,Y\bigr)=:f_{\tilde H}(Y),
  \qquad \tilde H:=P^\T H P.
\]
Moreover, the gradients transform equivariantly:
\[
  \nabla_E f_H(PYR^\T)=H(PYR^\T)=P(\tilde H Y)R^\T
  =P\,\nabla_Y f_{\tilde H}(Y)\,R^\T.
\]

\begin{proposition}[Rotation equivariance of \algname{GD} and \algname{PolarGD} on quadratics]
\label{prop:gd_pol_equiv}
Fix $H=H^\T\succeq 0$. Let $\tilde H=P^\T H P$ for some orthogonal $P\in\R^{d\times d}$, and let
$R\in\R^{m\times m}$ be orthogonal. Consider corresponding iterates related by
$E_k=PY_kR^\T$.

\begin{enumerate}
\item[\textup{(i)}]
\textup{(\algname{GD} direction)} If $D_{\mathrm{GD}}(E):=\nabla f_H(E)=HE$, then
\[
  D_{\mathrm{GD}}(E_k)=P\,D_{\mathrm{GD}}^{(\tilde H)}(Y_k)\,R^\T,
  \qquad
  D_{\mathrm{GD}}^{(\tilde H)}(Y):=\nabla f_{\tilde H}(Y)=\tilde H Y.
\]

\item[\textup{(ii)}]
\textup{(\algname{PolarGD} direction)} If $D_{\mathrm{Pol}}(E):=\polar(\nabla f_H(E))$, then
\[
  D_{\mathrm{Pol}}(E_k)=P\,D_{\mathrm{Pol}}^{(\tilde H)}(Y_k)\,R^\T,
  \qquad
  D_{\mathrm{Pol}}^{(\tilde H)}(Y):=\polar(\nabla f_{\tilde H}(Y)),
\]
by Lemma~\ref{lem:polar_equiv}.

\item[\textup{(iii)}]
\textup{(Line search invariance)} For either direction choice
$D\in\{D_{\mathrm{GD}},D_{\mathrm{Pol}}\}$, the exact line-search step sizes
\[
  \gamma^\star_H(D;E):=\arg\min_{\gamma\ge 0} f_H(E-\gamma D(E)),
  \qquad
  \gamma^\star_{\tilde H}(D;Y):=\arg\min_{\gamma\ge 0} f_{\tilde H}(Y-\gamma D^{(\tilde H)}(Y)),
\]
where $D^{(\tilde H)}$ is the corresponding direction for $f_{\tilde H}$ from (i)--(ii), satisfy
\[
  \gamma^\star_H(D;E_k)=\gamma^\star_{\tilde H}(D;Y_k),
\]
and the one-step decrease in objective is preserved:
\[
  f_H(E_k)-f_H(E_{k+1}) = f_{\tilde H}(Y_k)-f_{\tilde H}(Y_{k+1}).
\]
\end{enumerate}
Consequently, the \algname{GD} and \algname{PolarGD} trajectories are identical up to orthogonal change of basis.
\end{proposition}

\begin{proof}
Parts (i) and (ii) follow from the gradient transformation and Lemma~\ref{lem:polar_equiv}.
For (iii), the exact line-search formulas on a quadratic depend only on
$\ip{G}{D}$ and $\tr(D^\T H D)$. Under $E=PYR^\T$ and $D=P\tilde D R^\T$, invariance of
$\ip{\cdot}{\cdot}$ and cyclicity of trace give
$\ip{G}{D}=\ip{\tilde G}{\tilde D}$ and $\tr(D^\T H D)=\tr(\tilde D^\T \tilde H \tilde D)$,
hence the minimizing $\gamma$ and the decrease coincide.
\end{proof}

\begin{corollary}[Diagonalization is lossless for \algname{GD} and \algname{PolarGD}]
\label{cor:diag_wlog}
Let $H=O\Lambda O^\T$ be an eigendecomposition with $O$ orthogonal and $\Lambda$ diagonal.
For \algname{GD} and \algname{PolarGD} (with fixed, LMO-normalized, dual-adapted, or exact line-search steps), analyzing
$f_H(E)$ is equivalent to analyzing $f_\Lambda(Y)$ under the change of variables $Y=O^\T E$
(up to the same objective values and the same one-step decreases).
\end{corollary}

\begin{remark}[Why \algname{SignGD} is not rotation-equivariant]
\label{rem:sign_not_equiv}
\algname{SignGD} uses the entrywise direction $D_{\mathrm{Sign}}(E):=\sign(\nabla f_H(E))$.
In general, for a dense orthogonal $P$,
\[
  \sign(PG)\neq P\,\sign(G),
\]
so \algname{SignGD} is \emph{not} equivariant under arbitrary orthogonal changes of basis and therefore
diagonalizing $H$ changes the algorithm. A minimal example: take $d=2$, $m=1$, $H=I_2$ (so $\tilde H=H$ for every $P$) and $E_0=(2,1)^\T$. Then $S_0=(1,1)^\T$, exact line search gives $\gamma^\star=3/2$, and $E_1=(1/2,-1/2)^\T\neq0$. For the rotation $P$ mapping $(\sqrt5,0)^\T$ to $E_0$, the rotated initialization $\tilde E_0=P^\T E_0=(\sqrt5,0)^\T$ is solved in one step, so the trajectories are not related by $P$, whereas \algname{GD} and \algname{PolarGD} reach the minimizer in one step from both points.
\algname{SignGD} is equivariant under signed permutations
(coordinate relabelings and sign flips), i.e., orthogonal matrices with exactly one nonzero entry
$\pm 1$ per row and per column. Consequently, diagonal (or nearly diagonal) $H$ should be
interpreted as a \emph{coordinate-aligned regime} for \algname{SignGD} rather than a lossless reduction.
\end{remark}

\paragraph{Proof of Theorem~\ref{thm:orthogonal}.}
\label{app:proof_thm2}
\Equivariance*
\begin{proof}
Suppose $E_k=P\tilde E_kR^\T$, which holds for $k=0$ by definition. Then $HE_k=P(\tilde H\tilde E_k)R^\T$, and Lemma~\ref{lem:polar_equiv} gives $\polar(HE_k)=P\,\polar(\tilde H\tilde E_k)\,R^\T$. The inner product and the norms $\norm{\cdot}_\F$ and $\norm{\cdot}_{\nuc}$ are invariant under $Z\mapsto PZR^\T$, so the exact line-search steps (Proposition~\ref{prop:gd_pol_equiv}(iii)), and the scalings $\norm{G}_\F$ and $\norm{G}_{\nuc}$ of the LMO-normalized and dual-adapted steps, coincide for the two problems. Hence $E_{k+1}=P\tilde E_{k+1}R^\T$, and the claim follows by induction; the losses coincide because $f_H(PYR^\T)=f_{\tilde H}(Y)$. Remark~\ref{rem:sign_not_equiv} shows that \algname{SignGD} does not have this property.
\end{proof}

\paragraph{Experimental implication.}
Because of Proposition~\ref{prop:gd_pol_equiv}, ``rotating'' a diagonal Hessian by an orthogonal
change of basis should not change \algname{GD} or \algname{PolarGD} behavior (up to the corresponding rotation of the
initialization), while it can substantially change \algname{SignGD} behavior. This separation helps design
experiments that isolate truly spectral effects (\algname{PolarGD}) from coordinate effects (\algname{SignGD}).

\section{An Analytic $2\times 2$ Example Where the Local Winner Alternates}
\label{app:2x2_alternation}
\begin{figure}
    \centering
    \includegraphics[width=\linewidth]{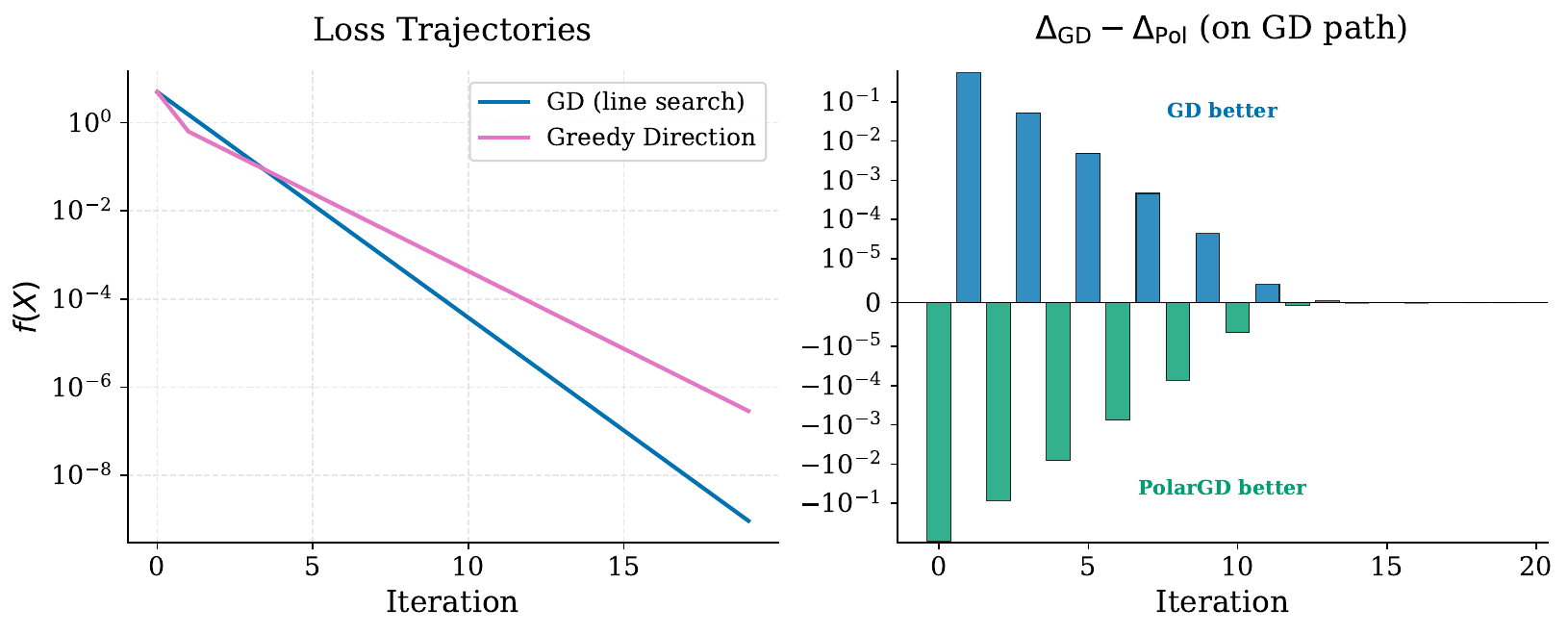}
    \caption{
    \textbf{Local advantage oscillates along the \algname{GD} path.}
    $2\times 2$ diagonal quadratic with $H = \mathrm{diag}(\kappa, 1)$, $\kappa=10$, initialized at $X_0 = \mathrm{diag}(1, \sqrt{\kappa})$.
    \textit{(Left)} Loss trajectories for \textcolor[HTML]{0072B2}{\textbf{\algname{GD} (exact line search)}} and a \textcolor[HTML]{e377c2}{\textbf{Greedy}} comparison. At each iterate $X_k$, Greedy selects the direction $D_k \in \{G_k, Q_k\}$ with the larger one-step decrease. Despite always choosing the locally better step, Greedy underperforms \algname{GD} globally. \textit{(Right)} The one-step advantage $\Delta_{\mathrm{GD}} - \Delta_{\mathrm{Pol}}$ along the \algname{GD} trajectory.}
    \label{fig:alternating}
\end{figure}
This section gives a closed-form example showing that the local one-step preference
between \algname{GD} and \algname{PolarGD} can \emph{swap repeatedly} along a trajectory, even on a convex quadratic.

Consider $d=m=2$, $X^\star=0$, and a diagonal Hessian
\[
  H=\begin{pmatrix}\kappa & 0\\ 0 & 1\end{pmatrix},
  \qquad \kappa>1.
\]
Restrict attention to diagonal iterates $X=\mathrm{diag}(x_1,x_2)$ with $x_1x_2\neq 0$.
Then $G=HX=\mathrm{diag}(\kappa x_1,x_2)$ and the exact one-step decreases under line search are
\begin{equation}
\label{eq:deltas_diag_2x2}
  \Delta_{\mathrm{GD}}(X)
  = \frac{(\kappa^2 x_1^2+x_2^2)^2}{2(\kappa^3 x_1^2+x_2^2)},
  \qquad
  \Delta_{\mathrm{Pol}}(X)
  = \frac{(\kappa|x_1|+|x_2|)^2}{2(\kappa+1)},
\end{equation}
where we used that for full-rank $2\times 2$ gradients the polar factor $Q=\polar(G)$ is orthogonal and
$\tr(Q^\T H Q)=\tr(H)=\kappa+1$.

Define the (scale-invariant) anisotropy parameter
\[
  \theta(X) := \frac{\kappa|x_1|}{|x_2|}\,.
\]
Substituting $|x_1|=\theta|x_2|/\kappa$ into \eqref{eq:deltas_diag_2x2} and simplifying gives
\begin{equation}
\label{eq:g_poly}
  \Delta_{\mathrm{Pol}}(X)-\Delta_{\mathrm{GD}}(X)
  = \frac{x_2^2\,g_\kappa(\theta)}{2(\kappa+1)(\kappa \theta^2+1)},
  \qquad
  g_\kappa(\theta)
  := (\theta+1)^2(\kappa \theta^2+1) - (\kappa+1)(\theta^2+1)^2,
\end{equation}
with $\theta=\theta(X)$. Hence $\Delta_{\mathrm{Pol}}(X)>\Delta_{\mathrm{GD}}(X)$ if and only if $g_\kappa(\theta(X))>0$.

\paragraph{A two-cycle under \algname{GD} with exact line search.}
Now run \emph{\algname{GD} with exact line search} from $X_0=\diag(1,\sqrt\kappa)$, the initialization used in Figures~\ref{fig:alternating_winner} and~\ref{fig:alternating}.
Since $H$ and $X_0$ are diagonal, \algname{GD} preserves diagonality:
$X_k=\mathrm{diag}(x_{1,k},x_{2,k})$ for all $k$.
With the exact step $\gamma_k^\star=\norm{G_k}_\F^2/\tr(G_k^\T H G_k)=(\kappa^2x_{1,k}^2+x_{2,k}^2)/(\kappa^3x_{1,k}^2+x_{2,k}^2)$, a direct computation gives
$x_{1,k+1}=(1-\kappa\gamma_k^\star)x_{1,k}$, $x_{2,k+1}=(1-\gamma_k^\star)x_{2,k}$ and
\begin{equation}
\label{eq:s_recurrence}
  \theta(X_{k+1}) = \frac{1}{\theta(X_k)}.
\end{equation}
Since $\theta(X_0)=\sqrt\kappa$, we get $\theta(X_{2t})=\sqrt\kappa$ and $\theta(X_{2t+1})=1/\sqrt\kappa$ for all $t\ge0$. Evaluating $g_\kappa$ at these two values,
\[
  g_\kappa(\sqrt\kappa) = 2\sqrt\kappa\,(\sqrt\kappa-1)(\kappa^{3/2}-1) > 0,
  \qquad
  g_\kappa(1/\sqrt\kappa)
  = -\frac{\kappa^3+\kappa^2+\kappa+1-4\kappa^{3/2}}{\kappa^2} < 0
\]
for every $\kappa>1$, where the second inequality is the AM--GM inequality $\kappa^3+\kappa^2+\kappa+1\ge 4\kappa^{3/2}$, with equality only at $\kappa=1$. Therefore, along the \algname{GD} trajectory,
\[
  \Delta_{\mathrm{Pol}}(X_{2t}) > \Delta_{\mathrm{GD}}(X_{2t}),
  \qquad
  \Delta_{\mathrm{GD}}(X_{2t+1}) > \Delta_{\mathrm{Pol}}(X_{2t+1}),
  \qquad
  \forall\,t\ge 0.
\]
The same argument applies to any diagonal initialization with $g_\kappa(\theta_0)>0>g_\kappa(1/\theta_0)$, where $\theta_0:=\theta(X_0)$; for example, $X_0=I_2$ gives $\theta_0=\kappa$, $g_\kappa(\kappa)=\kappa(\kappa+1)(\kappa-1)^2>0$ and $g_\kappa(1/\kappa)=-(\kappa+1)(\kappa-1)^2(\kappa^2+\kappa+1)/\kappa^4<0$.
Thus, even on a convex quadratic, the local one-step preference between \algname{GD} and \algname{PolarGD}
swaps \emph{every iteration} along a natural trajectory, so local comparisons should be read as diagnostics at a point, not as a global ranking.

\section{Line Search versus Tuned Fixed Steps}
\label{app:dim_ret}
\begin{figure}
    \centering
    \includegraphics[width=\linewidth]{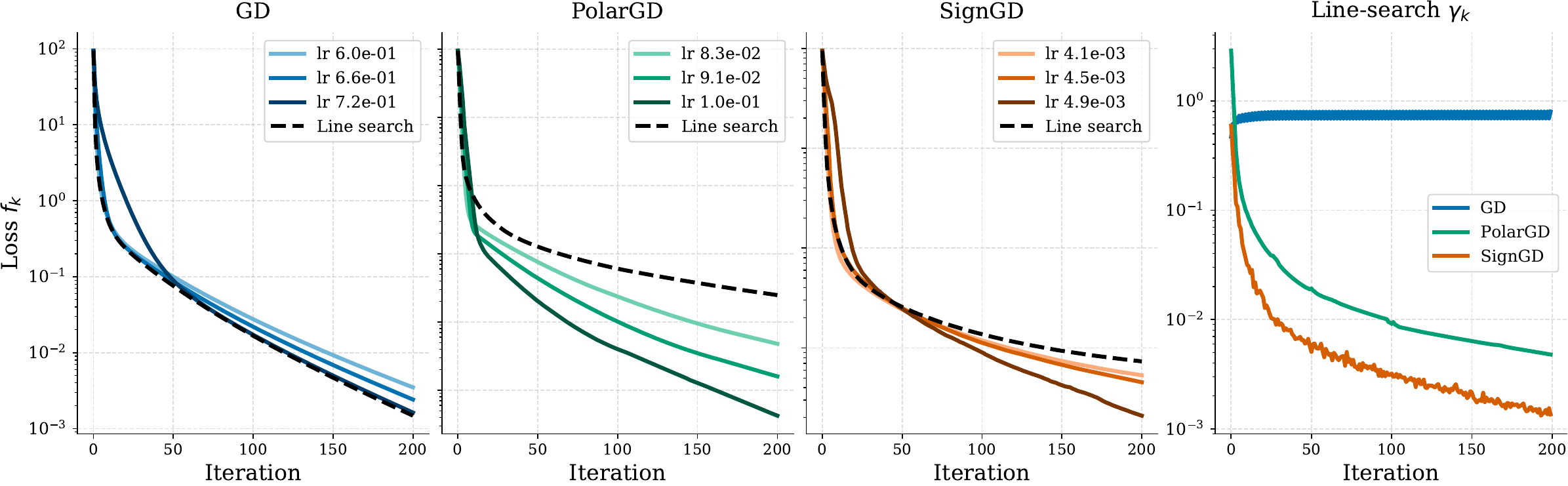}
    \caption{\textbf{Diminishing returns of line search.} Comparing three fixed LMO-normalized learning rates, $X_{k+1}=X_k-\gamma\,\mathrm{LMO}(G_k)$ (best from the grid of App.~\ref{app:exp_details:lr_tune}), to line search for \algname{GD}, \algname{PolarGD}, and \algname{SignGD} on a quadratic with dense MP Hessian ($d=20$, aspect ratio $d/N=0.8$), and Gaussian initialization ($X_0 \sim \mathcal{N}(0, 1)$).  We see that while line search typically outperforms a tuned learning rate for \algname{GD}, this is not true for \algname{PolarGD} and \algname{SignGD}, where a fixed learning rate can significantly outperform line search. The rightmost panel displays the step sizes selected by each line search method. For \algname{PolarGD} and \algname{SignGD} this step size, measured in units of the normalized direction, is positively homogeneous of degree one in $X_k$ (Proposition~\ref{prop:scale_invariance}). This reflects only the unit scale of $\polar(G)$ and $\sign(G)$: in dual-adapted units the \algname{PolarGD} line-search step is constant when $G_k$ has full row rank (Theorem~\ref{prop:precond:progress}), and the gap to the tuned rate comes from the trajectory.
    }
    \label{fig:lineSearchPol}
\end{figure}

\begin{figure}
    \centering
    \includegraphics[width=\linewidth]{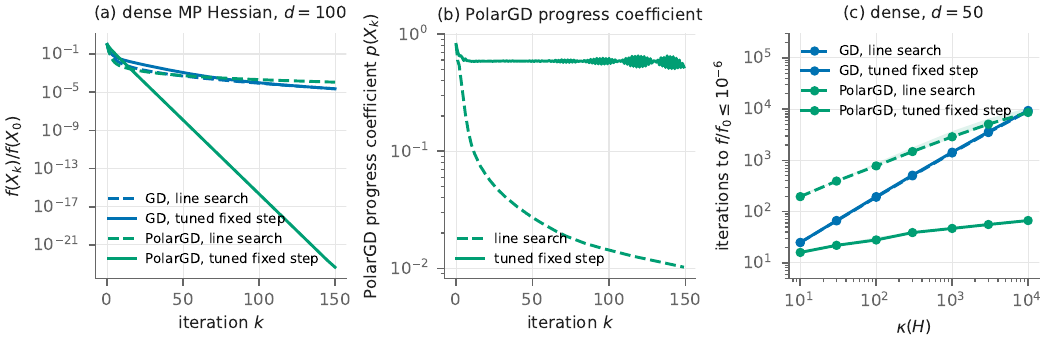}
    \caption{\textbf{Line search vs.\ tuned dual-adapted steps.} Dense Marchenko--Pastur Hessian (aspect ratio $d/N=0.8$, $d=m=100$, $\kappa(H)\approx192$; a separate draw from Figure~\ref{fig:representative_plot}), Gaussian initialization, 150 iterations; details in App.~\ref{app:exp_details:ls_vs_fixed}. \textit{(a)} Relative loss. \textit{(b)} Progress coefficient $p(X_k)$ along the two \algname{PolarGD} runs. Exact line search is the dual-adapted rate $\alpha=1$ in~\eqref{eq:precond:fixed_step}; the tuned rate $\alpha\approx1.7$ keeps $p(X_k)$ near $0.6$, whereas line search drives the iterate toward shapes with small $p(X_k)$. For \algname{GD}, line search and the tuned fixed step perform similarly. \textit{(c)} Iterations to reach $f(X_k)/f(X_0)\le10^{-6}$ versus $\kappa(H)$ for a dense Hessian with log-spaced spectrum ($d=m=50$, Gaussian initialization, median over 3 seeds). \algname{GD}'s iteration count grows roughly linearly with $\kappa(H)$ and tuned \algname{PolarGD}'s only slowly, so the gap exceeds two orders of magnitude at $\kappa(H)=10^4$; \algname{PolarGD} with line search is slower than its tuned fixed step throughout.}
    \label{fig:ls_vs_fixed}
\end{figure} 
Figure~\ref{fig:lineSearchPol} uses LMO-normalized fixed steps and Figure~\ref{fig:ls_vs_fixed} dual-adapted ones. The gap is not caused by a vanishing step: the LMO-normalized line-search step shrinks proportionally to $\norm{X}$ only because $Q$ has unit singular values (Proposition~\ref{prop:scale_invariance} below), while in the dual-adapted parameterization the line-search step of \algname{PolarGD} is the constant $1/\tr(H)$ when $G_k$ has full row rank (\citet{lau2025polargrad} observe that Armijo backtracking on the LMO-normalized step recovers this nuclear-norm scaling approximately). The mechanism of Section~\ref{sec:precond} is visible in the singular values: along the run of Figure~\ref{fig:ls_vs_fixed}(a), the share of singular values of $X_k$ below $10^{-2}\sigma_{\max}(X_k)$ grows to about $90\%$ under line search but stays mostly between $10\%$ and $20\%$ for the tuned step. The same comparison for dual-adapted steps on a full-batch one-hidden-layer MLP is in App.~\ref{app:cifar5k_ls}. We first record how the LMO-normalized line-search step scales with the iterate (Proposition~\ref{prop:scale_invariance}).
\subsection{Proof of Proposition \ref{prop:scale_invariance}}

\begin{proposition}
\label{prop:scale_invariance}
For $f(X)= \frac{1}{2} \tr(X^\T H X)$ with $H\succ0$ and $c>0$, the exact line-search step sizes satisfy
\[
\gamma^\star_{\mathrm{GD}}(cX) = \gamma^\star_{\mathrm{GD}}(X), \qquad \gamma^\star_{\mathrm{Pol}}(cX) = c\,\gamma^\star_{\mathrm{Pol}}(X).
\]
\end{proposition}
\begin{proof}
Consider the quadratic objective $f(X) = \frac{1}{2} \tr(X^\T H X)$ with $H \succ 0$. For any search direction $D$, the exact step size $\gamma^\star$ minimizes $f(X - \gamma D)$.
\paragraph{Gradient Descent.} For \algname{GD}, the direction is the gradient $G = \nabla f(X) = HX$. The step size is:
\[
    \gamma^\star_{\mathrm{GD}}(X) = \frac{\tr((HX)^\top HX)}{\tr((HX)^\top H (HX))} = \frac{\|HX\|_\F^2}{\tr(X^\top H^3 X)}.
\]
Replacing $X$ by $cX$ multiplies both numerator and denominator by $c^2$, leaving the ratio unchanged.

\paragraph{\algname{PolarGD}.} For \algname{PolarGD}, the direction is $Q = \polar(HX)$. By positive homogeneity of the polar factor, $\polar(c HX) = \polar(HX) = Q$ for any $c > 0$. The step size is:
\[
    \gamma^\star_{\mathrm{Pol}}(X) = \frac{\tr((HX)^\top Q)}{\tr(Q^\top H Q)}.
\]
Since $\|H(cX)\|_{\nuc} = c\|HX\|_{\nuc}$ while $\tr(Q^\top H Q)$ does not depend on $c$, we obtain $\gamma^\star_{\mathrm{Pol}}(cX) = c\,\gamma^\star_{\mathrm{Pol}}(X)$.
\end{proof}
\begin{remark}[\algname{SignGD}]
The same logic applies to \algname{SignGD}, where the direction is $D = \sign(HX)$. Since $\sign(c \cdot HX) = \sign(HX)$ for $c > 0$, the direction remains unchanged under scaling. Consequently, $\gamma^\star_{\mathrm{Sign}}(cX) = c\,\gamma^\star_{\mathrm{Sign}}(X)$, and the exact line search step size vanishes as $X\to 0$.
\end{remark}
\subsection{Line Search on a CIFAR-5K MLP}
\label{app:cifar5k_ls}
We evaluate on a two-layer fully-connected network with ReLU activations,
following the CIFAR-5K testbed of \citet{cohen2022gradientdescentneuralnetworks}, a subsampled collection of the CIFAR-10 dataset \citep{krizhevsky2009cifar} consisting of $5$ classes, each with $1,000$ samples.
The architecture has input dimension $d_{\mathrm{in}} = 3{,}072$ (flattened $32\times32\times3$ images),
hidden dimension $h = 4{,}608$ ($= 1.5\,d_{\mathrm{in}}$), output dimension $5$, and no bias terms;
weights are initialized via Kaiming normal initialization \citep{he2015delvingdeeprectifierssurpassing}.
We train with cross-entropy loss.

We compare standard gradient descent (\algname{GD}) and \algname{PolarGD} (the core \algname{Muon} update without momentum); \algname{PolarGD} applies the polar direction to both weight matrices. Since we are in the full batch setting with no momentum, the only hyperparameter is the step size.
For each optimizer we evaluate (i) a grid of fixed learning rates and (ii) an exact line search implemented via Brent's method, which searches for the step size in the wide interval $[0,4]$ with tolerance $10^{-5}$; each line-search step requires approximately $10$--$15$ forward passes.
All runs share a fixed random seed and are trained for $100$ iterations.
\begin{figure}
    \centering
    \includegraphics[width=\linewidth]{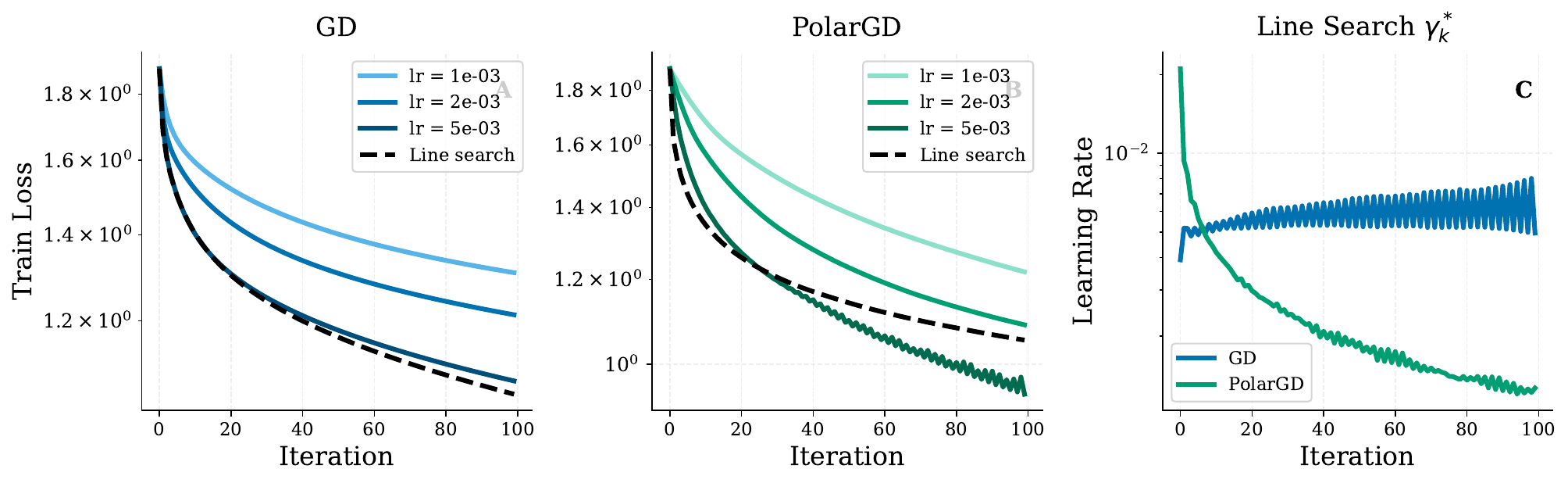}
    \caption{\textbf{Line search on an MLP.}
    Training loss on CIFAR-5K under fixed dual-adapted learning rates and line search (dashed). The three fixed rates shown are the best-performing from a grid search. For \algname{GD}, line search matches or slightly beats the best fixed rate; for \algname{PolarGD}, a fixed learning rate outperforms it. This is consistent with the local-global gap observed in the quadratic case (Section~\ref{sec:alternations}). The rightmost panel plots the line-search step size: for \algname{GD} it oscillates, whereas for \algname{PolarGD} it decays by about an order of magnitude over 100 iterations.
    }
    \label{fig:cifar5k_ls}
\end{figure}

\section{Proofs and Details for Section~\getrefnumber{sec:precond}}
\label{app:precond_proofs}

\subsection{Proofs for the Implicit-Preconditioning Identities}
\label{app:precond_proofs:identities}

\PropConditioning*
\begin{proof}
Since $H\succ0$ and $\rank(X)=d$, we have $\Sigma\succ0$ and $H\Sigma H\succ0$, so all inverse square roots are ordinary ones and $A=(H\Sigma H)^{-1/2}H$. First,
\[
  A\Sigma A^\T=(H\Sigma H)^{-1/2}\,(H\Sigma H)\,(H\Sigma H)^{-1/2}=I_d .
\]
Second, using $(H\Sigma H)^{-1}=H^{-1}\Sigma^{-1}H^{-1}$,
\[
  A^\T A=H(H\Sigma H)^{-1}H=H\,H^{-1}\Sigma^{-1}H^{-1}\,H=\Sigma^{-1}.
\]
The squared singular values of $A$ are therefore the eigenvalues of $\Sigma^{-1}$, namely $\sigma_i(X)^{-2}$, so the singular values of $A$ are $1/\sigma_i(X)$ and
\[
  \kappa(A)=\frac{\sigma_{\max}(X)}{\sigma_{\min}(X)}=\kappa(X)=\sqrt{\kappa(\Sigma)}.
\]
\end{proof}

\PropProgress*
\begin{proof}
Let $G=HX$, $Q=\polar(G)$, and $\Sigma=XX^\T$. Since $H\succ0$ and $\rank(X)=d$, also $\rank(G)=d$, so $QQ^\T=I_d$ and $\tr(Q^\T HQ)=\tr(H)$; moreover $\ip{G}{Q}=\norm{G}_{\nuc}$ and $f(X)=\tfrac12\tr(H\Sigma)$.

\emph{Step 1: line search and fixed steps.} For every $\beta\ge0$,
\[
  f(X-\beta Q)=f(X)-\beta\norm{G}_{\nuc}+\tfrac{\beta^2}{2}\tr(H).
\]
Minimizing over $\beta$ gives $\gamma^\star(Q)=\norm{G}_{\nuc}/\tr(H)$ and $f(X^+)=f(X)-\norm{G}_{\nuc}^2/(2\tr(H))$, i.e., $f(X^+)/f(X)=1-p(X)$. The dual-adapted step $\beta=\gamma\norm{G}_{\nuc}$ with $\alpha=\gamma\tr(H)$ gives $f(X^+)=f(X)-\alpha(2-\alpha)\norm{G}_{\nuc}^2/(2\tr(H))$, which is~\eqref{eq:precond:fixed_step}.

\emph{Step 2: upper bound and its equality case.} Write $G=H^{1/2}(H^{1/2}X)$. The H\"older inequality $\norm{AB}_{\nuc}\le\norm{A}_\F\norm{B}_\F$ gives
\[
  \norm{G}_{\nuc}^2\le\norm{H^{1/2}}_\F^2\,\norm{H^{1/2}X}_\F^2=\tr(H)\,\tr(H\Sigma),
\]
i.e., $p(X)\le1$. If $\Sigma=cI_d$, then $H\Sigma H=cH^2$, so
\[
  \norm{G}_{\nuc}=\tr\bigl((H\Sigma H)^{1/2}\bigr)=\sqrt c\,\tr(H),
  \qquad
  \tr(H\Sigma)=c\,\tr(H),
\]
and $p(X)=1$. Conversely, $p(X)=1$ means $f(X^+)=0$, i.e., $X^+=0$ since $H\succ0$; \algname{PolarGD} is then one-step exact, and Theorem~\ref{thm:onestep_regimes}(iii) gives $\Sigma=cI_d$.

\emph{Step 3: first lower bound.} Using $\norm{G}_{\nuc}\ge\norm{G}_\F$,
\[
  \norm{G}_{\nuc}^2\ \ge\ \norm{G}_\F^2=\tr(X^\T H^2X)=\tr(H^2\Sigma)\ \ge\ \lambda_{\min}(H)\,\tr(H\Sigma),
\]
and dividing by $\tr(H)\tr(H\Sigma)$ gives $p(X)\ge\lambda_{\min}(H)/\tr(H)$.

\emph{Step 4: second lower bound.} Let $a:=\sigma_{\min}(X)$, $b:=\sigma_{\max}(X)$, $R:=(H\Sigma H)^{1/2}$, $T:=H^{-1/2}RH^{-1/2}$, and $\rho:=H/\tr(H)$. From $a^2H^2\preceq H\Sigma H\preceq b^2H^2$ and operator monotonicity of the matrix square root, $aH\preceq R\preceq bH$, i.e., $aI_d\preceq T\preceq bI_d$. Since $\norm{G}_{\nuc}=\tr(R)=\tr(HT)$ and $\tr(H\Sigma)=\tr(RH^{-1}R)=\tr(HT^2)$,
\begin{equation}
\label{eq:app:p_via_T}
  p(X)=\frac{\tr(HT)^2}{\tr(H)\tr(HT^2)}=\frac{\tr(\rho T)^2}{\tr(\rho T^2)}.
\end{equation}
Since the spectrum of $T$ lies in $[a,b]$, the matrix $(T-aI_d)(bI_d-T)$ is positive semidefinite; as $\rho\succeq0$ and $\tr(\rho)=1$, with $t:=\tr(\rho T)$,
\[
  0\le\tr\bigl(\rho(T-aI_d)(bI_d-T)\bigr)=(a+b)\,t-ab-\tr(\rho T^2).
\]
Thus $0<\tr(\rho T^2)\le(a+b)t-ab$, and in particular $t>ab/(a+b)$. Hence
\[
  p(X)\ge\frac{t^2}{(a+b)t-ab}\ge\min_{\tau>ab/(a+b)}\frac{\tau^2}{(a+b)\tau-ab}=\frac{4ab}{(a+b)^2}=\frac{4\kappa(X)}{(1+\kappa(X))^2},
\]
where the minimum is attained at $\tau=2ab/(a+b)$.

\emph{Step 5: tightness.} The first lower bound is approached by the family of Appendix~\ref{app:pessimistic_polar_bound}. The second is attained by $d=m=2$, $H=\diag(b,a)$, and $X=\diag(a,b)$ with $0<a<b$: then $G=ab\,I_2$, $\norm{G}_{\nuc}=2ab$, $\tr(H)=a+b$, and $\tr(H\Sigma)=ab(a+b)$, so $p(X)=4ab/(a+b)^2$ with $\kappa(X)=b/a$.
\end{proof}

\begin{remark}[Progress coefficients as coefficients of variation]
\label{rem:cv}
By~\eqref{eq:app:p_via_T}, \(p(X)=1/(1+\mathrm{CV}^2)\), where
\(\mathrm{CV}^2:=\bigl(\tr(\rho T^2)-\tr(\rho T)^2\bigr)/\tr(\rho T)^2\) is the squared coefficient of
variation of \(T\) in the state \(\rho=H/\tr(H)\). If \(H\Sigma=\Sigma H\), then \(T=\Sigma^{1/2}\) and,
in a common eigenbasis with \(H=U\diag(\lambda_i)U^\top\) and \(\Sigma=U\diag(s_i^2)U^\top\),
\[
  p(X)=\frac{\bigl(\sum_i\pi_is_i\bigr)^2}{\sum_i\pi_is_i^2},
  \qquad
  \pi_i:=\frac{\lambda_i}{\tr(H)},
\]
so \(p(X)\) is determined by the spread of the singular values of \(X\), weighted by curvature. \algname{GD} with
exact line search has the dual form: writing \(e_i:=\|u_i^\top X\|_2^2\) for the energy of \(X\) along
the \(i\)th eigenvector of \(H\) and \(w_i:=\lambda_ie_i/\sum_j\lambda_je_j\), its progress coefficient is
\[
  p_{\mathrm{GD}}(X):=\frac{\Delta_{\mathrm{GD}}(X)}{f(X)}=\frac{\|G\|_\F^4}{\tr(G^\top HG)\tr(H\Sigma)}
  =\frac{\bigl(\sum_iw_i\lambda_i\bigr)^2}{\sum_iw_i\lambda_i^2},
\]
determined by the spread of the eigenvalues of \(H\), weighted by the iterate. The same argument as above
gives the classical bound \(p_{\mathrm{GD}}(X)\ge4\kappa(H)/(1+\kappa(H))^2\); for \algname{PolarGD} the role of
\(\kappa(H)\) is played by \(\kappa(X)\).
\end{remark}

\subsection{A Near-Rank-Deficient Instance Approaching the Pessimistic PolarGD Bound}
\label{app:pessimistic_polar_bound}

Theorem~\ref{prop:precond:progress} gives the lower bound
\[
  p(X)\ge \frac{\lambda_{\min}(H)}{\tr(H)}.
\]
We now show that this bound is essentially tight. The example also illustrates why exact polar
normalization can behave poorly near rank deficiency.

Consider the square case \(m=d\). Let
\[
  H=\diag(\mu,L,\dots,L),
  \qquad
  0<\mu\ll L,
  \qquad
  \kappa:=\frac{L}{\mu},
\]
and take the full-rank diagonal iterate
\[
  X_\varepsilon=\diag(1,\varepsilon,\dots,\varepsilon),
  \qquad
  \varepsilon>0.
\]
Then
\[
  \Sigma_\varepsilon=X_\varepsilon X_\varepsilon^\top
  =
  \diag(1,\varepsilon^2,\dots,\varepsilon^2),
  \qquad
  G_\varepsilon=HX_\varepsilon
  =
  \diag(\mu,L\varepsilon,\dots,L\varepsilon).
\]
Since \(G_\varepsilon\) is diagonal with nonnegative entries, its nuclear norm is
\(\|G_\varepsilon\|_{\nuc}=\mu+(d-1)L\varepsilon\). Moreover,
\[
  \tr(H)=\mu+(d-1)L,
  \qquad
  \tr(H\Sigma_\varepsilon)=\mu+(d-1)L\varepsilon^2.
\]
Therefore the exact \algname{PolarGD} progress coefficient is
\[
  p(X_\varepsilon)
  =
  \frac{\bigl(\mu+(d-1)L\varepsilon\bigr)^2}
       {\bigl(\mu+(d-1)L\bigr)\bigl(\mu+(d-1)L\varepsilon^2\bigr)}.
\]
Taking \(\varepsilon\downarrow0\) gives
\[
  \lim_{\varepsilon\downarrow0}p(X_\varepsilon)
  =
  \frac{\mu}{\mu+(d-1)L}
  =
  \frac{1}{1+(d-1)\kappa}
  \approx
  \frac{1}{\kappa d}.
\]
Thus the lower bound \(\lambda_{\min}(H)/\tr(H)\) is attained in the near-rank-deficient limit.

\paragraph{Comparison with \algname{GD}.}
On the same family, \algname{GD} with exact line search is nearly one-step exact. Its progress coefficient (Remark~\ref{rem:cv}) is
\[
  p_{\mathrm{GD}}(X)
  =
  \frac{\Delta_{\mathrm{GD}}(X)}{f(X)}
  =
  \frac{\|G\|_\F^4}{\tr(G^\top H G)\tr(H\Sigma)}.
\]
For \(X=X_\varepsilon\), we have
\[
  \|G_\varepsilon\|_\F^2
  =
  \mu^2+(d-1)L^2\varepsilon^2,
  \qquad
  \tr(G_\varepsilon^\top H G_\varepsilon)
  =
  \mu^3+(d-1)L^3\varepsilon^2.
\]
Hence
\[
  p_{\mathrm{GD}}(X_\varepsilon)
  =
  \frac{\bigl(\mu^2+(d-1)L^2\varepsilon^2\bigr)^2}
       {\bigl(\mu^3+(d-1)L^3\varepsilon^2\bigr)
        \bigl(\mu+(d-1)L\varepsilon^2\bigr)}
  \longrightarrow 1
  \qquad
  \text{as }\varepsilon\downarrow0.
\]
Thus this family separates the two methods: \algname{GD} is nearly one-step exact, while exact \algname{PolarGD}
decreases only a \(1/(1+(d-1)\kappa)\)-fraction of the loss in its first step.

\paragraph{The slowdown persists along the trajectory.}
Since \(H\) and \(X_\varepsilon\) are diagonal, \(H\Sigma_\varepsilon=\Sigma_\varepsilon H\), and
\(X_\varepsilon\) has two distinct singular values, \(1\) on the first coordinate and \(\varepsilon\) on
the others. Corollary~\ref{cor:two_cluster} applies with curvature mass
\(w=\mu/(\mu+(d-1)L)=1/(1+(d-1)\kappa)\) on the first coordinate: for every \(k\ge1\),
\[
  p(X_k)=4w(1-w),
  \qquad
  f(X_{k+1})=(1-2w)^2f(X_k)=\left(\frac{(d-1)\kappa-1}{(d-1)\kappa+1}\right)^2f(X_k),
\]
for every \(\varepsilon\in(0,1)\). This is the worst-case rate of \algname{GD} with exact line search with
\(\kappa\) replaced by \((d-1)\kappa\). By Theorem~\ref{thm:orthogonal}, the same holds for every
rotation of this instance.

\paragraph{Interpretation.}
For every \(\varepsilon>0\), the iterate is full rank, so full-row-rank \algname{PolarGD} pays curvature
\(\tr(H)=\mu+(d-1)L\). However, as \(\varepsilon\downarrow0\), almost all gradient gain is
concentrated in the first, low-curvature direction. Exact polar normalization still maps the tiny
singular modes \(L\varepsilon\) to unit-size directions, so the line-search denominator pays for all
directions even though most of the gain comes from only one.

At the limiting point \(\varepsilon=0\), the gradient is rank one. Then \algname{PolarGD} and \algname{GD} are collinear,
and the active-subspace formula from Appendix~\ref{app:rank_deficient_active_subspace} uses the
active curvature \(\tr(H\Pi_G)=\mu\) instead of \(\tr(H)\). In that rank-one limit, \algname{PolarGD} again reaches
the minimizer in one step. Hence the pathology is a near-rank-deficient full-rank effect: exact polar
is discontinuous at rank drops because it hard-normalizes every nonzero singular mode. This is the
failure mode targeted by rank truncation, damping, and finite \approxalg{Newton--Schulz} filters.

\subsection{Generic Multiplicative Recursion Under Exact Line Search}
\label{app:ls_recursion}

Consider a quadratic objective $f(X)=\frac{1}{2}\tr(X^\T H X)$ with $H\succeq 0$.
Let a generic method produce updates
\[
  X_{k+1}=X_k-\gamma_k D_k,
\]
where $\gamma_k$ is chosen by exact line search along $D_k$.
For any $D_k\neq 0$ with $\tr(D_k^\T H D_k)>0$, exact line search yields
\[
  \gamma_k = \frac{\ip{G_k}{D_k}}{\tr(D_k^\T H D_k)},
  \qquad
  G_k=\nabla f(X_k)=H X_k,
\]
and the one-step decrease equals
\[
  \Delta_k := f(X_k)-f(X_{k+1})
  =\frac{\ip{G_k}{D_k}^2}{2\,\tr(D_k^\T H D_k)}.
\]
Whenever $f(X_k)>0$, define the progress coefficient
\[
  p_k := \frac{\Delta_k}{f(X_k)} \in [0,1],
\]
so that
\begin{equation}
  \label{eq:app:multiplicative}
  f(X_{k+1}) = (1-p_k)\,f(X_k),
  \qquad
  f(X_T)=f(X_0)\prod_{k=0}^{T-1}(1-p_k)
  \le f(X_0)\exp\left(-\sum_{k=0}^{T-1}p_k\right).
\end{equation}
Thus any \emph{trajectory-wise} lower bound on $p_k$ implies a global linear (geometric) rate.

\subsection{GD with Exact Line Search: Classical Rate}
\label{app:gd_ls_rate}

For \algname{GD}, $D_k=G_k$, hence
\[
  \Delta_k^{\mathrm{GD}}=\frac{\norm{G_k}_\F^4}{2\,\tr(G_k^\T H G_k)}.
\]
For $H\succ 0$, it is classical \citep[Theorem~3.3]{nocedal2006numerical} that exact line search yields the sharp contraction
\begin{equation}
  \label{eq:app:gd_rate}
  f(X_{k+1})\le \left(\frac{\kappa(H)-1}{\kappa(H)+1}\right)^2 f(X_k),
  \qquad
  \kappa(H)=\frac{\lambda_{\max}(H)}{\lambda_{\min}(H)}.
\end{equation}
Equivalently, the progress coefficient satisfies $p_k^{\mathrm{GD}}\ge 1-\left(\frac{\kappa(H)-1}{\kappa(H)+1}\right)^2=\frac{4\kappa(H)}{(\kappa(H)+1)^2}$.

\subsection{\algname{PolarGD} with Exact Line Search: Trajectory Dependence Through $\Sigma_k$}
\label{app:pol_ls_rate}

In the full-row-rank regime ($d\le m$ and $\mathrm{rank}(X_k)=d$), Theorem~\ref{prop:precond:progress}
gives an explicit $p(X_k)$:
\[
  p(X_k)=\frac{\norm{H X_k}_{\nuc}^2}{\tr(H)\,\tr(H X_k X_k^\T)}.
\]
Combining this with \eqref{eq:app:multiplicative} shows that global behavior depends on how
$p(X_k)$ evolves along the trajectory. In particular, the bound
$p(X_k)\ge 4\kappa(X_k)/(1+\kappa(X_k))^2$ implies that if $\kappa(X_k)\le K$ uniformly along the path, then
\[
  f(X_k)\le \left(\frac{K-1}{K+1}\right)^{2k} f(X_0),
\]
which is the classical rate of \algname{GD} with exact line search~\eqref{eq:app:gd_rate}, with $\kappa(H)$ replaced by $K$.
This makes precise the intuition from Section~\ref{sec:precond}: \algname{PolarGD} can be fast if the iterate
geometry stays close to row-isotropic, but can slow down if $\Sigma_k=X_k X_k^\T$ becomes highly anisotropic.

\newpage

\subsection{Rank-Deficient and Active-Subspace Extensions}
\label{app:rank_deficient_active_subspace}

The main text focuses on the full-row-rank regime, where $QQ^\T=I_d$ and $\tr(Q^\T HQ)=\tr(H)$. Outside it, the same identities hold with $I_d$ replaced by the active projector $\Pi_G:=QQ^\T$. Throughout this subsection, $f(X)=\tfrac12\tr(X^\T HX)$ with $H\succ0$, $G=HX$, $\Sigma=XX^\T$, and $G=U\Sigma_GV^\T$ is a thin SVD with $U\in\R^{d\times r}$, $V\in\R^{m\times r}$, $\Sigma_G=\diag(\sigma_1,\dots,\sigma_r)\succ0$, and $r=\rank(G)=\rank(X)$; we write $R:=(GG^\T)^{1/2}=U\Sigma_GU^\T$, so that $G=RQ$.

\begin{lemma}[Active projector identities]
\label{lem:active_projector_identities}
$Q=\polar(G)=UV^\T=(GG^\T)^{\dagger/2}G$, $\ QQ^\T=UU^\T=\Pi_G=\Pi_{\range(G)}$, $\ Q^\T Q=VV^\T=\Pi_{\range(G^\T)}$, and $\tr(Q^\T HQ)=\tr(H\Pi_G)$.
\end{lemma}

\begin{proof}
Since $GG^\T=U\Sigma_G^2U^\T$, we have $(GG^\T)^{\dagger/2}=U\Sigma_G^{-1}U^\T$, and therefore
\[
  (GG^\T)^{\dagger/2}G=U\Sigma_G^{-1}U^\T\,U\Sigma_GV^\T=UV^\T=Q.
\]
Moreover, $U^\T U=V^\T V=I_r$ gives
\[
  QQ^\T=UV^\T VU^\T=UU^\T,
  \qquad
  Q^\T Q=VU^\T UV^\T=VV^\T,
\]
and the trace identity follows from cyclicity: $\tr(Q^\T HQ)=\tr(HQQ^\T)=\tr(H\Pi_G)$.
\end{proof}

\subsubsection{Active-Subspace Preconditioning}

\begin{proposition}[Active-subspace preconditioner identities]
\label{prop:active_subspace_whitening}
With $A(X):=(H\Sigma H)^{\dagger/2}H$,
\[
  Q=A(X)X,\qquad A(X)\Sigma A(X)^\T=\Pi_G,\qquad A(X)^\T A(X)=H(H\Sigma H)^\dagger H .
\]
\end{proposition}

\begin{proof}
Since $H\Sigma H=HXX^\T H=GG^\T$, Lemma~\ref{lem:active_projector_identities} gives
\[
  A(X)X=(H\Sigma H)^{\dagger/2}HX=(GG^\T)^{\dagger/2}G=Q .
\]
Next,
\[
  A\Sigma A^\T=(GG^\T)^{\dagger/2}\,(GG^\T)\,(GG^\T)^{\dagger/2}
\]
is the orthogonal projector onto $\range(GG^\T)=\range(G)$, i.e., $\Pi_G$. Finally, $A^\T A=H(H\Sigma H)^{\dagger/2}(H\Sigma H)^{\dagger/2}H=H(H\Sigma H)^\dagger H$.
\end{proof}

In the full-row-rank case, $(H\Sigma H)^{-1}=H^{-1}\Sigma^{-1}H^{-1}$ gives $A^\T A=\Sigma^{-1}$ and hence $\kappa(A)=\kappa(X)$ (Proposition~\ref{prop:precond:conditioning}). For singular $\Sigma$ this factorization fails in general, and the conditioning of $A$ also depends on how the active subspace is aligned with $H$. We write $\sigma^+_{\min}(\cdot)$ and $\kappa_+(\cdot)$ for the smallest positive singular value and the condition number restricted to the positive singular values.

\begin{proposition}[Conditioning of the active preconditioner]
\label{prop:active_preconditioner_conditioning_bound}
$\sigma_{\max}(A)\le\lambda_{\max}(H)/\sigma^+_{\min}(G)$ and $\sigma^+_{\min}(A)\ge\lambda_{\min}(H)/\sigma_{\max}(G)$. Consequently, $\kappa_+(A)\le\kappa(H)\,\kappa_+(G)\le\kappa(H)^2\,\kappa_+(X)$.
\end{proposition}

\begin{proof}
Let $B:=GG^\T=H\Sigma H$, so that $A=B^{\dagger/2}H$. The positive singular values of $A$ are the square roots of the positive eigenvalues of $AA^\T=B^{\dagger/2}H^2B^{\dagger/2}$, whose range is $\range(B)$. Since $\lambda_{\min}(H)^2I_d\preceq H^2\preceq\lambda_{\max}(H)^2I_d$, congruence by $B^{\dagger/2}$ gives
\[
  \lambda_{\min}(H)^2\,B^\dagger\ \preceq\ AA^\T\ \preceq\ \lambda_{\max}(H)^2\,B^\dagger .
\]
The positive eigenvalues of $B^\dagger$ are $\sigma_i(G)^{-2}$, $i=1,\dots,r$, which gives
\[
  \sigma_{\max}(A)\le\frac{\lambda_{\max}(H)}{\sigma^+_{\min}(G)},
  \qquad
  \sigma^+_{\min}(A)\ge\frac{\lambda_{\min}(H)}{\sigma_{\max}(G)},
\]
and hence $\kappa_+(A)\le\kappa(H)\kappa_+(G)$. Finally, $\sigma_{\max}(G)\le\lambda_{\max}(H)\sigma_{\max}(X)$, and for every unit vector $v\in\range(G^\T)=\range(X^\T)$,
\[
  \norm{Gv}_2=\norm{HXv}_2\ge\lambda_{\min}(H)\norm{Xv}_2\ge\lambda_{\min}(H)\,\sigma^+_{\min}(X),
\]
so $\sigma^+_{\min}(G)\ge\lambda_{\min}(H)\sigma^+_{\min}(X)$ and $\kappa_+(G)\le\kappa(H)\kappa_+(X)$.
\end{proof}

\begin{proposition}[Invariant active subspaces]
\label{prop:invariant_subspace_conditioning}
If $\mathcal S:=\range(X)$ is $H$-invariant, then $\range(G)=\mathcal S$, $\Pi_G=\Pi_{\mathcal S}$, $A\Sigma A^\T=\Pi_G$, $A^\T A=\Sigma^\dagger$, and hence $\kappa_+(A)=\kappa_+(X)$.
\end{proposition}

\begin{proof}
Since $H$ is invertible and maps $\mathcal S$ into itself, $\range(G)=H\mathcal S=\mathcal S$, so $\Pi_G=\Pi_{\mathcal S}$. Choose an orthonormal basis adapted to $\R^d=\mathcal S\oplus\mathcal S^\perp$. Because $H$ is symmetric and $\mathcal S$ is $H$-invariant, both $H$ and $\Sigma$ are block diagonal in this basis:
\[
  H=\begin{pmatrix}H_{\mathcal S}&0\\0&H_{\mathcal S^\perp}\end{pmatrix},
  \qquad
  \Sigma=\begin{pmatrix}\Sigma_{\mathcal S}&0\\0&0\end{pmatrix},
  \qquad
  \Sigma_{\mathcal S}\succ0 .
\]
Hence
\[
  A^\T A=H(H\Sigma H)^\dagger H
  =\begin{pmatrix}H_{\mathcal S}(H_{\mathcal S}\Sigma_{\mathcal S}H_{\mathcal S})^{-1}H_{\mathcal S}&0\\0&0\end{pmatrix}
  =\begin{pmatrix}\Sigma_{\mathcal S}^{-1}&0\\0&0\end{pmatrix}
  =\Sigma^\dagger .
\]
The identity $A\Sigma A^\T=\Pi_G$ is Proposition~\ref{prop:active_subspace_whitening}.
\end{proof}

\subsubsection{Exact Line Search and Progress in the Rank-Deficient Case}

\begin{proposition}[Active-subspace line search]
\label{prop:active_subspace_progress}
If $G\neq0$, then $f(X-\gamma Q)=f(X)-\gamma\norm{G}_{\nuc}+\frac{\gamma^2}{2}\tr(H\Pi_G)$ for every $\gamma\ge0$. Hence $\gamma^\star(Q)=\norm{G}_{\nuc}/\tr(H\Pi_G)$, $\ f(X-\gamma^\star(Q)Q)=f(X)-\norm{G}_{\nuc}^2/(2\tr(H\Pi_G))$, and, whenever $f(X)>0$,
\[
  \frac{f(X-\gamma^\star(Q)Q)}{f(X)}=1-\widetilde p(X),
  \qquad
  \widetilde p(X):=\frac{\norm{G}_{\nuc}^2}{\tr(H\Pi_G)\tr(H\Sigma)}.
\]
\end{proposition}

\begin{proof}
The quadratic expansion along $Q$, together with $\ip{G}{Q}=\norm{G}_{\nuc}$ and $\tr(Q^\T HQ)=\tr(H\Pi_G)$ (Lemma~\ref{lem:active_projector_identities}), gives
\[
  f(X-\gamma Q)=f(X)-\gamma\ip{G}{Q}+\frac{\gamma^2}{2}\tr(Q^\T HQ)=f(X)-\gamma\norm{G}_{\nuc}+\frac{\gamma^2}{2}\tr(H\Pi_G).
\]
Minimizing over $\gamma\ge0$ gives the step and the decrease, and dividing by $f(X)=\tfrac12\tr(H\Sigma)$ gives the ratio.
\end{proof}

\begin{corollary}[Rank-dependent pessimistic bound]
\label{cor:rank_dependent_pessimistic_bound}
$\widetilde p(X)\ge\lambda_{\min}(H)/\tr(H\Pi_G)\ge1/(r\,\kappa(H))$.
\end{corollary}

\begin{proof}
As in Step~3 of the proof of Theorem~\ref{prop:precond:progress},
\[
  \norm{G}_{\nuc}^2\ \ge\ \norm{G}_\F^2=\tr(H^2\Sigma)\ \ge\ \lambda_{\min}(H)\,\tr(H\Sigma),
\]
which gives the first inequality. The second follows from $\tr(H\Pi_G)\le r\lambda_{\max}(H)$, since $\Pi_G$ is a rank-$r$ orthogonal projector.
\end{proof}

Compared with the full-row-rank bound $\lambda_{\min}(H)/\tr(H)\ge1/(d\kappa(H))$, the dimension factor becomes the rank $r$: rank deficiency is not automatically pessimistic, since the curvature paid drops from $\tr(H)$ to the active curvature $\tr(H\Pi_G)$ (see also the comparison at the end of this subsection).

\subsubsection{Rank Monotonicity and Rank Drops}

\begin{proposition}[Rank is non-increasing under exact \algname{PolarGD} directions]
\label{prop:rank_nonincreasing}
For every $\gamma\ge0$, $X^+:=X-\gamma Q$ satisfies $\rank(X^+)=\rank(HX^+)\le\rank(G)=\rank(X)$. In particular, exact \algname{PolarGD} cannot create new active gradient directions.
\end{proposition}

\begin{proof}
Since $G=RQ$,
\[
  HX^+=H(X-\gamma Q)=G-\gamma HQ=(R-\gamma H)Q,
\]
so $\rank(HX^+)\le\rank(Q)=\rank(G)$. As $H$ is invertible, $\rank(X^+)=\rank(HX^+)$ and $\rank(G)=\rank(X)$.
\end{proof}

\begin{proposition}[Characterizing rank drops]
\label{prop:rank_drop_characterization}
For $X^+=X-\gamma Q$, $\rank(X^+)=\rank\bigl((R-\gamma H)U\bigr)$. Hence the rank drops if and only if $(R-\gamma H)u=0$ for some nonzero $u\in\range(G)$. In the full-row-rank case, this happens if and only if $\det(R-\gamma H)=0$, i.e., $\gamma\in\spec(H^{-1/2}RH^{-1/2})$.
\end{proposition}

\begin{proof}
By the previous proof, $HX^+=(R-\gamma H)Q=(R-\gamma H)UV^\T$. Since $V^\T$ has full row rank $r$,
\[
  \rank(X^+)=\rank(HX^+)=\rank\bigl((R-\gamma H)U\bigr).
\]
The $d\times r$ matrix $(R-\gamma H)U$ loses column rank if and only if $(R-\gamma H)Ua=0$ for some $a\neq0$; setting $u=Ua\in\range(G)$ gives the condition. If $r=d$, then $U$ is orthogonal, so the rank drops if and only if $R-\gamma H$ is singular, and
\[
  R-\gamma H=H^{1/2}\bigl(H^{-1/2}RH^{-1/2}-\gamma I_d\bigr)H^{1/2}
\]
shows that this happens exactly when $\gamma\in\spec(H^{-1/2}RH^{-1/2})$.
\end{proof}

\subsubsection{A Rank-$r$ One-Step Exactness Regime}

\begin{proposition}[One-step exactness on an invariant active subspace]
\label{prop:rank_r_onestep_exactness}
If $X=\sqrt c\,UV^\T$ with $c>0$ and $U\in\R^{d\times r}$, $V\in\R^{m\times r}$ having orthonormal columns, and $\range(U)$ is $H$-invariant, then one exact-line-search \algname{PolarGD} step reaches the minimizer. Conversely, if a nonzero rank-$r$ point $X$ is one-step exact under \algname{PolarGD}, then $X=\sqrt c\,UV^\T$ for some $c>0$ and such $U,V$, with $\range(U)$ $H$-invariant.
\end{proposition}

\begin{proof}
Suppose first that $\range(U)$ is $H$-invariant. Then $HU=UB$ with $B:=U^\T HU\succ0$, and
\[
  G=HX=\sqrt c\,HUV^\T=\sqrt c\,UBV^\T .
\]
Since $B\succ0$, the polar factor of $UBV^\T$ is $UV^\T$, so $Q=UV^\T$ and
\[
  \gamma^\star(Q)=\frac{\ip{G}{Q}}{\tr(Q^\T HQ)}=\frac{\sqrt c\,\tr(B)}{\tr(B)}=\sqrt c,
  \qquad
  X^+=\sqrt c\,UV^\T-\sqrt c\,UV^\T=0 .
\]
Conversely, suppose $X-\gamma^\star Q=0$, where $Q=UV^\T$ comes from the thin SVD $G=U\Sigma_GV^\T$, and set $\sqrt c:=\gamma^\star$. Then $X=\sqrt c\,UV^\T$ and
\[
  \sqrt c\,HUV^\T=HX=G=U\Sigma_GV^\T .
\]
Multiplying on the right by $V$ gives $HU=U\Sigma_G/\sqrt c$, so $\range(U)$ is $H$-invariant.
\end{proof}

For $r=d$ the invariance condition is automatic, and Proposition~\ref{prop:rank_r_onestep_exactness} reduces to Theorem~\ref{thm:onestep_regimes}(iii): the general statement is \emph{isotropy on an invariant active subspace}.

\subsubsection{Active-Subspace Root Dynamics}

\begin{proposition}[Root dynamics with an active projector]
\label{prop:root_dynamics_active_projector}
For $X(\gamma):=X-\gamma Q$ and $\Sigma(\gamma):=X(\gamma)X(\gamma)^\T$, $\ H\Sigma(\gamma)H=(R-\gamma H)\Pi_G(R-\gamma H)$. In the full-row-rank case $\Pi_G=I_d$, so $(H\Sigma(\gamma)H)^{1/2}=|R-\gamma H|$ (cf.\ Proposition~\ref{prop:precond:root_dynamics}).
\end{proposition}

\begin{proof}
Since $G=RQ$, we have $HX(\gamma)=G-\gamma HQ=(R-\gamma H)Q$, and therefore
\[
  H\Sigma(\gamma)H=HX(\gamma)X(\gamma)^\T H=(R-\gamma H)\,QQ^\T\,(R-\gamma H)=(R-\gamma H)\Pi_G(R-\gamma H).
\]
\end{proof}

\subsubsection{Local Comparison with Active Curvature}

The exact local comparison between \algname{PolarGD} and \algname{GD} remains valid in the rank-deficient case with the \algname{PolarGD} curvature written as $\tr(H\Pi_G)$:
\[
  \Delta_{\mathrm{Pol}}>\Delta_{\mathrm{GD}}
  \;\Longleftrightarrow\;
  \frac{\tr(G^\T HG)}{\tr(H\Pi_G)}>\left(\frac{\norm{G}_\F^2}{\norm{G}_{\nuc}}\right)^2.
\]
Rank deficiency affects this comparison in two competing ways: the active curvature $\tr(H\Pi_G)$ can be much smaller than $\tr(H)$, which helps \algname{PolarGD}, while lower rank reduces the gap between $\norm{G}_{\nuc}$ and $\norm{G}_\F$. At rank one, $Q$ is collinear with $G$, and \algname{PolarGD} coincides with \algname{GD} under exact line search.

\newpage

\section{Rank-Truncated and Damped Polar Directions}
\label{app:trunc_damp}

This appendix formalizes simple modifications of the polar direction that address the
``tiny singular values'' sensitivity suggested by the representation
$Q=(G G^\T)^{\dagger/2}G$.

\subsection{Rank-Truncated Polar Direction}

Let $G\in\R^{d\times m}$ with thin SVD $G=U\Sigma_GV^\T$,
$\Sigma_G=\mathrm{diag}(\sigma_1,\dots,\sigma_r)$, $\sigma_1\ge\cdots\ge\sigma_r>0$.
For an integer $1\le s\le r$, define the rank-$s$ truncation
\begin{equation}
\label{eq:rank_s_polar}
  Q^{(s)} := U_{:,1:s} V_{:,1:s}^\T.
\end{equation}
Then $\mathrm{rank}(Q^{(s)})=s$, $\|Q^{(s)}\|_2=1$, and
\begin{equation}
\label{eq:rank_s_projector}
  Q^{(s)} (Q^{(s)})^\T = U_{:,1:s}U_{:,1:s}^\T
\end{equation}
is the projector onto the span of the top-$s$ left singular vectors.

\paragraph{LMO interpretation.}
$Q^{(s)}$ is the maximizer of $\ip{G}{Q}$ over spectral-norm-bounded matrices with rank at most $s$:
\begin{equation}
\label{eq:rank_s_lmo}
  Q^{(s)} \in \arg\max\Bigl\{\ip{G}{Q} : \|Q\|_2 \le 1,\ \mathrm{rank}(Q)\le s\Bigr\}.
\end{equation}
Moreover, the attained gain is the Ky Fan $s$-norm
\begin{equation}
\label{eq:kyfan_gain}
  \ip{G}{Q^{(s)}} = \sum_{i=1}^s \sigma_i =: \|G\|_{\nuc,s}.
\end{equation}

\paragraph{Effect on exact line search for quadratics.}
For the quadratic $f(X)=\tfrac12\tr(X^\T H X)$ with gradient $G=HX$,
the exact line-search decrease along $Q^{(s)}$ is
\begin{equation}
\label{eq:delta_rank_s}
  \Delta_{\mathrm{Pol}}^{(s)}(X)
  = \frac{\ip{G}{Q^{(s)}}^2}{2\,\tr\bigl((Q^{(s)})^\T H Q^{(s)}\bigr)}
  = \frac{\|G\|_{\nuc,s}^2}{2\,\tr\bigl((Q^{(s)})^\T H Q^{(s)}\bigr)}.
\end{equation}
In the isotropic case $H=\lambda I_d$, we have
$\tr((Q^{(s)})^\T H Q^{(s)})=\lambda\|Q^{(s)}\|_\F^2=\lambda s$, hence
\begin{equation}
\label{eq:isotropic_rank_s_progress}
  f(X^+) = f(X) - \frac{\lambda\,\|X\|_{\nuc,s}^2}{2s},
  \qquad
  \frac{f(X^+)}{f(X)} = 1 - \frac{\|X\|_{\nuc,s}^2}{s\,\|X\|_\F^2}.
\end{equation}
When $X$ has rapidly decaying singular values, choosing $s\ll d$ can significantly improve the
fractional decrease by ignoring tiny singular modes that otherwise degrade the isotropic
coefficient $p(X)=\|X\|_{\nuc}^2/(r\|X\|_\F^2)$ of~\eqref{eq:isotropic_ratio_pol}, with $r=\rank(X)$.

\subsection{Damped Polar Direction}

A soft alternative is Tikhonov damping:
for $\varepsilon>0$, define
\begin{equation}
\label{eq:damped_polar}
  Q_\varepsilon := (G G^\T + \varepsilon I_d)^{-1/2} G.
\end{equation}
If $G=U\Sigma_GV^\T$, then
\begin{equation}
\label{eq:damped_svd}
  Q_\varepsilon = U\,\mathrm{diag}\left(\frac{\sigma_i}{\sqrt{\sigma_i^2+\varepsilon}}\right)V^\T,
\end{equation}
so each singular value is mapped to $\sigma_i/\sqrt{\sigma_i^2+\varepsilon}\in(0,1)$.
The preconditioner $(G G^\T+\varepsilon I_d)^{-1/2}$ is uniformly bounded:
$\|(G G^\T + \varepsilon I_d)^{-1/2}\|_2 \le 1/\sqrt{\varepsilon}$, with equality when $\rank(G)<d$.

Both truncation~\eqref{eq:rank_s_polar} and damping~\eqref{eq:damped_polar} keep the singular vectors of the polar direction and limit the weight given to tiny singular values.

\section{Finite Newton--Schulz as Damped Polar}
\label{app:newton_schulz_filter}

This appendix relates finite \approxalg{Newton--Schulz} orthogonalization to exact polar, damped polar, and rank-truncated polar directions. Finite \approxalg{Newton--Schulz} is a polynomial map of the singular values rather than an exact damping or truncation, but it behaves like damped polar: unlike the exact polar map, it does not apply an unbounded multiplier $1/\sigma_i$ to tiny nonzero singular values.

\subsection{A General Polynomial Newton--Schulz Filter}

Consider iterations of the form
\begin{equation}
\label{eq:ns_general_iteration}
  Z_{t+1}=a_tZ_t+b_tZ_tZ_t^\T Z_t+c_t(Z_tZ_t^\T)^2Z_t,
  \qquad t=0,\dots,q-1.
\end{equation}
This includes the quintic \approxalg{Newton--Schulz} iteration of \algname{Muon} \citep{jordan2024muon}, which uses the constant coefficients $(a_t,b_t,c_t)=(3.4445,-4.7750,2.0315)$, and step-dependent variants such as Polar Express \citep{amsel2025polar}.

\begin{proposition}[Finite \approxalg{Newton--Schulz} preserves singular vectors]
\label{prop:ns_preserves_singular_vectors}
If $Z_0=U\diag(s_i)V^\T$ is a thin SVD, then $Z_t=U\diag(\phi_t(s_i))V^\T$ for all $t$, where $\phi_0(s)=s$ and $\phi_{t+1}(s)=a_t\phi_t(s)+b_t\phi_t(s)^3+c_t\phi_t(s)^5$.
\end{proposition}

\begin{proof}
We argue by induction on $t$; the case $t=0$ is the assumption. If $Z_t=U\diag(\phi_t(s_i))V^\T$, then $Z_tZ_t^\T=U\diag(\phi_t(s_i)^2)U^\T$, so
\[
  Z_tZ_t^\T Z_t=U\diag\bigl(\phi_t(s_i)^3\bigr)V^\T,
  \qquad
  (Z_tZ_t^\T)^2Z_t=U\diag\bigl(\phi_t(s_i)^5\bigr)V^\T .
\]
Substituting into~\eqref{eq:ns_general_iteration} gives $Z_{t+1}=U\diag(\phi_{t+1}(s_i))V^\T$.
\end{proof}

\subsection{Comparison with Exact Polar, Damped Polar, and Truncation}

For $G=U\diag(\sigma_i)V^\T$, each direction considered here has the form $U\diag(\phi(\sigma_i))V^\T$, i.e., it applies a scalar filter $\phi$ to the singular values, or equivalently a multiplier $m(\sigma)=\phi(\sigma)/\sigma$ to each singular mode of $G$:
\begin{itemize}[leftmargin=1.2em,itemsep=0.1em]
\item exact polar, $Q=(GG^\T)^{\dagger/2}G$: $\phi(\sigma)=1$ and $m(\sigma)=1/\sigma$, unbounded as $\sigma\downarrow0$;
\item damped polar, $Q_\varepsilon=(GG^\T+\varepsilon I_d)^{-1/2}G$: $\phi_\varepsilon(\sigma)=\sigma/\sqrt{\sigma^2+\varepsilon}$, with $\sup_\sigma m_\varepsilon(\sigma)=\varepsilon^{-1/2}$ and $\phi_\varepsilon(\sigma)=\varepsilon^{-1/2}\sigma+O(\sigma^3)$;
\item rank-$s$ truncation $Q^{(s)}$: the hard filter $\phi(\sigma_i)=1$ for $i\le s$ and $0$ otherwise;
\item $q$ \approxalg{Newton--Schulz} steps: the polynomial filter $\phi_q$ of Proposition~\ref{prop:ns_preserves_singular_vectors}, applied to the normalized singular values.
\end{itemize}

\subsection{Near-Zero Behavior and Damping Equivalence}

\begin{proposition}[Finite \approxalg{Newton--Schulz} is damping-like near zero]
\label{prop:ns_near_zero_damping}
$\phi_q(\sigma)=\ell_q\sigma+O(\sigma^3)$ as $\sigma\to0$, with $\ell_q=\prod_{t=0}^{q-1}a_t$. Near zero, finite \approxalg{Newton--Schulz} therefore matches a damped polar filter with $\varepsilon_q=\ell_q^{-2}$ in normalized singular-value units.
\end{proposition}

\begin{proof}
We argue by induction on $t$, starting from $\phi_0(\sigma)=\sigma$. If $\phi_t(\sigma)=\ell_t\sigma+O(\sigma^3)$, then $\phi_t(\sigma)^3$ and $\phi_t(\sigma)^5$ are $O(\sigma^3)$, so
\[
  \phi_{t+1}(\sigma)=a_t\phi_t(\sigma)+b_t\phi_t(\sigma)^3+c_t\phi_t(\sigma)^5=a_t\ell_t\,\sigma+O(\sigma^3).
\]
Hence $\ell_{t+1}=a_t\ell_t$ and $\ell_q=\prod_{t=0}^{q-1}a_t$. Since the damped polar filter satisfies $\phi_\varepsilon(\sigma)=\varepsilon^{-1/2}\sigma+O(\sigma^3)$, matching slopes gives $\varepsilon_q=\ell_q^{-2}$.
\end{proof}

For \algname{Muon}'s constant coefficients, $\ell_q=3.4445^q$; e.g., $\ell_2\approx11.9$, $\ell_3\approx40.9$, and $\ell_5\approx485$. For these coefficients the slope at zero is also the largest amplification.

\begin{proposition}[\algname{Muon}'s \approxalg{Newton--Schulz} filter]
\label{prop:ns_muon_bound}
Let \((a,b,c)=(3.4445,-4.7750,2.0315)\) and \(\phi_q\) as in
Proposition~\ref{prop:ns_preserves_singular_vectors} with constant coefficients. Then for every
\(q\ge1\) and \(s\in(0,1]\),
\[
  0<\phi_q(s)\le \bar x:=1.2024,
  \qquad
  \frac{\phi_q(s)}{s}\le a^q,
  \qquad
  \lim_{s\downarrow0}\frac{\phi_q(s)}{s}=a^q .
\]
Hence the constant \(L_q\) defined in Proposition~\ref{prop:ns_bounded_amplification} below equals \(\ell_q=a^q\).
\end{proposition}

\begin{proof}
Write one step as \(g(x)=x\,h(x^2)\) with \(h(y)=a+by+cy^2\). Since \(b^2-4ac<0\), \(h>0\), so \(g(x)>0\)
for \(x>0\). On \([0,\bar x]\), \(g'\) vanishes at \(x^2=(-3b\pm\sqrt{9b^2-20ac})/(10c)\), i.e.\ at
\(x\approx0.5545\) (local maximum, \(g\approx1.20237\)) and \(x\approx1.0501\) (local minimum,
\(g\approx0.6818\)), and \(g(\bar x)\approx0.9466\); hence \(g([0,\bar x])\subseteq[0,\bar x]\).
For \(x\in[0,\bar x]\) we have \(x^2\le\bar x^2<1.446<-b/c\approx2.350\), so
\(h(x^2)-a=x^2(b+cx^2)\le0\). Starting from \(\phi_0(s)=s\in(0,1]\subset(0,\bar x]\), induction gives
\(\phi_t(s)\in(0,\bar x]\) and \(\phi_{t+1}(s)/\phi_t(s)=h(\phi_t(s)^2)\in(0,a]\) for all \(t\).
Multiplying over \(t=0,\dots,q-1\) gives \(\phi_q(s)/s\le a^q\). The limit follows from
Proposition~\ref{prop:ns_near_zero_damping}.
\end{proof}

The proposition also locates the transition between the two regimes. Since $\phi_q(s)\approx\ell_q s$ for small $s$
and $\phi_q(s)\approx1$ for larger $s$, the two regimes meet at the damping scale $\tau_q:=1/\ell_q$, where the linear
part reaches $1$; equivalently, $\tau_q=\sqrt{\varepsilon_q}$ is the scale at which damped polar with
$\varepsilon_q=\ell_q^{-2}$ switches from $\sigma/\sqrt{\varepsilon_q}$ to $1$. For \algname{Muon}'s coefficients,
\(\tau_q\approx0.29,\ 8.4\cdot10^{-2},\ 2.4\cdot10^{-2},\ 2.1\cdot10^{-3},\ 4.3\cdot10^{-6}\) for \(q=1,2,3,5,10\).
Above $\tau_q$ the map behaves like exact polar; below it, it is close to the linear map \(s\mapsto\ell_qs\), i.e.,
to a \algname{GD} step on those singular values. Other conventions give similar scales: the point where $\phi_5$ first
reaches $\tfrac12$ is $1.1\cdot10^{-3}$, and the point where it first reaches $1$ is $2.5\cdot10^{-3}$. Since $\tau_q$
decreases with $q$, more steps make the map closer to exact polar, which helps far from rank deficiency and can hurt
near it (Figure~\ref{fig:ns_abl}).

\subsection{Bounded Amplification}

\begin{proposition}[Finite \approxalg{Newton--Schulz} has bounded singular-mode amplification]
\label{prop:ns_bounded_amplification}
Let $Z_0=G/\nu$ with $\nu>0$ chosen so that the singular values $s_i=\sigma_i(G)/\nu$ lie in $[0,1]$, and let $L_q:=\sup_{s\in(0,1]}|\phi_q(s)/s|$. Then $Z_q=\psi_q(GG^\T)G$ for the polynomial $\psi_q(t):=\phi_q(\sqrt t/\nu)/\sqrt t$ (extended continuously at $t=0$), and $\norm{\psi_q(GG^\T)}_2\le L_q/\nu<\infty$. By contrast, exact polar applies the multiplier $1/\sigma_i(G)$ to the $i$th singular mode, which is unbounded as $\sigma_i(G)\downarrow0$.
\end{proposition}

\begin{proof}
Since $\phi_q$ is an odd polynomial, $\psi_q$ is a polynomial in $t$. Write $G=U\diag(\sigma_i)V^\T$. By Proposition~\ref{prop:ns_preserves_singular_vectors},
\[
  Z_q=U\diag\bigl(\phi_q(\sigma_i/\nu)\bigr)V^\T
  =U\diag\Bigl(\frac{\phi_q(\sigma_i/\nu)}{\sigma_i}\Bigr)U^\T\,G
  =\psi_q(GG^\T)\,G .
\]
The eigenvalues of $\psi_q(GG^\T)$ are $\phi_q(s_i)/(\nu s_i)$ on $\range(G)$ and $\psi_q(0)=\ell_q/\nu$ on its orthogonal complement, all bounded in absolute value by $L_q/\nu$.
\end{proof}

Bounded amplification explains why finite \approxalg{Newton--Schulz} can converge faster than exact polar near rank deficiency: tiny nonzero singular values are not mapped to unit size but multiplied by at most $L_q$ (equal to $\ell_q$ for \algname{Muon}'s coefficients), as by damped polar with $\varepsilon_q=\ell_q^{-2}$. Finite \approxalg{Newton--Schulz} does not create new singular directions, however: exact zero singular values remain zero, so it softens near-rank-deficiency but does not undo exact low rank.

\section{Transformer Diagnostic}
\label{app:transformer}

\paragraph{Setup.} We trained a decoder-only transformer with \algname{Muon} and saved, at 25\%, 50\%, 75\%, and 100\% of
training, the weights, the momentum matrices $M$, and the raw minibatch gradients $G$ of all 50 matrices trained by
\algname{Muon}. Table~\ref{tab:transformer_setup} lists the configuration.

\begin{table}[h]
\centering
\caption{Configuration of the transformer run.}
\label{tab:transformer_setup}
\begin{tabular}{@{}lp{0.8\linewidth}@{}}
\toprule
Data & Nemotron-CC \citep{su2025nemotroncctransformingcommoncrawl}, 1.8B tokens; GPT-2 BPE tokenizer, vocabulary $50{,}304$; sequence length $2048$ \\
Model & decoder-only, $89.8$M parameters ($31.9$M non-embedding); $8$ layers, width $576$, $9$ heads of dimension $64$; RoPE ($\theta=5\cdot10^5$), pre-RMSNorm, fused SwiGLU MLP with hidden width $1536$; untied embedding and LM head \\
Training & $3433$ steps, batch $524{,}288$ tokens ($32$ sequences $\times$ $8$ accumulation micro-steps); one NVIDIA A100; bfloat16 mixed precision (fp32 master weights and gradients); gradient clipping at $1.0$ \\
Schedule (\algname{Muon}) & linear warmup over $10\%$ of the steps to $8\cdot10^{-3}$, then cosine decay to $10^{-6}$ \\
\algname{Muon} & 50 matrices: query, key, value, and output projections and both MLP matrices of each layer, embedding, LM head; momentum $0.9$ without Nesterov; weight decay $0$; per-matrix scaling $0.2\sqrt{\max(d_{\mathrm{out}},d_{\mathrm{in}})}$ \\
\approxalg{Newton--Schulz} & $5$ steps, coefficients $(a,b,c)=(3.4445,-4.7750,2.0315)$, applied to $M/(\norm{M}_\F+10^{-7})$ \\
Other parameters & normalization gains, \algname{Adam} with $\beta=(0.95,0.95)$ \\
\bottomrule
\end{tabular}
\end{table}

\paragraph{Definitions.} Let $s_i:=\sigma_i(M)/(\norm{M}_\F+10^{-7})$ be the singular values of the matrix that
\approxalg{Newton--Schulz} receives. We call the singular values with $s_i<\tau_5=1/\ell_5\approx2.1\cdot10^{-3}$ the
\emph{\algname{GD}-like range} (Appendix~\ref{app:newton_schulz_filter}). For an update $O=\sum_i o_iu_iv_i^\T$ ($o_i=1$ for
exact polar, $o_i=\phi_5(s_i)$ for five \approxalg{Newton--Schulz} steps), we report the share $\sum_{s_i<\tau_5}o_i^2/\sum_io_i^2$
of $\norm{O}_\F^2$ in the \algname{GD}-like range; for exact polar it equals the fraction of singular values in that range.

\paragraph{Where the \algname{GD}-like range occurs.} All singular values in the \algname{GD}-like range belong to attention matrices,
where they make up about a quarter of all singular values; the MLP matrices, the embedding, and the LM head have none.
The fraction is largest in the first layer (89\% for the query matrix and 60\% for the output projection at the final
checkpoint) and decreases with depth (Figure~\ref{fig:transformer}(b)). The momentum spectra are nearly stationary over
training (Figure~\ref{fig:app_evolution}), whereas the raw gradients have many more singular values in the \algname{GD}-like
range early in training (about 40\% for attention at 25--50\% of training) and about as many as the momentum at the
end: momentum averaging stabilizes the spectrum that \approxalg{Newton--Schulz} receives. Other cutoffs than $1/\ell_5$ change the
percentages (between $5\%$ and $28\%$ for attention at the final checkpoint) but not this pattern.

\begin{figure}[h]
    \centering
    \includegraphics[width=\linewidth]{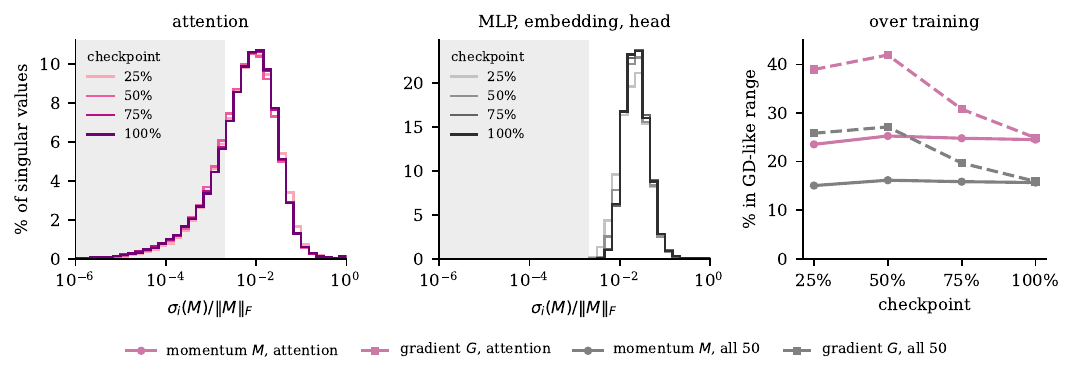}
    \caption{\textit{(Left, middle)} Distribution of the normalized singular values of the momentum matrices at four checkpoints, for the attention matrices (left) and for the MLP, embedding, and LM-head matrices (middle); the shaded band is the \algname{GD}-like range. \textit{(Right)} Percentage of singular values in the \algname{GD}-like range over training for the momentum $M$ and the raw gradient $G$, for the attention matrices and for all 50 matrices.}
    \label{fig:app_evolution}
\end{figure}

\paragraph{Relation to the weights.} In the quadratic model, tiny singular values of the gradient come from tiny
singular values of the iterate. Here the attention weights are themselves ill-conditioned, but the \algname{GD}-like
directions of $M$ are not aligned with the bottom singular directions of the weights: for the span of the $k$ left (right)
singular vectors of $M$ with $s_i<1.1\cdot10^{-3}$ (where $\phi_5(s_i)<\tfrac12$) and the span of the $k$ bottom left
(right) singular vectors of the weight matrix, the mean squared cosine of the principal angles is, in the median over the
$32$ matrices with such directions at the final checkpoint, $1.05$ ($1.02$) times its value $k/n$ for a random
$k$-dimensional subspace of $\R^n$.
transfer literally.

\paragraph{Extreme directions.} At the final checkpoint, five attention matrices have smallest normalized singular
values $s_{\min}=\sigma_{\min}/\norm{M}_\F$ between $5.6\cdot10^{-9}$ and $2.7\cdot10^{-7}$. Exact polar multiplies these
directions by $1/s_{\min}$, that is, by $1.8\cdot10^{8}$ to $3.7\cdot10^{6}$, whereas two,
three, and five \approxalg{Newton--Schulz} steps multiply them by $\ell_q=a^q\approx11.9$, $40.9$, and $485$
(Proposition~\ref{prop:ns_muon_bound}). \approxalg{Newton--Schulz} becomes polar-like on all directions of a matrix once
$\ell_q s_{\min}\ge1$, i.e., for $q\ge\lceil\log(1/s_{\min})/\log a\rceil$; over the four checkpoints, this gives $9$--$18$ steps
for the attention matrices and $4$--$5$ for the MLP matrices. The training implementation reproduces
$\mathrm{NS}_5(M)=U\diag(\phi_5(s_i))V^\T$ on all momentum matrices except the embedding and the LM head, which we did not
check, to within $10^{-4}$ (maximal absolute deviation of the diagonal entries $u_i^\T\mathrm{NS}_5(M)v_i$ from
$\phi_5(s_i)$; the off-diagonal part of $U^\T\mathrm{NS}_5(M)V$ is below $10^{-5}$ in relative Frobenius norm).

\section{Exact Dynamics through the Square-Rooted Gradient Covariance}
\label{subsec:precond:root_dynamics}

In the full-row-rank regime, the exact \algname{PolarGD} trajectory is most transparent in terms of the \emph{square-rooted gradient covariance} $R:=(H\Sigma H)^{1/2}=(GG^\T)^{1/2}$, where $\Sigma=XX^\T$ and $G=HX$.

\begin{proposition}[Exact state recursion for \algname{PolarGD} in the full-row-rank regime]
\label{prop:precond:root_dynamics}
Assume $H\succ0$ and $\rank(X)=d$. For $\gamma\ge0$, let $X(\gamma):=X-\gamma Q$ and $\Sigma(\gamma):=X(\gamma)X(\gamma)^\T$. Then
\begin{align}
  \Sigma(\gamma)&=\Sigma-\gamma\bigl(\Sigma H(H\Sigma H)^{-1/2}+(H\Sigma H)^{-1/2}H\Sigma\bigr)+\gamma^2 I_d,
  \label{eq:precond:sigma_recursion}\\
  H\Sigma(\gamma)H&=(R-\gamma H)^2,
  \label{eq:precond:root_recursion_square}\\
  f(X(\gamma))&=\tfrac12\tr\bigl(H^{-1}(R-\gamma H)^2\bigr)=\tfrac12\bigl[\tr(H^{-1}R^2)-2\gamma\tr(R)+\gamma^2\tr(H)\bigr].
  \label{eq:precond:f_along_polar_via_root}
\end{align}
Consequently, $\gamma^\star(Q)=\tr(R)/\tr(H)=\norm{G}_{\nuc}/\tr(H)$, and for $X^+:=X-\gamma^\star(Q)Q$,
\begin{equation}
\label{eq:precond:root_update_abs}
  R^+:=(H\Sigma^+H)^{1/2}=\bigl|R-\gamma^\star(Q)H\bigr|,
  \qquad
  \Sigma^+=H^{-1}\bigl(R-\gamma^\star(Q)H\bigr)^2H^{-1},
\end{equation}
where $|M|:=(M^2)^{1/2}$ for symmetric $M$.
\end{proposition}

\begin{proof}
Since $\rank(G)=d$, we have $QQ^\T=I_d$ and
\[
  Q=(GG^\T)^{-1/2}G=R^{-1}G=(H\Sigma H)^{-1/2}HX .
\]
\emph{The covariance recursion.} Expanding,
\[
  \Sigma(\gamma)=(X-\gamma Q)(X-\gamma Q)^\T=\Sigma-\gamma(XQ^\T+QX^\T)+\gamma^2QQ^\T,
\]
and substituting $Q$ and $QQ^\T=I_d$ gives~\eqref{eq:precond:sigma_recursion}.

\emph{The root recursion.} Since $G=RQ$, we have $HX(\gamma)=G-\gamma HQ=(R-\gamma H)Q$, and therefore
\[
  H\Sigma(\gamma)H=(R-\gamma H)\,QQ^\T\,(R-\gamma H)=(R-\gamma H)^2 .
\]
\emph{The loss along $Q$.} Since $X(\gamma)=H^{-1}(R-\gamma H)Q$,
\[
  f(X(\gamma))=\tfrac12\tr\bigl(X(\gamma)^\T HX(\gamma)\bigr)=\tfrac12\tr\bigl(H^{-1}(R-\gamma H)^2\bigr),
\]
and expanding the square with cyclicity of the trace gives~\eqref{eq:precond:f_along_polar_via_root}. Minimizing this quadratic in $\gamma$ gives $\gamma^\star(Q)=\tr(R)/\tr(H)$, where $\tr(R)=\norm{G}_{\nuc}$, and~\eqref{eq:precond:root_update_abs} follows from~\eqref{eq:precond:root_recursion_square} at $\gamma=\gamma^\star(Q)$.
\end{proof}

\begin{corollary}[A scalar decomposition of the one-step decrease]
\label{cor:precond:scalar_decomp}
Under the assumptions of Proposition~\ref{prop:precond:root_dynamics}, $2f(X)=\tr(H^{-1}R^2)$ and $2f(X^+)=\tr(H^{-1}R^2)-\tr(R)^2/\tr(H)$. Hence
\begin{equation}
\label{eq:precond:p_via_root}
  p(X)=\frac{\tr(R)^2}{\tr(H)\tr(H^{-1}R^2)}
  =1-\frac{\min_{\gamma\ge0}\tr\bigl(H^{-1}(R-\gamma H)^2\bigr)}{\tr(H^{-1}R^2)}.
\end{equation}
\end{corollary}

\begin{proof}
$\tr(H^{-1}R^2)=\tr(\Sigma H)=2f(X)$ because $R^2=H\Sigma H$; evaluate~\eqref{eq:precond:f_along_polar_via_root} at $\gamma^\star(Q)$.
\end{proof}

\begin{corollary}[Commuting case: additive shrinkage of the singular values]
\label{cor:precond:commuting}
Assume $H\succ0$, $\rank(X)=d$, and $H\Sigma=\Sigma H$.
Let
\[
  H = U \diag(\lambda_1,\dots,\lambda_d) U^\T,
  \qquad
  \Sigma = U \diag(s_1^2,\dots,s_d^2) U^\T,
\]
with $\lambda_i>0$ and $s_i>0$. Then
\begin{equation}
\label{eq:precond:commuting_gamma}
  R = (H\Sigma H)^{1/2} = U \diag(\lambda_1 s_1,\dots,\lambda_d s_d) U^\T,
  \qquad
  \gamma^\star(Q) = \bar s_H := \frac{\sum_{i=1}^d \lambda_i s_i}{\sum_{i=1}^d \lambda_i}.
\end{equation}
Moreover,
\begin{equation}
\label{eq:precond:commuting_sigma_update}
  \Sigma^+
  =
  (\Sigma^{1/2} - \bar s_H I_d)^2
  =
  U \diag\bigl((s_1-\bar s_H)^2,\dots,(s_d-\bar s_H)^2\bigr) U^\T.
\end{equation}
Equivalently, the singular values of $X$ evolve according to
\begin{equation}
\label{eq:precond:commuting_singular_update}
  s_i^+ = |s_i-\bar s_H|,
  \qquad i=1,\dots,d.
\end{equation}
Finally,
\begin{equation}
\label{eq:precond:commuting_loss_variance}
  f(X)=\frac12\sum_{i=1}^d \lambda_i s_i^2,
  \qquad
  f(X^+)=\frac12\sum_{i=1}^d \lambda_i (s_i-\bar s_H)^2,
\end{equation}
and
\begin{equation}
\label{eq:precond:commuting_p}
  p(X)
  =
  \frac{\bigl(\sum_{i=1}^d \lambda_i s_i\bigr)^2}
       {\bigl(\sum_{i=1}^d \lambda_i\bigr)\bigl(\sum_{i=1}^d \lambda_i s_i^2\bigr)}.
\end{equation}
\end{corollary}

\begin{proof}
Since $H\Sigma=\Sigma H$ and both matrices are symmetric, they are simultaneously diagonalizable,
and therefore
\[
  R=(H\Sigma H)^{1/2}=H\Sigma^{1/2}=\Sigma^{1/2}H.
\]
Define
\[
  W := \Sigma^{-1/2} X.
\]
Then $W W^\T = I_d$ and $X=\Sigma^{1/2}W$. Using $G=HX=H\Sigma^{1/2}W=RW$, we get
\[
  Q = R^{-1}G = W = \Sigma^{-1/2}X.
\]
Hence
\[
  X^+ = X-\gamma^\star(Q)Q = (\Sigma^{1/2}-\gamma^\star(Q)I_d)W,
\]
which implies
\[
  \Sigma^+ = X^+ X^{+\top} = (\Sigma^{1/2}-\gamma^\star(Q)I_d)^2.
\]
This proves \eqref{eq:precond:commuting_sigma_update}. The formula for $\gamma^\star(Q)$ follows
from Proposition~\ref{prop:precond:root_dynamics} and
\[
  \tr(R)=\tr(H\Sigma^{1/2})=\sum_{i=1}^d \lambda_i s_i.
\]
The expressions for $f(X)$, $f(X^+)$, and $p(X)$ follow by simultaneous diagonalization.
\end{proof}

\paragraph{Fixed steps.} The same argument applies to the fixed dual-adapted step \(X^+=X-\gamma\norm{G}_{\nuc}Q\) of Theorem~\ref{prop:precond:progress}, with \(\alpha=\gamma\tr(H)\). Since \(\norm{G}_{\nuc}=\tr(R)=\sum_i\lambda_is_i\), the step along \(Q\) is \(\alpha\bar s_H\), so \(\Sigma^+=(\Sigma^{1/2}-\alpha\bar s_H I_d)^2\) and \(s_i^+=|s_i-\alpha\bar s_H|\). In both cases \(\Sigma^+\) is a function of \(\Sigma\), so it again commutes with \(H\), and the commuting structure persists along the trajectory.

\begin{corollary}[Two-cluster dynamics]
\label{cor:two_cluster}
Under the assumptions of Corollary~\ref{cor:precond:commuting}, suppose that the singular values
\(s_1,\dots,s_d\) of \(X\) take exactly two distinct values, \(a\) for \(i\in I\) and \(b\) for
\(i\notin I\), and let \(w:=\sum_{i\in I}\lambda_i/\tr(H)\in(0,1)\). Let \(X_0:=X\) and let \(X_k\) be the
\algname{PolarGD} iterates with exact line search. Then, for every \(k\ge1\) with \(X_k\neq0\), the singular values of
\(X_k\) equal \(c_k(1-w)\) for \(i\in I\) and \(c_kw\) for \(i\notin I\), for some \(c_k>0\) (two distinct values if
\(w\neq\tfrac12\)), and
\[
  p(X_k)=4w(1-w),
  \qquad
  f(X_{k+1})=(1-2w)^2f(X_k).
\]
If \(w=\tfrac12\), then \(X_2=0\).
\end{corollary}

\begin{proof}
By Corollary~\ref{cor:precond:commuting}, the iterates stay in the common eigenbasis of \(H\) and the
singular values evolve as \(s_i^+=|s_i-\bar s_H|\) with \(\bar s_H=\sum_i\pi_is_i\),
\(\pi_i=\lambda_i/\tr(H)\). Two values \(a\) (mass \(w\)) and \(b\) (mass \(1-w\)) give
\(\bar s_H=wa+(1-w)b\), hence \(|a-\bar s_H|=(1-w)c\) and \(|b-\bar s_H|=wc\) with \(c:=|a-b|>0\).
For values \(c(1-w)\) on \(I\) and \(cw\) on its complement, \(\bar s_H=2cw(1-w)\), and the next values
are \(c(1-w)|1-2w|\) and \(cw|1-2w|\): the same configuration scaled by \(|1-2w|\). By induction, for
\(k\ge1\) the singular values of \(X_k\) are \(c_k(1-w)\) on \(I\) and \(c_kw\) on its complement with
\(c_{k+1}=|1-2w|\,c_k\). Since \(f=\tfrac12\sum_i\lambda_is_i^2\) is homogeneous of degree two in the
singular values, \(f(X_{k+1})=(1-2w)^2f(X_k)\), and \(p(X_k)=1-(1-2w)^2=4w(1-w)\). If \(w=\tfrac12\),
then \(c_2=0\). While \(c_k>0\) all singular values are positive, so the full-rank assumption holds
along the trajectory.
\end{proof}

Corollary~\ref{cor:two_cluster} gives exact trajectory-level rates in both directions.
(i)~\emph{Adverse:} the near-rank-deficient family of Appendix~\ref{app:pessimistic_polar_bound} has
\(w=1/(1+(d-1)\kappa)\) and contracts at \(((d-1)\kappa-1)^2/((d-1)\kappa+1)^2\) per step.
(ii)~\emph{Favorable:} if the curvature mass is balanced, \(w=\tfrac12\), \algname{PolarGD} reaches the minimizer
in two steps for any \(\kappa(H)\). (iii)~\emph{Figure~\ref{fig:ns_abl}:} with \(H=I_d\) and
\(X_0=O^\T\diag(1,\delta,\dots,\delta)O\), we have \(w=1/d\), so \(p(X_k)=4(d-1)/d^2\) and the loss
contracts by \((1-2/d)^2\) per step for every \(k\ge1\) and every \(\delta\in(0,1)\). (iv)~For \(d=2\) every full-rank iterate is a two-cluster configuration after one
step, and \(w=\lambda_{\min}(H)/\tr(H)\) recovers Corollary~\ref{cor:precond:2x2_postcollapse}.

\begin{remark}[\algname{PolarGD} does not monotonically improve $\kappa(\Sigma_k)$]
\label{rem:precond:no_monotone_whitening}
Even in the commuting case with $H=I_d$, \algname{PolarGD} does \emph{not} monotonically improve
the conditioning of $\Sigma_k$. For example, if $d=3$ and the singular values of $X$ are
\[
  (s_1,s_2,s_3) = (4,2,1),
\]
then $\bar s = (4+2+1)/3 = 7/3$, so
\[
  (s_1^+,s_2^+,s_3^+) = \left(\frac53,\frac13,\frac43\right).
\]
Hence
\[
  \kappa(\Sigma)=\left(\frac{4}{1}\right)^2 = 16,
  \qquad
  \kappa(\Sigma^+) = \left(\frac{5/3}{1/3}\right)^2 = 25.
\]
Thus the row-covariance can become \emph{more} ill-conditioned after a \algname{PolarGD} step.
\end{remark}

\begin{remark}[The progress coefficient $p(X_k)$ is not monotone]
\label{rem:precond:p_not_monotone}
The exact progress coefficient $p(X_k)$ from Theorem~\ref{prop:precond:progress}
need not be monotone along a \algname{PolarGD} trajectory. In the commuting case with $H=I_3$,
\[
  p(X)=\frac{(s_1+s_2+s_3)^2}{3(s_1^2+s_2^2+s_3^2)}.
\]
If $(s_1,s_2,s_3)=(3,3,2)$, then
\[
  p(X)=\frac{32}{33},
  \qquad
  (s_1^+,s_2^+,s_3^+) = \left(\frac13,\frac13,\frac23\right),
  \qquad
  p(X^+) = \frac89 < \frac{32}{33}.
\]
On the other hand, if $(s_1,s_2,s_3)=(4,1,1)$, then
\[
  p(X)=\frac23,
  \qquad
  (s_1^+,s_2^+,s_3^+)=(2,1,1),
  \qquad
  p(X^+) = \frac89 > \frac23.
\]
So $p(X_k)$ may either decrease or increase from one step to the next.
\end{remark}

\begin{proposition}[One-step collapse to gradient isotropy in the $2\times 2$ full-rank case]
\label{prop:precond:2x2_collapse}
Assume $d=2$, $H\succ0$, and $\rank(X)=2$, and let $X^+=X-\gamma^\star(Q)Q$. Then $R^+=(H\Sigma^+H)^{1/2}=cI_2$ for some $c\ge0$; equivalently, $G^+G^{+\T}=c^2I_2$ and $\Sigma^+=c^2H^{-2}$. In particular, $H\Sigma^+=\Sigma^+H$.
\end{proposition}

\begin{proof}
By Proposition~\ref{prop:precond:root_dynamics}, $R^+=|M|$ with $M:=R-\gamma^\star(Q)H$, and $\tr(M)=\tr(R)-\gamma^\star(Q)\tr(H)=0$. A symmetric trace-zero $2\times2$ matrix has eigenvalues $\pm c$ for some $c\ge0$, so $|M|=cI_2$. The equivalent forms follow from $G^+G^{+\T}=H\Sigma^+H=(R^+)^2$.
\end{proof}

\begin{corollary}[Exact post-collapse dynamics in $2\times 2$]
\label{cor:precond:2x2_postcollapse}
Under the assumptions of Proposition~\ref{prop:precond:2x2_collapse}, let $0<\lambda_1\le\lambda_2$ be the eigenvalues of $H$. For every $k\ge1$, $(HX_kX_k^\T H)^{1/2}=c_kI_2$ with $c_{k+1}=\rho c_k$ and $\rho:=(\lambda_2-\lambda_1)/(\lambda_1+\lambda_2)=(\kappa(H)-1)/(\kappa(H)+1)$, so that $f(X_{k+1})=\rho^2f(X_k)$. Moreover, whenever $c_k>0$, the \algname{PolarGD} and \algname{GD} direction rays coincide.
\end{corollary}

\begin{proof}
For $k=1$ the claim is Proposition~\ref{prop:precond:2x2_collapse}. Suppose $R_k:=(HX_kX_k^\T H)^{1/2}=c_kI_2$. Then $\gamma^\star_k(Q)=\tr(R_k)/\tr(H)=2c_k/(\lambda_1+\lambda_2)$, and Proposition~\ref{prop:precond:root_dynamics} gives, with $H=U\diag(\lambda_1,\lambda_2)U^\T$,
\[
  R_{k+1}=\Bigl|c_kI_2-\frac{2c_k}{\lambda_1+\lambda_2}H\Bigr|
  =U\diag\Bigl(c_k\Bigl|1-\frac{2\lambda_1}{\lambda_1+\lambda_2}\Bigr|,\ c_k\Bigl|1-\frac{2\lambda_2}{\lambda_1+\lambda_2}\Bigr|\Bigr)U^\T
  =\rho\,c_k\,I_2 .
\]
Since $R_k=c_kI_2$, we have $G_k=R_kQ_k=c_kQ_k$, so the \algname{PolarGD} and \algname{GD} directions are collinear whenever $c_k>0$. Finally, $HX_kX_k^\T H=c_k^2I_2$ gives $X_kX_k^\T=c_k^2H^{-2}$ and $f(X_k)=\tfrac{c_k^2}{2}\tr(H^{-1})$, so $f(X_{k+1})/f(X_k)=(c_{k+1}/c_k)^2=\rho^2$.
\end{proof}

\begin{remark}[Why $d=2$ is special]
The collapse uses that a symmetric trace-zero $2\times2$ matrix has eigenvalues $\pm c$. For $d>2$, a trace-zero symmetric matrix need not have eigenvalues of equal magnitude, so $R^+$ is generally not isotropic and \algname{PolarGD} need not reduce to \algname{GD} after one step. In $d=2$, by contrast, \algname{PolarGD} contracts at exactly \algname{GD}'s worst-case rate after the first step unless it terminates.
\end{remark}

\subsection{Specialization to the \texorpdfstring{$2\times 2$}{2x2} Case}

We specialize the line-search comparison between \algname{GD} and \algname{PolarGD} to $d=m=2$ with $H=\diag(\lambda_1,\lambda_2)$, $0<\lambda_1\le\lambda_2$, which is without loss of generality by Corollary~\ref{cor:diag_wlog}. Let $g_1^\T,g_2^\T$ be the rows of $G=HX$, and let $r_j:=\norm{g_j}_2^2$, $c:=\ip{g_1}{g_2}$, and $a:=r_1+r_2=\norm{G}_\F^2$. The eigenvalues of $GG^\T$ are $\sigma_1^2,\sigma_2^2$ with $\sigma_1^2+\sigma_2^2=a$ and $\sigma_1\sigma_2=\sqrt{\det(GG^\T)}=\sqrt{r_1r_2-c^2}$, so
\begin{equation}
\label{eq:2x2-nuclear}
  \norm{G}_{\nuc}^2=(\sigma_1+\sigma_2)^2=a+2\sqrt{r_1r_2-c^2}.
\end{equation}
With $K:=\tr(G^\T HG)=\lambda_1r_1+\lambda_2r_2$, \algname{GD} with exact line search has $\gamma^{\mathrm{GD}}=a/K$ and
\begin{equation}
\label{eq:2x2-Delta-GD}
  \Delta_{\mathrm{GD}}=\frac{a^2}{2K}.
\end{equation}
If $\rank(G)=2$, then $Q$ is orthogonal, $\tr(Q^\T HQ)=\tr(H)=\lambda_1+\lambda_2$, and
\begin{equation}
\label{eq:2x2-Delta-Mu-expanded}
  \Delta_{\mathrm{Pol}}=\frac{\norm{G}_{\nuc}^2}{2(\lambda_1+\lambda_2)}=\frac{a+2\sqrt{r_1r_2-c^2}}{2(\lambda_1+\lambda_2)}.
\end{equation}
If $\rank(G)=1$, $Q$ is collinear with $G$ and $\Delta_{\mathrm{Pol}}=\Delta_{\mathrm{GD}}$. Thus, for full-rank $G$ ($r_1r_2>c^2$), combining~\eqref{eq:2x2-Delta-GD} and~\eqref{eq:2x2-Delta-Mu-expanded},
\begin{equation}
\label{eq:2x2-condition-raw}
  \Delta_{\mathrm{Pol}}>\Delta_{\mathrm{GD}}
  \;\Longleftrightarrow\;
  \bigl(a+2\sqrt{r_1r_2-c^2}\bigr)(\lambda_1r_1+\lambda_2r_2)>a^2(\lambda_1+\lambda_2).
\end{equation}
With $\omega:=r_1/a\in[0,1]$, the share of gradient energy in the first eigendirection, and $\chi:=c/\sqrt{r_1r_2}\in[-1,1]$, the correlation between the gradient rows, we have $\sqrt{r_1r_2-c^2}=a\sqrt{\omega(1-\omega)}\sqrt{1-\chi^2}$, and~\eqref{eq:2x2-condition-raw} becomes
\begin{equation}
\label{eq:2x2-condition-alpha-rho}
  \bigl[\lambda_1\omega+\lambda_2(1-\omega)\bigr]\bigl[1+2\sqrt{\omega(1-\omega)}\sqrt{1-\chi^2}\bigr]>\lambda_1+\lambda_2.
\end{equation}
Three consequences follow. If $G$ has rank one ($\chi=\pm1$), \algname{PolarGD} coincides with \algname{GD}. If $H$ is isotropic ($\lambda_1=\lambda_2$), \eqref{eq:2x2-condition-alpha-rho} cannot hold, since $2\sqrt{\omega(1-\omega)(1-\chi^2)}\le1$, with equality only for $\omega=\tfrac12$, $\chi=0$, i.e., equal singular values of $G$; so \algname{GD} is always at least as good. If $\lambda_1<\lambda_2$, taking $\chi=0$ and $\omega$ slightly below $\tfrac12$ satisfies~\eqref{eq:2x2-condition-alpha-rho}, so there are configurations in which \algname{PolarGD} achieves a strictly larger decrease than \algname{GD}.

\subsection{Optimizer Trajectories in Singular Value Space}
By von Neumann's trace inequality,
\begin{equation}
\label{eq:vn}
f(X) = \frac{1}{2}\tr(H XX^\T) \le \frac12\sum_{i=1}^{\min(d, m)} \sigma_i (X)^2 \lambda _i(H),
\end{equation}
where $\lambda _i(H)$, $\sigma _i(X)$ are in decreasing order. Crucially, this upper bound depends only on the singular values of $X$ and the eigenvalues of $H$. In the $2 \times 2$ case we can therefore overlay this upper bound on the trajectories of singular values. This visualization (Figure~\ref{fig:2x2contour}) makes the interplay between the spectrum of $X$ and the structure of $H$ immediately apparent in two dimensions. 
\begin{figure}
    \centering
    \includegraphics[width=0.8\linewidth]{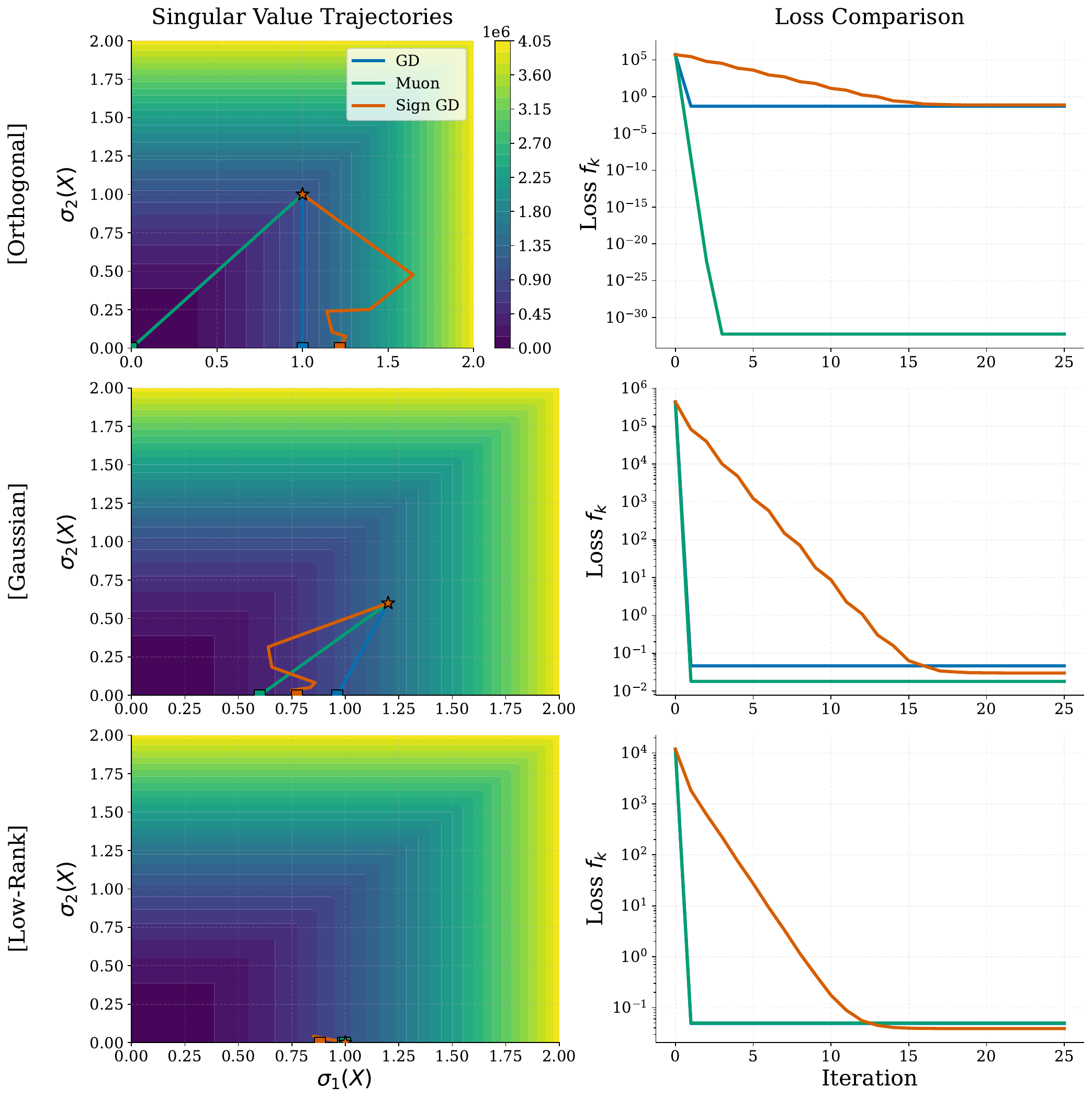}
\caption{\textbf{Singular value trajectories in $\R^2$.}
    Rows correspond to different initializations; the right column shows the loss curves.
    For each row, optimizer trajectories in singular value space, all start from the same initialization ({\large$\star$}), with the von Neumann bound (\ref{eq:vn}) overlaid. Terminal iterates are marked with a square ($\square$). Here $H=O\diag(1, 10^4)O^\T$ with $O$ a random orthogonal matrix, and the rows use orthogonal, Gaussian, and low-rank initializations.
\algname{PolarGD}'s singular value trajectory maintains a $45^\circ$ direction in singular value space. By contrast \algname{GD} is seen to optimize the larger eigen-direction first. \algname{SignGD} does not obviously exploit singular geometry here. 
    }
    \label{fig:2x2contour}
\end{figure}

\end{document}